\documentclass[letterpaper,11pt]{article}
\usepackage{etex}
\usepackage[margin=1in]{geometry}

\usepackage{amssymb,amsfonts,amsmath,amstext,amsthm,mathrsfs}
\usepackage{graphics,latexsym,epsfig,psfrag,wrapfig,comment,paralist}
\usepackage{xspace}
\usepackage{subcaption,bm,multirow}
\usepackage{booktabs}
\usepackage{mathdots}
\usepackage{bbold}
\usepackage{centernot}
\usepackage{algorithm, algorithmicx, algpseudocode}
\usepackage{prettyref}
\usepackage{listings}
\usepackage{tikz}
\usepackage{pgfplots}
\usetikzlibrary{shapes}
\usetikzlibrary{positioning, arrows, matrix, calc, patterns, fit}
\usepackage{caption}
\usepackage{array}
\usepackage{mdwmath}
\usepackage{multirow}
\usepackage{mdwtab}
\usepackage{eqparbox}
\usepackage{multicol}
\usepackage{amsfonts}
\usepackage{tikz}
\usepackage{multirow,bigstrut,threeparttable}
\usepackage{amsthm}
\usepackage{array}
\usepackage{bbm}
\usepackage{epstopdf}
\usepackage{mdwmath}
\usepackage{mdwtab}
\usepackage{eqparbox}
\usepackage{tikz}
\usepackage{latexsym}
\usepackage{amssymb}
\usepackage{bm}
\usepackage{graphicx}
\usepackage{mathrsfs}
\usepackage{epsfig}
\usepackage{psfrag}
\usepackage{setspace}
\usepackage[CJKbookmarks=true,
            bookmarksnumbered=true,
            bookmarksopen=true,
            colorlinks=true,
            citecolor=blue,
            linkcolor=blue,
            anchorcolor=red,
            urlcolor=blue
            ]{hyperref}
\usepackage{natbib}
\setcitestyle{authoryear,round}
\newtheoremstyle{break}
  {\topsep}{\topsep}%
  {\itshape}{}%
  {\bfseries}{}%
  {\newline}{}%

\usepackage{pgfplots}

\pgfmathdeclarefunction{gauss}{3}{%
  \pgfmathparse{#3/(#2*sqrt(2*pi))*exp(-((x-#1)^2)/(2*#2^2))}%
}
\pgfmathdeclarefunction{mingauss}{4}{%
  \pgfmathparse{#4/(#3*sqrt(2*pi))*exp(-max((x-#1)^2, (x-#2)^2)/(2*#3^2))}%
}
\usetikzlibrary{shapes}
\usetikzlibrary{positioning, arrows, matrix, calc, patterns, fit}

\tikzset{
  invisible/.style={opacity=0},
  visible on/.style={alt={#1{}{invisible}}},
  alt/.code args={<#1>#2#3}{%
    \alt<#1>{\pgfkeysalso{#2}}{\pgfkeysalso{#3}}
  },
}

\newtheorem{definition}{\textbf{Definition}}[section]
\newtheorem{corollary}{\textbf{Corollary}}[section]
\newtheorem{lemma}{\textbf{Lemma}}[section]
\newtheorem{theorem}{\textbf{Theorem}}[section]

\newtheorem*{insight*}{\textbf{Observation}}

\newtheorem*{proposition*}{\textbf{Proposition}}

\newtheorem*{lemmai*}{\textbf{Lemma (informal)}}
\newtheorem{remark}{\textbf{Remark}}[section]

\newtheorem{example}{\textbf{Example}}[section]

\newtheorem{problem}{\indent \em Problem}[section]

\newcommand{\bR}{\text{\boldmath{$R$}}}

\newcommand{\bbE}{\mathbb{E}}

\newcommand{\cE}{\mathcal{E}}

\def\[#1\]{\begin{align}#1\end{align}}
\def\(#1\){\begin{align*}#1\end{align*}}

\def\argmax{\operatornamewithlimits{arg\,max}}
\def\argmin{\operatornamewithlimits{arg\,min}}

\newcommand{\bprf}{\begin{proof}}
\newcommand{\eprf}{\end{proof}}
\newcommand{\blem}{\begin{lemma}}
\newcommand{\elem}{\end{lemma}}

\newcommand{\pseudoE}{E_v}

\DeclareMathOperator{\Regret}{Regret}

\newcommand{\bP}{\mathbb{P}}
\newcommand{\bE}{\mathbb{E}}

\theoremstyle{definition}

\newcommand{\sos}{\succeq_{\mathrm{sos}}}
\newcommand{\psos}{\preceq_{\mathrm{sos}}}

\newcommand{\gs}{S}

\newcommand{\TV}{\mathsf{TV}}

\usepackage[normalem]{ulem}

\newcommand{\GG}{\mathcal{G}}

\renewcommand{\cite}[1]{\citep{#1}}

\usepackage{color}

\title{Robust estimation via generalized quasi-gradients}
\author{Banghua Zhu, Jiantao Jiao, Jacob Steinhardt\thanks{Banghua Zhu is with the Department of Electrical Engineering and Computer Sciences, University of California, Berkeley. Jiantao Jiao is with the Department of Electrical Engineering and Computer Sciences and the Department of Statistics, University of California, Berkeley. Jacob Steinhardt is with the Department of Statistics and the Department of Electrical Engineering and Computer Sciences, University of California, Berkeley. Email: \{banghua, jiantao,jsteinhardt\}@berkeley.edu.}}
\date{\today}

\begin{document}

\maketitle

\begin{abstract}

We explore why many recently proposed robust estimation problems are efficiently solvable, even though the underlying optimization problems are non-convex. We study the loss landscape of these robust estimation problems, and identify the existence of ``generalized 
quasi-gradients''. Whenever these quasi-gradients exist, a large family of low-regret algorithms are guaranteed 
to approximate the global minimum; this includes the commonly-used filtering algorithm.

For  robust mean estimation of distributions 
under bounded covariance,
we show that any first-order stationary point of the associated  optimization problem is an {approximate global minimum} if and only if the corruption level $\epsilon < 1/3$. Consequently, any optimization algorithm that aproaches a stationary point yields an efficient 
robust estimator with breakdown point $1/3$. With careful initialization and step size, we improve this to $1/2$, which is optimal.

For other tasks, including linear regression and joint mean and covariance estimation, the loss landscape is more rugged: there 
are stationary points arbitrarily far from the global minimum. Nevertheless, we show that generalized quasi-gradients exist 
and construct efficient algorithms.
These algorithms are simpler than previous ones in the literature, and for linear regression we improve the estimation error 
from $O(\sqrt{\epsilon})$ to the optimal rate of $O(\epsilon)$ for small $\epsilon$ assuming certified hypercontractivity. For mean estimation with 
near-identity covariance, we show that a simple gradient descent algorithm achieves breakdown point $1/3$ and iteration 
complexity $\tilde{O}(d/\epsilon^2)$.

\end{abstract}
\tableofcontents
\newpage

\section{Introduction and main results}

We study the problem of robust estimation in the presence of outliers. In general, this means that we observe a 
dataset of $n$ points, and an adversary can corrupt (via additions or deletions) any subset of $\epsilon n$ of the points. Our 
goal is to estimate some property of the original points (such as the mean) under some assumptions (such as the good 
points having bounded covariance). In addition to mean estimation~\cite{huber1973robust, donoho1982breakdown, beran1977minimum, davies1992asymptotics, adrover2002projection, hubert2010minimum, diakonikolas2016robust}, %
we are interested linear regression~\cite{diakonikolas2019efficient, klivans2018efficient} and covariance estimation~\cite{kothari2017outlier,  diakonikolas2017being}.  %

Robust estimation has been extensively studied, and a general issue is how to design computationally efficient estimators. Recent 
papers have provided general (inefficient) recipes for solving these problems, showing that it suffices to solve an optimization problem that 
removes outliers to obtain a nice distribution---where ``nice'' can be formalized and is problem-dependent~\cite{ steinhardt2018robust, zhu2019generalized}. Although  this 
recipe in general leads to non-convex or otherwise seemingly intractable estimators, a variety of 
efficient algorithms have been proposed for many problems~\cite{lai2016agnostic,diakonikolas2019robust,kothari2017outlier, diakonikolas2017being, steinhardt2018robust,klivans2018efficient, dong2019quantum,cheng2019faster,cheng2019high}.

The large variety of computationally efficient estimators suggests that robust estimation is easier than we would have 
expected given its non-convexity. How can we explain this? Here we analyze the non-convex optimization landscape for 
several problems---mean estimation, covariance estimation, and linear regression---and show that, while the landscape is indeed 
non-convex, it is nevertheless nice enough to admit efficient optimization algorithms.

This claim is easiest to formalize for mean estimation under bounded covariance. 
In this case, we observe points $X_1, \ldots, X_n \in \bR^d$, such that a subset $S$ of $(1-\epsilon)n$ ``good'' points 
is guaranteed to have bounded covariance: $\|\Sigma_{p_S}\| \leq \sigma$, where $p_S$ is the empirical distribution over $S$,
$\Sigma_p$ is the covariance matrix under $p$, and $\|\cdot\|$ is operator norm.
As shown in 
\citet{diakonikolas2016robust, steinhardt2018resilience}, %
 estimating the mean of 
$p_S$ only requires finding \emph{any} large subset of the data 
with small covariance, as in the (non-convex) optimization problem below:
\begin{example}[Mean estimation with bounded operator norm of covariance matrix]\label{example.mean_bdd_cov}
Let $\Delta_{n,\epsilon}$ denote the set of $\epsilon$-deleted distribution:
\begin{align}\label{eqn.deltanepsdefinition}
    \Delta_{n, \epsilon} \triangleq \{q\mid \sum_{i=1}^n q_i = 1, 0\leq q_i \leq \frac{1}{(1-\epsilon)n} \},
\end{align}
where $q_i$ is the probability $q$ assigns to point $X_i$. We solve the feasibility problem\footnote{We discuss other formulations of the mean estimation problem in Appendix~\ref{sec.classicalliter}.}
\begin{align}
&\textrm{find}   \qquad  \qquad q \nonumber \\
& \textrm{subject to} \quad q\in\Delta_{n, \epsilon},
   \|\Sigma_q\| \leq \sigma'^2, \label{eqn.mean_intro_feasibility}
\end{align}
where the parameter $\sigma'^2 \geq \sigma^2$ depends on $\epsilon$, and is close to $\sigma^2$ when $\epsilon$ is small. 
In this case the mean $\mu_q$ satisfies $\|\mu_q - \mu_{p_S}\| = O((\sigma + \sigma')\sqrt{\epsilon})$ for any feasible $q$.
\end{example}

Although $\Delta_{n,\epsilon}$ is a convex constraint on $q$,
the function $\| \Sigma_q\|$ is non-convex in $q$. In dimension one, it reduces to the variance of $q$, which is concave, and not even quasiconvex. %
Nevertheless, we show that all stationary points (or 
approximate stationary points) of Example~\ref{example.mean_bdd_cov} are approximate global optima
(formal version Theorem~\ref{thm.low_regret_mean_relation}):
\begin{theorem}[Informal]\label{thm.informalboundedmean}
All first-order stationary points of minimizing $\|\Sigma_q\|$ subject to $q\in \Delta_{n,\epsilon}$ are an approximate global minimum with worst case approximation ratio $(1-\epsilon)^2/(1-3\epsilon)^2$ (and infinite approximation ratio if $\epsilon \geq 1/3$). 
Approximate stationary points are also approximate global minimum, and gradient descent algorithms can approach them efficiently.  
\end{theorem}

The approximation ratio   $(1-\epsilon)^2/(1-3\epsilon)^2$ is tight even in the constant. 
For $\epsilon \geq 1/3$, we exhibit examples where stationary points (indeed, local minima) 
can be arbitrarily far from the global minimum (Theorem~\ref{eqn.lower_bound_13}). However, we show that a carefully initialized gradient descent algorithm 
approximates the global optimum whenever $\epsilon < 1/2$ (Theorem~\ref{thm.filtering_mean_cov}), which is the highest breakdown point for any translation equivariant mean estimator~\citep[Page 270]{rousseeuw1987robust}. This gradient algorithm is an instance of the commonly-used filtering algorithm in the literature~\cite{li2018principled, jerry2019lecnotes, diakonikolas2017being, steinhardt2018robust}, but we provide a tighter analysis with optimal breakdown point. This algorithm achieves approximation ratio $2(1-\epsilon)/(1-2\epsilon)^2$~(Theorem~\ref{thm.filtering_mean_cov}) for all $\epsilon\in (0,1/2)$.

We might hope that the optimization landscape is similarly well-behaved for other robust estimation problems beyond mean estimation. 
However, this is not true in general. For linear regression, we show that the analogous optimization problem can 
have arbitrarily bad stationary points even as $\epsilon \to 0$ (Section~\ref{sec.low_regret}). 
We nevertheless show that the landscape is tractable, by identifying a property that we call \emph{generalized quasi-gradients}. 
Such quasi-gradients  allow many gradient descent algorithms to approximate the global optima.
\begin{definition}[Generalized quasi-gradient]\label{def.generalized_quasi_g}
In the optimization problem  $\min_{q\in A} F(q)$, we say $g(q)$ is a generalized quasi-gradient with parameter $C\geq 1$ 
if the following holds for all $q,p \in A$:
\begin{align}
    \langle g(q), q-p \rangle \leq 0 \implies F(q)\leq C \cdot F(p). 
\end{align}
Moreover, we call $g(q)$ a strict generalized quasi-gradient with parameters $ C_1(\alpha, \beta), C_2(\alpha, \beta)$ if the following holds for all $q, p\in A, \alpha,\beta\geq 0$, 
\begin{align}\label{eqn.implication_strict_generalized_quasi}
    \langle g(q), q-p \rangle \leq \alpha \langle |g(q)|, p \rangle + \beta \implies F(q)\leq C_1(\alpha, \beta) \cdot F(p) + C_2(\alpha, \beta),
\end{align}
where $|g(q)|$ is the point-wise absolute value of vector $g(q)$.
\end{definition}
The conventional quasi-gradient is a generalized quasi-gradient with parameter $C = 1$~\cite{boyd2007subgradient}, which only exists for quasi-convex functions. Our next result shows that even though the target functions we consider are not quasi-convex ($q \mapsto \|\Sigma_q\|$ is not quasi-convex as a concrete example),  we can still find generalized quasi-gradients:

\begin{theorem}[Informal]\label{thm.intro_quasi_exists}
Generalized quasi-gradient exists for all 
 the optimization problems %
investigated in the paper, including mean estimation (Example~\ref{example.mean_bdd_cov},~\ref{example.mean_id_cov}),  linear regression (Example~\ref{example.linreg}), and joint mean and covariance estimation (Example~\ref{example.joint}). Here the set $A = \Delta_{n,\epsilon}$. %
\end{theorem}

Strict generalized quasi-gradients are important %
because every low-regret algorithm can approach points $q$ such that the inequality $\langle g(q), q-p_S\rangle \leq 0$ approximately holds~\cite{nesterov2009primal, arora2012multiplicative}. This then immediately implies that $F(q)\leq C_1 \cdot F(p_S) + C_2$ (see e.g. Theorem~\ref{thm.approximate_mean_small}). Thus once we identify (strict) generalized quasi-gradients that are efficiently computable, we immediately obtain a family of algorithms that approximately solve the feasibility problem. We elaborate on our concrete algorithm constructions in Section~\ref{sec.introalgorithms}.

\subsection{Constructing generalized quasi-gradients}
We next describe generalized quasi-gradients for several tasks. 
Our starting point is the following optimization problem, which generalizes Example~\ref{example.mean_bdd_cov}:
\begin{problem}[Approximate Minimum Distance (AMD) functional with $\TV$] 
\label{prob.feasibility}
We solve the feasibility problem
\begin{align}
&\textrm{find}   \qquad  \qquad q \nonumber \\
& \textrm{subject to} \quad q\in \Delta_{n,\epsilon},
   F(q) \leq \xi, \label{eqn.intro_feasibility}
\end{align}
where $\Delta_{n,\epsilon}$ is defined in~(\ref{eqn.deltanepsdefinition}). 
\end{problem}
The only difference between optimization problem~\eqref{eqn.intro_feasibility} and~\eqref{eqn.mean_intro_feasibility} is that we have replaced $\|\Sigma_q\|$ with a more general function $F(q)$.  We often also consider the minimization form $\min_{q\in \Delta_{n, \epsilon}} F(q)$. %

The appropriate $F$ to use is problem-dependent and depends on what distributional assumptions we are willing to make.  
\citet{zhu2019generalized} provides a general treatment for how to choose $F$. For linear regression (Example~\ref{example.linreg}, also in~\citet[Example 3.2]{zhu2019generalized}) and joint mean and covariance estimation  under the Mahalanobis distances  %
(Example~\ref{example.joint}, also in~\citet{kothari2017outlier}), 
the appropriate $F$ is closely related to the hypercontractivity coefficient of $q$, 
represented by the function 
\begin{align}\label{eqn.linreg_hyper_F}
F_1(q) = \sup_{v\in\bR^d} \frac{\bE_{q}[(v^\top X)^4]}{\bE_{q}[(v^\top X)^2]^2} \leq \kappa^2. 
\end{align}
As with the covariance $\|\Sigma_q\|$, 
the function $F_1(q)$ in~(\ref{eqn.linreg_hyper_F}) is generally not a convex function of $q$. Indeed, if $d = 1$, then its sublevel set is the complementary set of a convex set, which makes the function not even quasi-convex. %
But more problematically, as mentioned above, we can construct first-order stationary points of \eqref{eqn.intro_feasibility} where 
$F_1(q)$ is arbitrarily big while $\epsilon$ and $F_1(p_S)$ are both small~(Theorem~\ref{thm.hyperstationarybad}).

Nevertheless, the following function (among others) is a generalized quasi-gradient for $F_1(q)$ with $C = 4$ when $9\kappa^2\epsilon\leq 1$:%
\begin{align}\label{eqn.genquasigradienthyper}
    g_1(X; q) = \bE_q[(v^\top X)^4],  \text{ where } v\in\argmax_{v\in\bR^d} \frac{\bE_{q}[(v^\top X)^4]}{\bE_{q}[(v^\top X)^2]^2},
\end{align}
which we analyze in Section~\ref{subsec.low_regret_linreg}. %
Since the supremum in~(\ref{eqn.genquasigradienthyper}) is not generally efficiently computable, in practice we make the stronger assumption that $p_S$ has Sum-of-Squares (SoS) certifiable hypercontractivity, and construct an efficient relaxation of \eqref{eqn.genquasigradienthyper} using pseudoexpectations
(see Appendix~\ref{appendix.sos} for formal definitions).

Given a quasi-gradient for $F_1$, we are most of the way to designing algorithms for joint mean and covariance estimation, as well 
as linear regression. For joint mean and covariance (Example~\ref{example.joint}), we actually need to handle a centered version of $F_1$, 
where we consider $X - \mu_q$ instead of $X$. We show in Section~\ref{sec.joint} that the analogous quasi-gradient has constant 
$C = 7$ when $200\kappa^2\epsilon<1$. %

For linear regression (Example~\ref{example.linreg}), we do not need to center $X$, but we do need to impose the following bounded noise condition in addition 
to the bound on $F_1$:
\begin{align}\label{eqn.linreg_noise_F}
F_2(q) = \frac{\bE_q[(Y-\theta(q)^\top X)^2(v^\top X)^2]}{\bE_q[(v^\top X)^2]}\leq \sigma^2,
\end{align}
where $\theta(p) = \argmin_{\theta\in\bR^d} \bE_p[(Y-\theta^\top X)^2]$ is the optimal regression parameters for $p$. The corresponding 
quasi-gradient is
\begin{equation}
    g_2(X; q) = (Y-X^\top \theta(q))^2(v^\top X)^2, \text{ where } v\in\argmax_{v\in\bR^d} \frac{\bE_q[ (Y-X^\top \theta(q))^2(v^\top X)^2]}{\bE_q[(v^\top X)^2]}.
\end{equation}
which we show in Section~\ref{subsec.low_regret_linreg} has $C = 3$ when $64\kappa^3\epsilon<1$.

Other robust estimation problems have been studied such as sparse mean estimation, sparse PCA and moment estimation~\cite{li2017robust, diakonikolas2019outlier, li2018principled}. For most cases we are aware of, we can similarly construct generalized quasi-gradients and obtain efficient algorithms. As a concrete example we exhibit quasig-radients for sparse mean estimation in   Theorem~\ref{thm.generalized_sparse}.

\subsection{Efficient algorithms from generalized quasi-gradients}\label{sec.introalgorithms}

Having constructed (strict) generalized quasi-gradients for several robust estimation problems, we next show that such generalized quasi-gradients 
enable efficient optimization. Specifically, any algorithm with vanishing regret $\sum_t \langle g(q_t), q_t-p_S\rangle  = o(t)$ as $t \to \infty$ converges to an approximate global minimum, assuming $g$ is a strict generalized quasi-gradient. Typically, any online learning algorithm 
will yield vanishing regret, but the robust setting is complicated by the fact that online convergence rates typically depend on the 
maximum norm of the gradients, and an adversary can include outliers that make these gradients arbitrarily large. We provide two 
strategies to handle this: explicit low-regret with na\"{i}ve pruning, and filtering. The first removes large points as a pre-processing 
step, after which we can employ standard regret bounds; the second picks the step size carefully to ensure convergence in 
$O(\epsilon n)$ steps even if the gradients can be arbitrarily large. Both algorithms are a form of gradient descent on $q$ using the generalized quasi-gradients. %

For the explicit low-regret algorithm, after the gradient step we project the distribution back to the set of deleted distribution $\Delta_{n, \epsilon} = \{q\mid \sum_{i=1}^n q_i = 1, 0\leq q_i \leq \frac{1}{(1-\epsilon)n} \}$ after one-step update. This explicitly ensures that $q\in\Delta_{n, \epsilon}$. The performance of explicit low-regret algorithms are analyzed in Lemma~\ref{lem.explicit_condition}. 

For the filter algorithm, we only project the distribution back to the probability simplex $\Delta_n =\{q\mid \sum_{i=1}^n q_i = 1, \forall i\in[n], q_i\geq 0 \}$. We show that if the strict generalized quasi-gradient is coordinate-wise non-negative with appropriate parameters~(Lemma~\ref{lem.filter_condition}), then the algorithm will output some $q$ with $\TV(q, p_S)\leq \epsilon/(1-\epsilon)$. The set of $q$ satisfying this 
property is a supserset of $\Delta_{n,\epsilon}$, but is exactly what we need for statistical inference. 
Our analysis closely follows previous analyses (see e.g.~\citet{li2018principled, jerry2019lecnotes, steinhardt2018robust, diakonikolas2017being}), but we provide tighter bounds 
at several points that lead to better breakdown point. %

Both algorithms converge to approximate global optima, but need different assumptions to achieve fast convergence. 
The explicit low-regret algorithm requires us to identify and remove bad points that can blow up the gradient, which 
is only possible in some settings. The filtering algorithm works if the strict generalized quasi-gradients are non-negative with appropriate parameters, which again only holds in some settings. Together, however, these cover all the settings we need for our analysis.

\paragraph{Concrete algorithmic results.}
Our result for mean estimation with bounded covariance provides an efficient algorithm with  breakdown point $1/2$ and iteration complexity $\epsilon n$. Our analysis  is the first  that achieves both optimal breakdown point and  optimal rate  $\Theta(\sqrt{\epsilon})$ for $\epsilon\leq 1/4$ in this task.

For mean estimation with near identity covariance, the projected gradient algorithm has breakdown point $1/3$ and iteration 
complexity $\tilde{O}(d/\epsilon^2)$ (Theorem~\ref{thm.low_regret_mean_identity_cov}), which improves the iteration complexity of $\tilde{O}(nd^3/\epsilon)$ in the 
concurrent work of \citet{cheng2020highdimensional}, since $\epsilon \geq 1/n$ without loss of generality. The breakdown point is also consistent with the lower bound in Theorem~\ref{thm.informalboundedmean} if we allow arbitrary initialization. %

The generalized quasi-gradients for linear regression immediately yield a filtering algorithm that achieves estimation error 
$O(\epsilon)$ for $\epsilon < 1/(200\kappa^3)$ under certified hypercontractivity (Theorem~\ref{thm.linreg_implicit}), which is optimal and improves over the previous bound of 
$O(\sqrt{\epsilon})$ in~\citet{klivans2018efficient}. We similarly obtain a filtering algorithm for 
joint mean and covariance estimation, which matches the $O(\epsilon^{3/4})$ rate in mean estimation and $O(\sqrt{\epsilon})$ in covariance estimation for $\epsilon \leq 1/(4\kappa^2)$ in \citet{kothari2017outlier} but with a 
simpler %
algorithm (Theorem~\ref{thm.filtering_joint}).

\subsection{Notation and discussion on the corruption model}\label{subsec.corruption_model}

{\bf Notations:} We use $X$ for random variables, $p$ for the population distribution, and $p_n$ for the corresponding empirical distribution from $n$ samples. Blackbold letter $\bbE$ is used for expectation. We write $A\lesssim B$ to denote that $A \leq C B$ for an absolute constant $C$. We let $\mu_p = \bE_p[X]$ and $\Sigma_p = \bE_p[(X-\mu_p)(X-\mu_p)^\top]$ denote the mean and covariance of a distribution $p$. We also use $\mathsf{Cov}(X)$ to denote the covariance of a random variable $X$. %

We use $\mathsf{TV}(p,q) = \sup_{A} p(A) - q(A)$ to denote the total variation distance between $p$ and $q$.  %
We use $ \mathsf{supp}(\cdot)$ to denote the support of a distribution,  $\mathsf{conv}(\cdot)$ to denote convex hull. We say that a distribution $q$ is an \emph{$\epsilon$-deletion} of $p$ if for any set $A$, $q(A)\leq p(A)/(1-\epsilon)$. This implies that $\TV(p,q)\leq \epsilon$ since $\TV(p,q) = \sup_{A} q(A)-p(A) \leq \sup_A \epsilon q(A) \leq \epsilon$. We use $\Delta_n = \{p\mid \sum_{i=1}^n p_i = 1\}$ to denote the probability simplex. 

We write $f(x) = O(g(x))$ for $x\in A$ if   there exists some positive real number $M$   such that $|f(x)|\leq M g(x)$ for all $x\in A$. If $A$ is not specified, we have $|f(x)|\leq M g(x)$ for all $x\in [0,+\infty)$ (thus the notation is non-asymptotic).  We use $\tilde O(\cdot)$ to be the big-$O$ notation ignoring logarithmic factors. 

We discuss our formal corruption model here. In the traditional {finite-sample} total-variation corruption model in~\cite{donoho1982breakdown}, it is assumed that there exists a set of $n$ good samples, and the adversary is allowed to either add or delete an $\epsilon$ fraction of points. In contrast, throughout the paper we assume that there exists a set of $n$ possibly corrupted samples, and a set $S$ of $(1-\epsilon)n$ of them are good samples. The final goal is to estimate some property of the good samples, assuming that the good samples  
satisfy some nice property $F(p_S) \leq \xi$, where $p_S$ is the empirical distribution. %

Although our formulation only allows the adversary to add points, for all the tasks we consider, the property 
$F(p_S) \leq \xi$ is stable under deletions. For instance, an $\epsilon$-deletion of a bounded-covariance distribution also has 
small covariance. Therefore, our results also apply to the total variation setting (additions and deletions) without loss of generality: 
if $S^*$ is the original set of all good points, and $S$ is the remaining set after deletions, then $F(p_S)$ is small whenever 
$F(p_{S^*})$ is small.
This point is shown in more detail in~\citet{steinhardt2018resilience, zhu2019generalized} in the form of a \emph{generalized resilience}
property.   %

Throughout the paper, we only impose assumptions on the true empirical distribution $p_S$ instead of the population distribution. However, 
for all of the assumptions we consider, \citet{zhu2019generalized} show that if they hold for the population distribution then they also 
hold in finite samples for large enough $n$.
The deterministic finite-sample setting frees us from probabilistic considerations and lets us focus on the deterministic optimization 
problem, and directly implies results in the statistical setting via the aforementioned generalization bounds.

\section{Mean estimation: a landscape theory}\label{sec.landscape}

In this section, we study the  landscape of the optimization problem induced by mean estimation. We first show that any first-order stationary point is an approximate global minimum, and then show that it suffices to have an approximate first-order stationary point to guarantee approximate global minimum.

We start by analyzing mean estimation with bounded covariance  (Example~\ref{example.mean_bdd_cov}). 
We consider the  optimization problem of minimizing $\|\Sigma_q\|$ subject to $q\in \Delta_{n,\epsilon}$. 
Since we have the representation 
$\|\Sigma_q\|_2 =  \sup_{v\in\bR^d, \|v\|_2=1} \sum_{i=1}^n q_i(v^\top X_i)^2-(v^\top \mu_q)^2$, 
the optimization problem can be formulated as 
\begin{align}
    \min_{q} \,  & \sup_{v\in\bR^d, \|v\|_2=1}  \sum_{i=1}^n q_i(v^\top X_i)^2-(v^\top \mu_q)^2 \nonumber \\ \label{eqn.KKT_min}
    \text{s.t.}\,  & q \in \Delta_{n, \epsilon} %
\end{align}

While $\Delta_{n, \epsilon}$ is a convex set, the objective function  is non-convex. In the following sections, we show that any stationary point is %
an approximate {global minimum} for this non-convex problem if and only if the corruption level $\epsilon<1/3$. We further show that 
approximate stationary points are also approximate global minima, with slightly worse breakdown point.

\subsection{Stationary points are approximate global minimum}

We first recall the definition of first-order stationary points for locally Lipschitz functions~\citep[Proposition 2.4.3 and Corollary]{clarke1990optimization}~\cite{bian2017optimality, lacoste2016convergence}.
\begin{definition}[First-order stationary points]
Consider the constrained optimization problem $\min_{x\in A} F(x)$, where $A$ is a closed convex set and $F(\cdot): B\mapsto \bR$ is a locally-Lipschitz function with domain $B\supset A$. We say that $x\in A$ is a first-order stationary point if  there exists $g\in \partial F(x)$ such that
\begin{align}\label{eqn.first_order}
    \langle g, x-y  \rangle \leq 0, \forall y \in A,
\end{align}
where $\partial F(x) $ is the  Clarke subdifferential of the function $F(x)$ on $B$, which is defined in Definition~\ref{def.clarke_subdifferential}
\end{definition}

We interpret this definition for the minimization problem~(\ref{eqn.KKT_min}):
\begin{lemma}\label{lem.stationary_quasi_stationary_mean}
If $q$ is a first-order stationary point of \eqref{eqn.KKT_min}, then for any $p\in\Delta_{n, \epsilon}$, there exists some $v \in \mathbf{R}^d, \|v\|_2 = 1$ such that 
\begin{align}\label{eqn.q_mean_bdd_stationary}
    \bE_q[(v^\top(X-\mu_q))^2] \leq  \bE_{p}[(v^\top(X-\mu_q))^2].
\end{align}
Moreover, $v$ is a principal  eigenvector of $\Sigma_q$: $v \in\argmax_{\|v\|_2}   \bE_q[(v^\top(X-\mu_q))^2]$.
\end{lemma}
\begin{proof}
Let $F(q) =  \sup_{v\in\bR^d, \|v\|_2=1}  \sum_{i=1}^n q_i(v^\top X_i)^2-(v^\top \mu_q)^2$ be defined on $\bR^n$.
From Danskin’s formula~\citep[Theorem 10.22]{clarke2013functional}, we know that 
the subdifferential of $F(q)$ with respect to $q_i$ is   %
\begin{align}\label{eqn.KKT_huge_vector}
    \partial_{q_i} F(q) =  (X_i-\mu_q)^\top V (X_i-\mu_q) - \mu_q^\top V\mu_q,  %
  \end{align}
where $V = \sum_{i}\alpha_i v_iv_i^\top$ is a convex combination of supremum-achieving $v_i\in \argmax_{\|v\|_2=1} \sum_{i=1}^n q_i (v^\top(X_i - \mu_q))^2$.   
By taking $x=q$, $y = p$ in~\eqref{eqn.first_order}, we have
\begin{align}
    \bE_q[(X-\mu_q)^\top V(X-\mu_q)]\leq \bE_{p}[(X-\mu_q)^\top V(X-\mu_q)].
\end{align} 
Since the equality holds for a combination of $v_i$, it must hold for some single $v_i$ that maximizes $\bE_q[(v^\top (X-\mu_q))^2]$. Thus we derive the conclusion. %
\end{proof}

Define $p_*$ as the global minimum of the optimization problem~\eqref{eqn.KKT_min}.
By taking $p = p_*$ in  Equation~\eqref{eqn.q_mean_bdd_stationary}, we know that $\|\Sigma_q\|=\bE_q[(v^\top(X-\mu_q))^2]\leq \bE_{p_*}[(v^\top(X-\mu_q))^2]$.  With this condition, 
 we show in the following theorem that any  first-order stationary point  for the minimization problem~(\ref{eqn.KKT_min}) is an approximate global minimum.

\begin{theorem}[Stationary points are approximate global minimum]\label{thm.low_regret_mean_relation}
Assume $\epsilon \in[0, 1/3)$.  Then for any $q\in\Delta_{n, \epsilon}$ that satisfies~\eqref{eqn.q_mean_bdd_stationary},   we have %
\begin{align}\label{eqn.approximation_ratio_mean_bdd}
\|\Sigma_q\|\leq \left(\frac{1-\epsilon}{1-3\epsilon}\right)^2 \|\Sigma_{p_{*}}\|,
\end{align} 
which is tight in that there exists a first-order stationary point $q\in\Delta_{n, \epsilon}$ such that $\|\Sigma_q\| = \left(\frac{1-\epsilon}{1-3\epsilon}\right)^2 \|\Sigma_{p_{*}}\|$ for some set of observations $X_1,X_2,\ldots,X_n$. 
\end{theorem}

\begin{proof}[Proof of Theorem~\ref{thm.low_regret_mean_relation}]
For the supremum achieving $v$ chosen in Lemma~\ref{lem.stationary_quasi_stationary_mean},
we have
\begin{align}
    \|\Sigma_q\| & =  \mathbb{E}_{q}[(v^\top (X-\mu_q)^2)] \\
    & \stackrel{(i)}{\leq} \mathbb{E}_{p_{*}}[(v^\top (X-\mu_q)^2)] \\
    & = \mathbb{E}_{p_{*}}[ (v^\top (X-\mu_{p_{*}})^2) + (v^\top (\mu_q - \mu_{p_{*}}))^2] \\
    & \stackrel{(ii)}{\leq}  \sup_{v\in \mathbf{R}^d, \|v\|_2 \leq 1} \mathbb{E}_{p_{*}}[ (v^\top (X-\mu_{p_{*}})^2)] + \sup_{v\in \mathbf{R}^d, \|v\|_2 \leq 1}(v^\top (\mu_q - \mu_{p_{*}}))^2 \\
    & = \|\Sigma_{p_{*}}\| + \|\mu_q - \mu_{p_{*}}\|^2.
\end{align}
Here (i) comes from Lemma~\ref{lem.stationary_quasi_stationary_mean}, (ii) comes from substituting $v$ with the largest unit-norm vector.
To bound $\|\mu_q-\mu_{p_*}\|$, we introduce the following two lemmas. The first lemma upper bounds $\|\mu_q - \mu_{p^*}\|_2$ in terms 
of $\TV(q, p^*)$, while the second establishes that $\TV(q, p^*)$ is small. These types of results are standard in literature~\citep[Lemma 2.1, Lecture 4]{jerry2019lecnotes},~\citep[Lemma C.2]{zhu2019generalized},~\cite{diakonikolas2017being, steinhardt2018robust, dong2019quantum}. Here we provide the tight results that improve over the existing results:  %
\begin{lemma}\label{lem.mean_modulus}
For any distributions $p, q$ with $\TV(p, q)\leq \epsilon$, we have
\begin{align}
    \|\mu_p - \mu_q\|\leq \sqrt{\frac{\|\Sigma_p\| \epsilon}{1-\epsilon}} + \sqrt{\frac{\|\Sigma_q\| \epsilon}{1-\epsilon}} .
\end{align}
\end{lemma}
\begin{lemma}\label{lem.deletion_TV}
For a distribution $p_n$, suppose that $q \in \Delta_{n,\epsilon_1}$ and $q' \in \Delta_{n,\epsilon_2}$. 
Then, 
\begin{align}
    \TV(q,q') \leq \frac{\max\{\epsilon_2,\epsilon_1\}}{1-\min\{\epsilon_2,\epsilon_1\}}. 
\end{align}
\end{lemma}
The proofs are deferred to Appendices~\ref{proof.mean_modulus} and \ref{sec.proof_lemma_sec2}. 
Since both $q$ and $p_{*}$ are $\epsilon$-deletions of corrupted distribution $p_n$, from Lemma~\ref{lem.deletion_TV}  we have $\TV(q, p_{*})\leq \frac{\epsilon}{1-\epsilon}$. Combining it with Lemma~\ref{lem.mean_modulus}, we have 
\begin{align}
    \|\Sigma_q\|\leq \|\Sigma_{p_{*}}\| + \left(\sqrt{\frac{\|\Sigma_q\|\epsilon}{1-2\epsilon}} + \sqrt{\frac{\|\Sigma_{p_{*}}\|\epsilon}{1-2\epsilon}} \right)^2.
\end{align}
Solving the above inequality on $\|\Sigma_q\|$, we know that when $\epsilon\in[0, 1/3)$,
\begin{align}\label{eqn.tight_ratio}
    \|\Sigma_q\|\leq \left(\frac{1-\epsilon}{1-3\epsilon}\right)^2 \|\Sigma_{p_{*}}\|
\end{align}
We defer the result of  tightness of~\eqref{eqn.tight_ratio} to Appendix~\ref{sec.lower_KKT}, and illustrate the example in Figure~\ref{fig:KKT}.
\end{proof}

Since we always have $\|\Sigma_{p_*}\|\leq \|\Sigma_{p_S}\|$, 
Equation~\eqref{eqn.approximation_ratio_mean_bdd} also implies\footnote{In fact, $p_*$ and $p_S$ are interchangeable throughout the section, i.e. for all the results that relate $q$ and $p_*$, it is also true for $q$ and $p_S$.}   $\|\Sigma_q\|\leq \left(\frac{1-\epsilon}{1-3\epsilon}\right)^2 \|\Sigma_{p_{S}}\|$. The approximation ratio on covariance matrix is exactly tight, which indicates that a $1/3$ corruption level is tight for stationary points to be approximate global minimum: when $\epsilon\geq 1/3$, there exist cases when a local minimum is arbitrarily far from the true mean. We   illustrate this phenomenon in Figure~\ref{fig:KKT}, and provide a formal analysis in Appendix~\ref{sec.lower_KKT}. 

\begin{figure}
    \centering
    \includegraphics[width=\linewidth]{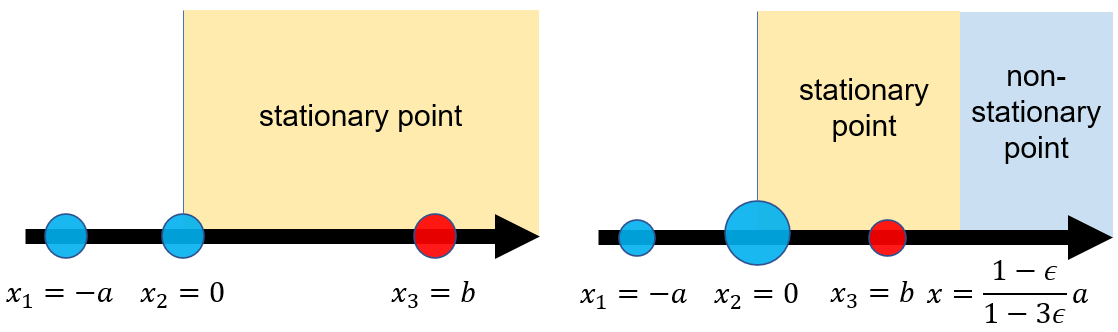}
    \caption{Illustration of the stationary points under different corruption level. The blue points are probability mass on $p_{*}$ (or $p_S$) and the red point is added by adversary. Both $x_1$ and $x_3$ have mass $\epsilon$ and $x_2$ has mass $1-2\epsilon$ in corrupted distribution $p$. Left: When $\epsilon = 1/3$,  the three points share equal probability. As long as $a>0$,   deleting $x_1$ and keeping the rest two points will yield a valid stationary point (in fact, a local minimum). Thus the adversary can drive the mean to infinity. Right: When $\epsilon<1/3$,   deleting $x_1$  only yields a valid stationary point when $b<\frac{1-\epsilon}{1-3\epsilon}\cdot a$. Thus the adversary cannot  create stationary points far from mean.} %
    \label{fig:KKT}
\end{figure}
As a direct  corollary of Theorem~\ref{thm.low_regret_mean_relation} and Lemma~\ref{lem.mean_modulus}, %
since $\TV(q, p_S)\leq \epsilon/(1-\epsilon)$ and both $q$ and $p_S$ have bounded covariance, we can bound the distance between the two means.
\begin{corollary}\label{cor.mean_bdd}
For any $q\in\Delta_{n, \epsilon}$ that satisfies~\eqref{eqn.q_mean_bdd_stationary}, for the true distribution $p_S$ we have
\begin{align}
   \| \mu_q -  \mu_{p_{S}}\|_2 \leq  \sigma\sqrt{\frac{\epsilon}{1-2\epsilon}} +\sigma \sqrt{\frac{\epsilon(1-\epsilon)^2}{(1-3\epsilon)^2(1-2\epsilon)}} = O\left ( \sigma \sqrt{\frac{\epsilon}{1-2\epsilon}} \cdot \frac{1-\epsilon}{1-3\epsilon} \right).
\end{align}
\end{corollary}
This corollary shows that any first-order stationary point of the optimization problem is a near-optimal estimator for mean under bounded covariance assumption up to the ratio $(1-\epsilon)/(1-3\epsilon)$,  since $\sigma \sqrt{\frac{\epsilon}{1-2\epsilon}}$ is the information theoretic limit in mean estimation under the bounded covariance assumption up to universal constants for all $\epsilon\in (0,1/2)$ (see~\citet{donoho1988automatic,zhu2019generalized} and Lemma~\ref{lem.mean_modulus}).

\subsection{Approximate stationary points are approximate global minimum}\label{sec.mean_quasi_stationary}

In the previous section, we show that any stationary point is an approximate global minimum. However,  in practice we cannot find an exact stationary point, but only an approximate stationary point. In this section, we show that even when~\eqref{eqn.q_mean_bdd_stationary} only holds approximately, $q$ is still an  approximate global minimum.   This proves to be  important for generalization to other tasks in Section~\ref{sec.low_regret} and algorithm design in Section~\ref{sec.filter}. 

Concretely, we relax the condition~\eqref{eqn.q_mean_bdd_stationary} to the following for $\alpha, \beta\geq 0$: 
\begin{align}\label{eqn.q_mean_bdd_approximate_stationary}
    \bE_q[(v^\top(X-\mu_q))^2] \leq  (1+\alpha) \bE_{p_*}[(v^\top(X-\mu_q))^2] + \beta,
\end{align}
where $v$ is a principal  eigenvector of $\Sigma_q$: $v \in\argmax_{\|v\|_2}   \bE_q[(v^\top(X-\mu_q))^2]$. Compared with~\eqref{eqn.q_mean_bdd_stationary}, we only require the relationship holds for the global minimum $p_*$ instead of all $p$, and  allow both multiplicative error  $\alpha\bE_{p_*}[(v^\top(X-\mu_q))^2]$ and additive error $\beta$.

In fact,
equation~\eqref{eqn.q_mean_bdd_approximate_stationary} is more general than the traditional definition of approximate stationary point where one requires that $\|\partial_q F(q)\|$ small~\cite{dutta2013approximate, cheng2020highdimensional}.  %
We show that as long as $\|\partial_q F(q)\|$ is upper bounded by some $\gamma\geq 0$, the condition~\eqref{eqn.q_mean_bdd_approximate_stationary} holds with $\alpha = 0$, $\beta = O(\epsilon\gamma)$.

Assume for simplicity that $v$ is the only principal  eigenvector of $\Sigma_q$. Then we have $\partial_{q_i} F(q) = (v^\top(X_i-\mu_q))^2 - (v^\top \mu_q)^2$, and  
\begin{align}
     & \bE_q[(v^\top(X-\mu_q))^2] -  \bE_{p_*}[(v^\top(X-\mu_q))^2]  \nonumber \\ 
     = &     \bE_q[(v^\top(X-\mu_q))^2-(v^\top \mu_q)^2] -  \bE_{p_*}[(v^\top(X-\mu_q))^2-(v^\top \mu_q)^2]
     \nonumber \\ 
     = & \sum_{i\in[n]} (q_i - p_{*, i})\cdot \partial_{q_i} F(q) \nonumber \\ 
      \stackrel{(i)}{\leq} & \sqrt{\sum_{i\in[n]} (q_i - p_{*, i})^2} \cdot \|\partial_{q} F(q)\| \nonumber \\ 
       \stackrel{(ii)}{\leq} & \sqrt{\frac{2\epsilon}{(1-\epsilon)^2n}} \cdot \|\partial_{q} F(q)\| \nonumber \\  \stackrel{(iii)}{\leq} & O(\epsilon \cdot \|\partial_{q} F(q)\|). \label{eqn.relation_tradition_stationary}
\end{align}
Here (i) comes from Cauchy Schwarz, (ii) comes from optimizing over all $q, p_{*, i}\in\Delta_{n, \epsilon}$, (iii) comes from that $1/n\leq \epsilon\leq 1/2$. Thus any approximate stationary point with small $\|\partial_{q} F(q)\|$ will also satisfy~\eqref{eqn.q_mean_bdd_approximate_stationary} with $\alpha = 0, \beta = O(\epsilon \cdot \|\partial_{q} F(q)\|)$. 

Now we show that any   point that satisfies~\eqref{eqn.q_mean_bdd_approximate_stationary} is an approximate global minimum.
\begin{theorem}\label{thm.approximate_mean_small}
Assume $\epsilon \in[0, 1/3)$. Define $p_*$ as the global minimum of the optimization problem~\eqref{eqn.KKT_min}. Assume $q\in\Delta_{n, \epsilon}$ and there exists some $v\in\argmax_{\|v\|\leq 1} \bE_q[(v^\top(X-\mu_q))^2]$ such that
\begin{align}
    \bE_q[(v^\top(X-\mu_q))^2] -  \bE_{p_*}[(v^\top(X-\mu_q))^2] \leq \alpha \bE_{p_*}[(v^\top(X-\mu_q))^2] + \beta.
\end{align}
Then for some universal constants $C_1, C_2$, we have
\begin{align} 
\|\Sigma_q\|\leq \left(1 +  \frac{C_1 (\alpha+\epsilon)}{(1-(3+\alpha)\epsilon)^2}\right) \|\Sigma_{p_{*}}\|+ \frac{C_2\beta}{(1-(3+\alpha)\epsilon)^2}.
\end{align} 
\end{theorem}

We defer the proof to Appendix~\ref{proof.approximate_mean_small}. 
For the task of mean estimation with bounded covariance, we want that $\|\Sigma_q\|\leq C\cdot \|\Sigma_{p_S}\|$. This is satisfied when $\alpha$ is constant and $\beta \lesssim  \|\Sigma_{p_S}\|$, in sacrifice of a smaller breakdown point  $1/(3+\alpha)$ .

\subsection{Application to the case of mean estimation with near identity covariance}\label{sec.id_quasi_gradient}

Beyond bounded covariance, we can make the stronger assumption that the covariance is close to the identity on all 
large subsets of the good data and the distribution has stronger tail bound (e.g. bounded higher moments or sub-Gaussianity). This stronger assumption yields tighter bounds~\cite{diakonikolas2017being, zhu2019generalized}. %
We can adapt our previous landscape analysis to this setting as well.

\begin{example}[Mean estimation with near identity covariance]\label{example.mean_id_cov}
Let $\Delta_{S, \epsilon}= \{r\mid \forall i \in [n], r_i \leq \frac{p_{S,i}}{1-\epsilon}\}$ denote the set of $\epsilon$-deletions on $p_S$. 
We assume that the true distribution $p_S$  has near identity covariance, and its mean is stable under deletions, i.e. the following holds for any $r\in\Delta_{S, \epsilon}$:
\begin{align}
 \|\mu_r - \mu_{p_S}\|\leq \rho,  \|\Sigma_{p_S} - I\|\leq \tau.\nonumber 
    \end{align} 
Our goal is to solve the following feasibility problem for some $\tau'\geq \tau$ that may depend on $\epsilon$, 
\begin{align}
    &\textrm{find}   \qquad  \qquad q \nonumber \\
& \textrm{subject to} \quad q\in \Delta_{n,\epsilon},
  \|\Sigma_q\|\leq 1+\tau'.\label{eqn.id_feasibility}
\end{align}
\end{example}
Once we find such a $q$, we have the following lemma to guarantee mean recovery:
\begin{lemma}\label{lem.empirical_id_abbreviated}
Under the same assumption on $p_S$ as Example~\ref{example.mean_id_cov}, any solution $q$ to the  feasibility problem~\eqref{eqn.id_feasibility} satisfies $\|\mu_p -\mu_q\| = O(\rho+\sqrt{\epsilon(\tau+\tau')}+\epsilon)$. 
\end{lemma}
We defer statement with detailed constants and the proof of the above lemma to Lemma~\ref{lem.empirical_identity_cov_modulus}, where we provide  a tighter analysis   than~\citet[Lemma E.3]{zhu2019generalized}.

The optimization is the same as in~\eqref{eqn.KKT_min}; only the assumptions on $p_S$ are different. 
Applying Theorem~\ref{thm.low_regret_mean_relation}, we therefore know that  the stationary point $q$ satisfies
\begin{align}
    \|\Sigma_q\|_2\leq \left(\frac{1-\epsilon}{1-3\epsilon}\right)^2\|\Sigma_{p_*}\|_2\leq \left(\frac{1-\epsilon}{1-3\epsilon}\right)^2\cdot (1+\tau) \leq 1 + \frac{C(\tau+\epsilon)}{(1-3\epsilon)^2}
 \end{align}
 for some universal constant $C$. 
 Thus we can guarantee that $\tau'$ is close to $\tau$ up to some constant when $\tau\gtrsim \epsilon$. The breakdown point $1/3$ is still tight in this case.  Indeed, the counterexample and argument in Figure~\ref{fig:KKT} and Appendix~\ref{sec.lower_KKT} still applies. %

The result for approximate stationary points in Theorem~\ref{thm.approximate_mean_small} also applies here. We know that any approximate stationary point $q$ that satisfies~\eqref{eqn.q_mean_bdd_approximate_stationary} will satisfy
\begin{align}
    \|\Sigma_q\|\leq \left(1 +  \frac{C_1 (\alpha+\epsilon)}{(1-(3+\alpha)\epsilon)^2}\right) \|\Sigma_{p_{*}}\|+ \frac{C_2\beta}{(1-(3+\alpha)\epsilon)^2} \leq 1 + \frac{C_3(\tau+\epsilon+\alpha+\beta)}{(1-(3+\alpha)\epsilon)^2}
\end{align}
 for some universal constants $C_1, C_2, C_3$. %

We interpret the assumptions and the results in Example~\ref{example.mean_id_cov} under concrete cases as below. Assume the true population distribution is  sub-Gaussian,  we have $\rho = C_1\cdot \epsilon \sqrt{\log(1/\epsilon)}, \tau = C_2\cdot \epsilon \log(1/\epsilon)$~\cite{diakonikolas2016robust, zhu2019generalized, cheng2020highdimensional}. Thus when $ \alpha \lesssim \tau$, $\beta\lesssim \tau $, we know that $\tau'$ is close to $\tau$ up to some constant.  From   Lemma~\ref{lem.empirical_id_abbreviated}   we know that $\|\mu_q - \mu_{p_S}\| = O(\epsilon\sqrt{\log(1/\epsilon)})$. This improves over the bound $O(\sqrt{\epsilon})$ in Corollary~\ref{cor.mean_bdd} since we have imposed stronger tail bound assumption.

Furthermore, if we know that $q$ is an approximate stationary point with $\|\partial_q F(q)\|\lesssim \log(1/\epsilon)$, from~\eqref{eqn.relation_tradition_stationary} we know that $\alpha = 0, \beta\lesssim \epsilon\log(1/\epsilon)$, thus we have  $\|\mu_q - \mu_{p_S}\| = O(\epsilon\sqrt{\log(1/\epsilon)})$. This implies the result in the independent and concurrent work~\citep[Theorem 3.2]{cheng2020highdimensional}. %

\section{From   gradient to generalized quasi-gradient }\label{sec.low_regret}

Given the success of the landscape analysis (Theorem~\ref{thm.low_regret_mean_relation} and~\ref{thm.approximate_mean_small}) for mean estimation, a natural question is  whether the stationary point story holds for other tasks, such as linear regression or joint mean and covariance estimation. %
We might hope that, as for mean estimation, 
all first-order stationary points of minimizing $F(q)$ subject to $q\in \Delta_{n,\epsilon}$ are approximate global minimum. 

We show that this is in general not true. The counterexample is minimizing the hypercontractivity coefficient, which appears as part of linear regression~\cite{zhu2019generalized} and joint mean and covariance estimation~\cite{kothari2017outlier}. The target function $F(q)$ to minimize takes the form of $\sup_{v\in\bR^d} \frac{\bE_{q}[(v^\top X)^4]}{\bE_{q}[(v^\top X)^2]^2}$, as is  defined in~\eqref{eqn.linreg_hyper_F}.  %
We show in Appendix~\ref{sec.KKT_counter} there exist first-order stationary points of minimizing $F(q)$ subject to $q\in \Delta_{n,\epsilon}$ such that its hypercontractivity coefficient is arbitrarily big. %

Instead, we identify a more general property, the existence of \emph{generalized quasi-gradients},  that allows us to handle the new tasks. This also motivates the algorithm design in Section~\ref{sec.filter}, where we use the generalized quasi-gradients as the gradient passed to algorithms. %

\subsection{Generalized quasi-gradients and mean estimation}

In this section, we interpret the result of mean estimation in Section~\ref{sec.landscape} in the lens of generalized quasi-gradients.

Recall that 
in the minimization problem $\min_{q\in \Delta_{n, \epsilon}} F(q)$, the
generalized quasi-gradients, as defined in Definition~\ref{def.generalized_quasi_g}, refers to any $g(X; q)$ such that for all $p,q \in \Delta_{n, \epsilon}$\footnote{Indeed, it suffices to show the implication holds for any $q\in\Delta_{n, \epsilon}$ and $p_S$ instead of any $q, p\in\Delta_{n, \epsilon}$, since we only want to relate $F(q)$ with $F(p_S)$. }:
\begin{align}\label{eqn.sec3_generalized_quasi}
    \bE_q[g(X; q)] - \bE_p[g(X; q)] \leq 0 \implies F(q)\leq C \cdot F(p), 
\end{align}
here $C$ is some constant that may depend on $\epsilon$.  We call it `generalized' quasi-gradient since in the literature, quasi-gradient usually refers to the case when $C=1$ in the implication~\cite{boyd2007subgradient,hazan2015beyond}, which requires the function $F$ to be quasi-convex.
In the case of mean estimation, 
we set the generalized quasi-gradient as
\begin{align}\label{eqn.choice_g_mean}
g(X; q) = (v^\top(X-\mu_q))^2, v\in\argmax_{\|v\|_2\leq 1} \bE_q[(v^\top(X-\mu_q))^2].
\end{align} Theorem~\ref{thm.low_regret_mean_relation} shows that $g(X; q)$ is a valid generalized quasi-gradients for $F(q) = \|\Sigma_q\|$ with approximation ratio $C = (1-\epsilon)^2/(1-3\epsilon)^2$.
For strict generalized quasi-gradient,  the condition in the left of the implication in~\eqref{eqn.implication_strict_generalized_quasi} reduces to~\eqref{eqn.q_mean_bdd_approximate_stationary}.  Theorem~\ref{thm.approximate_mean_small} shows that $g(X; q)$ is also a valid {strict} generalized quasi-gradient with parameter $C_1(\alpha, \beta) = (1+\frac{C_3(\alpha+\epsilon)}{(1-(3+\alpha)\epsilon)^2}), C_2(\alpha, \beta) = \frac{C_4\beta}{(1-(3+\alpha)\epsilon)^2}$ for some universal constant $C_3, C_4$. %

Once we identify a  generalized quasi-gradient $g$ for $F(q)$, it suffices to find some $q$ such that $\bE_q[g] \leq \bE_{p_S}[g]$ to guarantee that $F(q)$ is bounded by $C\cdot F(p_S)$. However, the condition  $\bE_q[g] \leq \bE_{p_S}[g]$ is impossible to check  since we do not know the true distribution $p_S$, 
why is this (strict) generalized quasi-gradient still important?

First, the conditions in strict generalized quasi-gradient can be approached via low-regret algorithms~\cite{arora2012multiplicative, nesterov2009primal}. By viewing $g(X; q)$ as loss, the low-regret algorithms usually provide the guarantee in the form of $\langle g, q-p_S\rangle\leq \alpha\langle  |g|, p_S\rangle + \beta$. Thus if we pick $g$ as a strict generalized quasi-gradient, we know that the output of the low-regret algorithm will guarantee an upper bound on $F(q)$. This is the key idea for algorithm design in Section~\ref{sec.filter}. 

Second, as we have shown in~\eqref{eqn.relation_tradition_stationary}, the condition in strict generalized quasi-gradient 
is more general than traditional approximate stationary points.
Furthermore, as we show throughout this section, stationary point fails to provide a good solution to the optimization problem in general. In the meantime the viewpoint of generalized quasi-gradients enables the flexibility of selecting $g$, and thus succeeds in all the tasks we considered.

Third, 
the termination condition can be based on the function value $F(q)$. As a concrete example, Theorem~\ref{thm.approximate_mean_small} indicates that $\|\Sigma_q\|$ will be small when we reach the condition  $\langle g, q-p_S\rangle\leq \alpha\langle  |g|, p_S\rangle + \beta$. Thus in practice it suffices for us to  check  $\|\Sigma_q\|$ for termination.

Now we show that for the tasks of linear regression and joint mean and covariance estimation, we can identify the generalized quasi-gradients, which enable algorithm design for these tasks in Section~\ref{sec.filter}. 

\subsection{Linear regression }\label{subsec.low_regret_linreg}

To demonstrate the power of generalized quasi-gradients, we first consider robust linear regression:

\begin{example}[Linear regression]\label{example.linreg}
We assume that the true distribution $p_S$  satisfies  $F_1(p_S) \leq \kappa^2, F_2(p_S) \leq \sigma^2$ for  $F_1(q) = \sup_{v\in\bR^d} \frac{\bE_{q}[(v^\top X)^4]}{\bE_{q}[(v^\top X)^2]^2}, F_2(q)=\frac{\bE_q[(Y-\theta(q)^\top X)^2(v^\top X)^2]}{\bE_q[(v^\top X)^2]}$ in~\eqref{eqn.linreg_hyper_F} and~\eqref{eqn.linreg_noise_F}. Our goal is to solve the following feasibility problem for some $\kappa'\geq \kappa, \sigma'\geq \sigma$ that may depend on $\epsilon$, 
\begin{align}
    &\textrm{find}   \qquad  \qquad q \nonumber \\
& \textrm{subject to} \quad q\in \Delta_{n,\epsilon},
   F_1(q) \leq \kappa'^2, F_2(q)\leq \sigma'^2.\label{eqn.linreg_feasibility}
\end{align}
As is shown in~\citet[Example 3.2]{zhu2019generalized}, any $q$ that satisfies the above condition would guarantee a small worst-case excess predictive loss (regression error) of $\Theta(  {(\kappa + \kappa')(\sigma + \sigma'){\epsilon}})$.
\end{example}

The linear regression problem is special in that we need to guarantee both $F_1, F_2$ small simultaneously. In fact, we can do it sequentially: we first solve~\eqref{eqn.linreg_feasibility} without the constraint that $F_2(q)\leq \sigma'^2$, then we treat the output distribution as a new `corrupted' distribution and further delete it such that $F_2(q)$ is small. The hypercontractivity is approximately closed under deletion ~\citep[Lemma C.6]{zhu2019deconstructing}. Thus the distribution will be hypercontractive throughout the second step. 

Now we will design two generalized quasi-gradients $g_1, g_2$ separately. We would like to find $g_1, g_2$ such that 
\begin{itemize}
    \item  $g_1$ is a generalized quasi-gradient for $F_1$, i.e. $\bE_q[g_1]\leq \bE_{p_S}[g_1]$ implies $F_1(q)\leq C F_1(p_S)$ for some $C$ that may depend on $\epsilon$.
    \item under the assumption of hypercontractive on the corrupted distribution $p_n$ (hence on any deletion of it),  $g_2$ is a generalized quasi-gradient for $F_2$, i.e. $\bE_q[g_2]\leq \bE_{p_S}[g_2]$ implies $F_2(q)\leq C F_2(p_S)$ for some $C$ that may depend on $\epsilon$.
\end{itemize}
For hypercontractivity, we take%
\begin{align}\label{eqn.linreg_low_regret_g1}
    g_1(X; q) = (v^\top X)^4, \text{ where } v\in\argmax_{v\in\bR^d} \frac{\bE_q[(v^\top X)^4]}{\bE_q[(v^\top X)^2]^2}.
\end{align}
For bounded noise,    we use the generalized quasi-gradient %
\begin{align}\label{eqn.linreg_low_regret_g2}
    g_2(X; q) = (Y-X^\top \theta(q))^2(v^\top X)^2, \text{ where } v\in\argmax_{v\in\bR^d} \frac{\bE_q[ (Y-X^\top \theta(q))^2(v^\top X)^2]}{\bE_q[(v^\top X)^2]}.
\end{align}

It is in general computationally hard to solve the maximization problem in~\eqref{eqn.linreg_low_regret_g1} related to higher moments. To guarantee the efficiency of computing $g$, we make a stronger assumption of hypercontracitivity under Sum-Of-Squares (SOS) proof. We refer to Appendix~\ref{appendix.sos} for the definition of  SOS proofs. To be precise, we make the hypercontractivity condition stronger in the sense that there exists a sum-of-squares proof for the inequality $\bE_q[(v^\top X)^4]\leq \kappa^2 \bE_q[(v^\top X)^2]^2$, in particular, we let
\begin{align}\label{eqn.SOS_F1}
  \tilde F_1(q) = \sup_{\pseudoE\in\cE_4} \frac{\pseudoE[\bE_{q}[(v^\top X)^4]]}{\pseudoE[\bE_{q}[(v^\top X)^2]^2]},
\end{align}
and assume that $\tilde F_1(p_S) \leq \kappa^2$. 
Here $\cE_4$ is the set of all degree-4 pseudo-expectations on the sphere, and $E_v$ denotes one of the pseudo-expectation with respect to the polynomials of $v$. %
We provide its concrete definition in Definition~\ref{def.pseudoexpectation}.
We call the inequality $\tilde F_1(q)\leq \kappa^2$ \emph{certifiable hypercontractivity}.  In this case, we would like to find some distribution $q$ that is also certifiably hypercontractive.  To guarantee certifiable hypercontractivity, We take $\tilde g_1$   as the pseudo-expectation of $(v^\top X)^4$, i.e.
\begin{align}\label{eqn.linreg_low_regret_g1_pseudo}
    \tilde g_1(X; q) = \pseudoE[(v^\top X)^4], \text{ where } \pseudoE \in \argmax_{\pseudoE\in\cE_4} \frac{\pseudoE[\bE_{q}[(v^\top X)^4]]}{\pseudoE[\bE_{q}[(v^\top X)^2]^2]}.
\end{align}

Now we are ready to show the main result: $\tilde g_1, g_2$ as selected above are generalized quasi-gradients for $\tilde F_1, F_2$.

\begin{theorem}[Generalized quasi-gradients for linear regression]\label{thm.low_regret_linreg}
Assume that $\tilde F_1(p_S)\leq \kappa^2, F_2(p_S)\leq \sigma^2$.  
For any $q\in\Delta_{n, \epsilon}$, when $\kappa^3\epsilon<1/64$, the following implications are true:
\begin{align}
     \bE_q[\tilde g_1(X; q)] \leq    \bE_{p_S}[\tilde  g_1(X; q)] & \Rightarrow \tilde F_1(q)\leq 4\tilde F_1(p_S),  \nonumber \\ 
      \tilde F_1(q)\leq 4 \kappa^2, \bE_q[ g_2(X; q)]  \leq    \bE_{p_S}[  g_2(X; q)] & \Rightarrow  F_2(q)\leq 3F_2(p_S).\nonumber
\end{align}
 Thus $\tilde g_1$ is a generalized quasi-gradients for $ \tilde F_1$. Given that $q$ is certifiably hypercontractive, $g_2$ is a generalized quasi-gradients for $F_2$. 
\end{theorem}

We defer the proof to Appendix~\ref{proof.low_regret_linreg}. 
Theorem~\ref{thm.low_regret_linreg} shows that as long as we can find some $q_1$ that are approximate global minimum for $F_1$, and further delete it to get $q_2$ as a global minimum for $F_2$, we are guaranteed to have a $q_2$ such that both $\tilde F_1(q_2)$ and $ F_2(q_2)$ are small enough, which leads to a near-optimal worst-case regression error\footnote{Although in the second part of further deleting $q_1$,  $\tilde F_1(q_2)$ is guaranteed to be upper bounded by not $4\kappa^2$ but $C \cdot \kappa^2$ for some $C$ that is larger than $4$ and may depend on $\kappa, \epsilon$, we can still achieve the guarantee $F_2(q_2)\leq 3F_2(p_S)$ by substituting the assumption $\kappa^3\epsilon<1/64$ with a tighter assumption. }. %
 
Similarly to mean estimation, $\tilde g_1, g_2$ are also strict generalized quasi-gradients.  We leave the detailed  analysis to Section~\ref{sec.algo_linreg}. %
 
We sketch the proof of $g_1$ in~\eqref{eqn.linreg_low_regret_g1} being the generalized quasi-gradients of $F_1$ in~\eqref{eqn.linreg_hyper_F} as below for the intuition of choosing generalized quasi-gradients. %

 \begin{proof}[Proof sketch of quasi-graident property for $g_1$]

 We would like to show that for $g_1$ in~\eqref{eqn.linreg_low_regret_g1} and $F_1(q) =\sup_{v\in\bR^d} \frac{\bE_{q}[(v^\top X)^4]}{\bE_{q}[(v^\top X)^2]^2}$, we have the following holds for any $p, q\in\Delta_{n, \epsilon}$
 \begin{align}
     \forall q, p\in\Delta_{n, \epsilon}, \bE_q[g_1(X; q)] \leq \bE_p[g_1(X; q)] \Rightarrow F_1(q) \leq C \cdot F_1(p).  
 \end{align}
 
Define $\kappa^2 = F_1(p_S)$ and  $\kappa'^2 = F_1(q)$. Assume that $\kappa' \geq \kappa$, since otherwise we already have the desired result. We know from Lemma~\ref{lem.deletion_TV} that $\TV(p, q) \leq \frac{\epsilon}{1-\epsilon}$.  From~\citet[Lemma C.2]{zhu2019generalized} (or Lemma~\ref{lem.modu_hyper}),    we know that the second moments of the two hypercontractive distributions are multiplicatively close, with the ratio depending on $\kappa'$, i.e.
\begin{align}\label{eqn.multiplicative_close_sec_hyper}
       \bE_q[(v^\top X)^2]^2 \leq \gamma^2    \bE_{p_S}[(v^\top X)^2]^2 ,
\end{align}
where $\gamma^2 = \frac{(1+\sqrt{\epsilon  \kappa^2/(1-2\epsilon)^2})^2}{(1-\sqrt{\epsilon \kappa^2/(1-2\epsilon)^2})^2} $.
Thus we know that for the supremum achieving $v$ picked in $g_1$,
\begin{align}
 \bE_q[ (v^\top X)^4)] & \leq  \bE_{p_S}[(v^\top X)^4)]  \nonumber \\ 
    & \stackrel{(i)}{\leq}   \kappa^2 \bE_{p_S}[(v^\top X)^2)]^2 \nonumber \\   & \stackrel{(ii)}{\leq} \gamma^2\kappa^2  \bE_q[ (v^\top X)^2)]^2  \\ 
    & \stackrel{(iii)}{=}\frac{\gamma^2\kappa^2}{\kappa'^2}  \bE_q[ (v^\top X)^4)].
\end{align}
Here (i) comes from the assumption that $F_1(p_S)\leq \kappa^2$, (ii) comes from~\eqref{eqn.multiplicative_close_sec_hyper}, (iii) comes from that $\kappa'^2 = F_1(q)$. 
By solving the above inequality on $\kappa'$, we have when $9\epsilon  \kappa^2/(1-2\epsilon)^2 \leq 1$,
\begin{align}
\kappa' \leq 2\kappa.
\end{align}
 \end{proof}

\subsection{Joint mean and covariance estimation}\label{sec.joint}

We summarize the setting and target for joint mean and covariance estimation as below.

\begin{example}[Joint mean and covariance estimation]\label{example.joint}

We assume that the true distribution $p_S$ satisfies $F(p_S)\leq \kappa^2$ for $F(q) = \sup_{v\in\bR^d} \frac{\bE_{q}[(v^\top (X-\mu_q))^4]}{\bE_{q}[(v^\top (X-\mu_q))^2]^2}$. Our goal is to solve the following feasibility problem for some $\kappa'\geq \kappa$ that may depend on $\epsilon$: 
\begin{align}
    &\textrm{find}   \qquad  \qquad q \nonumber \\
& \textrm{subject to} \quad q\in \Delta_{n,\epsilon},
   F(q) \leq \kappa'^2.
\end{align}
As is shown in~\citet[Example 3.3]{zhu2019generalized}, any $q$ that satisfies the above condition will guarantee good recovery 
of mean and covariance:
$   \|\Sigma_{p_S}^{-1/2}(\mu_q- \mu_{p_S})\| \leq \Theta( \sqrt{(\kappa+\kappa')}\epsilon^{3/4})$, and 
$   \|I_d - \Sigma_{p_S}^{-1/2}\Sigma_q\Sigma_{p_S}^{-1/2}\|_2\leq \Theta((\kappa + \kappa') \sqrt{\epsilon})$.
\end{example}

Different from the recovery metric in traditional mean estimation and covariance estimation, here we use the metric in transformed space as in~\citet{kothari2017outlier}. For mean the metric is also known as  the Mahalanobis distance.

For the efficiency of computing generalized quasi-gradient, we further restrict the assumption by making it SOS certifiable, as the case for linear regression in Section~\ref{subsec.low_regret_linreg}. We take  $\tilde F(q)$ as the coefficient for \emph{certifiable} hypercontractivity:
\begin{align}\label{eqn.SOS_joint}
  \tilde F(q) =  \sup_{\pseudoE\in\cE_4} \frac{\pseudoE[\bE_{q}[(v^\top (X-\mu_q))^4]]}{\pseudoE[\bE_{q}[(v^\top (X-\mu_q))^2]^2]}.
\end{align}

Making certifiably hypercontractive assumption on empirical distribution instead of population distribution is still reasonable. 
\citet[Lemma 5.5]{kothari2017outlier} shows that when the population distribution is  certifiably  hypercontractive  with parameter $\kappa$ and the sample size $n\gtrsim \frac{(d\log(d/\delta))^2}{\epsilon^2}$, the empirical distribution is  certifiably  hypercontractive  with parameter $ \kappa+4\epsilon$ with probability at least $1-\delta$. 
    
The assumptions and target are very similar to the case of hypercontractivity in linear regression, with the only difference that joint mean and covariance estimation has all moments centered  while hypercontractivity has all moments non-centered. Indeed, the choice of quasi-gradient is also following a similar route as the case of linear regression. 
We identify the generalized quasi-gradient $g$ as 
\begin{align}\label{eqn.joint_low_regret_g}
    g(X; q) = \pseudoE[(v^\top (X-\mu_q))^4], \text{ where } \pseudoE \in \argmax_{\pseudoE\in\cE_4} \frac{\pseudoE[\bE_{q}[(v^\top (X-\mu_q))^4]]}{\pseudoE[\bE_{q}[(v^\top (X-\mu_q))^2]^2]}.
\end{align}
Here $\cE_4$ is the set of all degree-4 pseudo-expectations on the sphere, and $E_v$ denotes one of the pseudo-expectation with respect to the polynomials of $v$.
We show that $g$ is a generalized quasi-gradient for $F$ in the following theorem.

\begin{theorem}[Generalized quasi-gradients for joint mean and covariance estimation]\label{thm.low_regret_joint_relation}
Assume that $\tilde F(p_S)\leq  \kappa^2$. For any $q\in\Delta_{n, \epsilon}$, when $\kappa^2\epsilon<1/200$, the following implication holds: 
\begin{align} 
   \bE_{q}[g(X; q)] \leq   \bE_{p_S}[g(X; q)]  \Rightarrow F(q)\leq 7F(p_S).
\end{align} 
 Thus $g$ is a generalized quasi-gradient  for $F$. %
\end{theorem}
We defer the proof to Appendix~\ref{proof.low_regret_joint}. Similar to the case of hypercontractivity in linear regression, we can extend it to the case of approximate quasi-gradient.  We defer the detailed analysis  to Section~\ref{sec.algo_joint}.

\section{Designing gradient descent algorithms}\label{sec.filter}

From  Section~\ref{sec.landscape} and~\ref{sec.low_regret}, we know that we can approximately solve the minimization problem $\min_{q\in\Delta_{n, \epsilon}} F(q)$ as long as we identify (strict) generalized quasi-gradients $g$ and some $q\in\Delta_{n, \epsilon}$ that satisfies 
\begin{equation}
\label{eq:quasi-stationary}
\bE_{X \sim q}[g(X; q)]\leq \bE_{p_S}[g(X; q)] +\alpha \bE_{p_S}[|g(X; q)|] + \beta
\end{equation}
for small enough $\alpha, \beta$.
In this section, we design algorithms to find such $q$ points given strict generalized quasi-gradients. This yields 
efficient algorithms for all the examples discussed above. 

Any low-regret online learning algorithm will eventually provide us with a $q$ that satisfies \eqref{eq:quasi-stationary}. 
To see this, for a sequence of iterates $q_1, \ldots, q_T$, define the loss function 
$\ell_t(p) = \bE_{X \sim p}[g(X; q_t)]$. Then the average regret relative to the distribution $p_S$ is
\begin{align}
\Regret(p_S) &= \frac{1}{T} \sum_{t=1}^T \ell_t(q_t) - \ell_t(p_S) \\
 &= \frac{1}{T} \sum_{t=1}^T \bE_{q_t}[g(X; q_t)] - \bE_{p_S}[g(X; q_t)]
\end{align}
As long as $\Regret(p_S) \to 0$, the average value of $\bE_{q_t}[g(X; q_t)] - \bE_{p_S}[g(X; q_t)]$ 
also approaches zero, and so eventually \eqref{eq:quasi-stationary} will hold for at least one of the $q_t$.

We make this argument more concrete by providing two low-regret algorithms based on gradient descent. 
The first, which we call the \emph{explicit low-regret} algorithm, takes gradient steps with respect to 
the quasi-gradients $g$, and then projects back onto the constraint set $\Delta_{n, \epsilon}$.
The second algorithm, which we call the \emph{filter} algorithm, instead projects only onto 
the probability simplex $\Delta_n $. Although this is larger than the original 
constraint set, we can guarantee that $\TV(q, p_S)\leq \epsilon/(1-\epsilon)$ 
under certain conditions including that   each coordinate of the generalized quasi-gradient  $g_i, i\in[n]$   is always 
non-negative. This bound on $\TV$ distance is all we needed to guarantee small worst-case error. The filter algorithm is commonly 
used in the literature for mean estimation and beyond (see e.g.~\citet{diakonikolas2017being, li2018principled, jerry2019lecnotes, steinhardt2018robust}). Our result in this paper improves over the breakdown point in the literature, and also provide new results under different settings.

We provide a general analysis of these two algorithms in Section~\ref{subsec.general_explicit_low_regret_algorithm} and~\ref{subsec.general_implicit_low_regret_algorithm}. %
Then we apply the analysis and establish the performance guarantee of the two algorithms for mean estimation with bounded covariance. For simplicity we only show the guarantee for one of the two algorithms in each of the other tasks including linear regression, joint mean and covariance estimation, and mean estimation with near-identity covariance. %

Our results provide the first near-optimal polynomial-time algorithm for linear regression under the hypercontractivity and bounded noise assumption, which improves the rate in~\citet{klivans2018efficient} from $O(\sqrt{\epsilon})$ to $O(\epsilon)$. We also give new efficient algorithms for joint mean and covariance estimation under the same setting as~\citet{kothari2017outlier}. We present an explicit low-regret algorithm for the case of mean estimation with near identity covariance, which improves over the independent and concurrent work~\cite{cheng2020highdimensional} in both iteration complexity (from $O(nd^3/\epsilon)$ to $O(d/\epsilon^2)$) and breakdown point (up to $1/3$).

\subsection{Designing  general explicit low-regret algorithm} \label{subsec.general_explicit_low_regret_algorithm}

We now describe the general framework for the explicit low-regret algorithm, which runs gradient descent using the 
quasigradients and projects onto the set $\Delta_{n,\epsilon}$. 
Pseudocode is provided in Algorithm~\ref{algo:explicit_lowregret}.

\begin{algorithm}[!htbp]
\centering
\caption{Explicit low-regret algorithm($ p_n, \xi$)}
        \begin{algorithmic}
        \small{
\State Input: corrupted empirical distribution $ p_n = \frac{1}{n}\sum_{i=1}^n\delta_{X_i}$, $\xi$.
\State Initialize $q_{i}^{(0)} = 1/n$, $i \in [n]$.
\For{$k = 0, 1, \dots$}
\If {$F(q^{(k)})\leq   \xi$}
    \State \Return $q^{(k)}$
    \Else
    \State Compute $g_i^{(k)} = g(X_i; q^{(k)})$,  $\tilde q_{i}^{(k+1)} = q_{i}^{(k)} \cdot (1- \eta^{(k)}\cdot g_i^{(k)} )$. 
    \State Update $q^{(k+1)} = \mathsf{Proj}^{KL}_{\Delta_{n, \epsilon}}(\tilde q^{(k+1)})  $. 
    \EndIf
\EndFor
}
\end{algorithmic}
\label{algo:explicit_lowregret}
\end{algorithm}

\paragraph{Approaching the approximate global minimum.}
The algorithm has the following regret bound from~\citet[Theorem 2.4]{arora2012multiplicative}:
\begin{lemma}[Regret bound for explicit low-regret algorithm]\label{lem.explicit_mean_bdd}
Denote $g^{(k)}(X) = g(X; q^{(k)})$. In Algorithm~\ref{algo:explicit_lowregret}, assume that $|g_i^{(k)}|\leq B$ for all $k$ and $i$, and $\eta^{(k)} = \eta /(2B)$, $\eta\in[0, 1]$. If the algorithm does not terminate before step $T$, we have\footnote{In~\citet{arora2012multiplicative}, the last term is $ \mathsf{KL}(p_S|| p_n)/T\eta$. One can  upper bound $\mathsf{KL}(p_S|| p_n)=\log(1/(1-\epsilon))$ by $2\epsilon$ when $\epsilon<1/2$.}
\begin{align}
    \frac{1}{T}\sum_{t=1}^T \left(\bE_{ q^{(k)}}[g^{(k)}(X) ] -\bE_{p_S} [g^{(k)}(X)] \right)\leq \frac{\eta}{T} \sum_{k=1}^T \bE_{p_S}[|g^{(k)}(X) |]  + \frac{2B\epsilon}{T\eta}.
\end{align}
\end{lemma} 

The lemma directly implies that there exists some $q^{(t_0)}$ with $t_0\in[T]$ such that 
\begin{align}\label{eqn.low_regret_quasi_algo}
\bE_{q^{(t_0)}}[g^{(t_0)}(X)] - \bE_{p_S}[g^{(t_0)}(X)] \leq \eta \bE_{p_S}[|g^{(t_0)}(X)|] + \frac{2B\epsilon}{T\eta}.
\end{align}
If $g(X; q)$ is a strict generalized quasi-gradient, then it implies that $q$ is an approximate global minimum.

\paragraph{Bounding the iteration complexity}
To bound the iteration complexity, one may adjust the iteration number $T$ and step size $\eta^{(k)}$ in~\eqref{eqn.low_regret_quasi_algo} to get the desired precision. Assume  that we want $F(q^{(k)})\leq \xi$ and the strict generalized quasi-gradient states that
\begin{align}\label{eqn.approximate_quasi_stationary_statement}
    \bE_{q}[g(X; q)] - \bE_{p_S}[g(X; q)] \leq \eta \bE_{p_S}[|g(X; q)|] + \beta \Rightarrow F(q)\leq \xi.
\end{align}
(See e.g.~Theorem~\ref{thm.approximate_mean_small}.)
Then by taking $T = 2B\epsilon/(\eta\beta)$, Equation~\eqref{eqn.low_regret_quasi_algo}     shows that we are able to find $q$ with $F(q)\leq \xi$ within $O(B\epsilon/\eta\beta)$ iterations. 
 
Furthermore, by taking   $\eta^{(k)} = \min(1/(2B), 1/\sqrt{T})$ and letting $T\rightarrow \infty$, we will eventually obtain some $q^{(t_0)}$ that approaches the approximate global minimum. We summarize this analysis  in the  lemma below.

\begin{lemma}[Sufficient condition for the success of explicit low-regret algorithm]\label{lem.explicit_condition}
Assume that  $|g(X_i; q )|\leq B$ for all $ q\in \Delta_{n, \epsilon}$ and $i\in[n]$, and that~\eqref{eqn.approximate_quasi_stationary_statement} holds for the strict generalized quasi-gradient $g$. Take  $\eta^{(k)} = \eta/(2B), \eta\in[0, 1]$.
Then  within $O(B\epsilon/\eta\beta)$ iterations, Algorithm~\ref{algo:explicit_lowregret} will terminate and output a $q$ such that 
$F(q)\leq  \xi, q\in\Delta_{n, \epsilon}$.
\end{lemma}

\subsection{Designing general filter algorithm}\label{subsec.general_implicit_low_regret_algorithm}

Now we introduce the filter algorithm. As noted, the only difference from the explicit low-regret algorithm is that we project onto the probability simplex $\Delta_{n}$ instead of the set of deleted distributions $\Delta_{n, \epsilon}$.  
We recall that projecting to the simplex under KL divergence is equivalent to renormalization\footnote{Here $\mathsf{Proj}^{KL}_{\Delta_{n}}(\tilde q^{(k+1)}) = \argmin_{p\in \Delta_n} \sum_{i=1}^n p_i \ln(\frac{p_i}{\tilde q_i^{(k+1)}})$ is projecting the updated distribution $\tilde q^{(k+1)}$ under Kullback-Leibler divergence onto the probability simplex set $\Delta_n$, which is equivalent to renormalizing $\tilde q^{(k+1)}$ by dividing $\sum_{i=1}^n \tilde q^{(k+1)}_i$.}. Pseudocode is provided in Algorithm~\ref{algo:filtering}.

\begin{algorithm}[!htbp]
\centering
\caption{filter algorithm ($ p_n,  \xi$)}
        \begin{algorithmic}
        \small{
\State Input: corrupted empirical distribution $ p_n = \frac{1}{n}\sum_{i=1}^n\delta_{X_i}$, $\xi$.
\State Initialize $q_{i}^{(0)} = 1/n$, $i \in [n]$.
\For{$k = 0, 1, \dots$}
\If {$F(q^{(k)})\leq  \xi$}
    \State \Return $q^{(k)}$
    \Else
    \State Compute $g_i^{(k)} = g(X_i; q^{(k)})\geq 0$,  $\tilde q_{i}^{(k+1)} = q_{i}^{(k)} \cdot (1- \eta^{(k)}\cdot g_i^{(k)} )$. 
    \State Update $q^{(k+1)} = \mathsf{Proj}^{KL}_{\Delta_{n}}(\tilde q^{(k+1)}) = \tilde q^{(k+1)}/ \sum_{i\in [n]} \tilde q^{(k+1)}_i $. 
    \EndIf
\EndFor
}
\end{algorithmic}
\label{algo:filtering}
\end{algorithm}

\paragraph{Approaching the approximate global minimum.}
The filter algorithm may at first seem strange, since we project onto the larger set $\Delta_n$ and so do not satisfy the 
constraint $q \in \Delta_{n,\epsilon}$. 
The key to analyzing this algorithm is that we can guarantee that $\TV(q, p_S)$ is small implicitly, and this property is all we need 
to guarantee small worst-case estimation error (see e.g. Lemma~\ref{lem.mean_modulus}). 

To bound $\TV(q, p_S)$, 
we keep track of the unnormalized weights, denoted as $c^{(k)}$. Concretely, for any   $ i \in [n]$,  we let $ c_i^{(0)} = 1/n, c_i^{(k+1)} = c_i^{(k)}(1-\eta^{(k)}\cdot g_i^{(k)})$. Then we always have the relationship  $q_i^{(k)} = c_i^{(k)}/(\sum_i c_i^{(k)})$.
We show that $\TV(q, p_{S})\leq \epsilon/(1-\epsilon)$ by establishing the following invariant:
\begin{align}\label{eqn.invariance_filter}
 \sum_{i\in S} (\frac{1}{n}-c_i^{(k)} ) \leq \sum_{i\in [n]/ S} (\frac{1}{n}-c_i^{(k)} ),
\end{align} 
which can be interpreted as saying that we delete more probability mass from  {bad} points $i \in [n]/S$ than from good points $i \in S$. This type of analysis based on~\eqref{eqn.invariance_filter} are well known in the literature~\cite{diakonikolas2017being, li2018principled, jerry2019lecnotes, steinhardt2018robust}. We show in the following lemma that as long as~\eqref{eqn.invariance_filter} holds, it is guaranteed that $\TV(q, p_S)$ is small: 
\begin{lemma}
If $c_i$ satisfies~(\ref{eqn.invariance_filter}) and $c_i \leq \frac{1}{n},$ for all $ i\in[n]$, then the normalized distribution $q_i = c_i/\sum_i c_i$ satisfies $\TV(q,p_S)\leq \epsilon/(1-\epsilon)$. 
\end{lemma}
We defer the detailed statement and proof to Lemma~\ref{lem.qS_and_pS}. 
 To guarantee $c_i\leq \frac{1}{n}$ always holds, we need to impose another assumption that $g(X_i; q^{(k)})$ is always non-negative.

We need to carefully design and balance the threshold $\xi$ in the algorithm such that %

\begin{enumerate}
    \item it is small enough such that  the condition $F(q^{(k)})\leq \xi$ guarantees a good bound when we terminate;%
    \item it is large enough such that if we do not terminate at step $k$ (which means $F(q^{(k)})> \xi$) and the invariance~(\ref{eqn.invariance_filter}) holds at $c^{(k)}$, then in step $k$ we still delete more probability mass from bad points than good points, i.e. the invariance~(\ref{eqn.invariance_filter}) still holds at step $k+1$. Since the deleted probability mass is proportional to $q_i^{(k)}g_i^{(k)}$ for each point $X_i$, it suffices to check that when $F(q^{(k)})>\xi$, we have
\begin{align}\label{eqn.invariance_equivalent}
\sum_{i\in S} q_i^{(k)}g_i^{(k)} \leq \frac{1}{2}\sum_{i=1}^n q_i^{(k)}g_i^{(k)}.
\end{align} 
\end{enumerate}

Although the above argument does not provide a regret bound   explicitly,  the low-regret argument  in Lemma~\ref{lem.explicit_mean_bdd} still applies here. Thus     there must exists %
some $q^{(t_0)}$ that satisfies~\eqref{eqn.low_regret_quasi_algo}. Combined with the fact that $\TV(q, p_S)\leq \epsilon/(1-\epsilon)$ we know that $q^{(t_0)}$ is an approximate global minimum.

\paragraph{Bounding the iteration complexity.}
To bound the iteration complexity,   we choose the largest possible step size
\begin{align}\label{eqn.choice_eta_general}
    \eta^{(k)} = {1}/{g_{\mathrm{max}}^{(k)}}, \text{where } g_{\mathrm{max}}^{(k)} = \max_{i\in[n]} g_i^{(k)},
\end{align}

We can easily bound the iteration complexity from the choice of $\eta^{(k)}$: 
since we set the mass of at least one point to zero each time, the invariance~\eqref{eqn.invariance_filter} implies we will terminate in $O(\epsilon n)$ steps. 
 By the time the algorithm terminates, we must have $F(q)\leq   \xi$, which is the desired result.

By combining the two arguments together, we derive the sufficient condition for the success of Algorithm~\ref{algo:filtering}.

\begin{lemma}[Sufficient condition guaranteeing the success of filter algorithm]\label{lem.filter_condition}
In Algorithm~\ref{algo:filtering}, take $\eta^{(k)}$ as in~\eqref{eqn.choice_eta_general}. Assume that for all $i\in[n]$ and $ k\geq 0$, we have $g_i^{(k)} \geq 0$, and that when $F(q^{(k)})> \xi$ and the invariance~(\ref{eqn.invariance_filter}) holds, (\ref{eqn.invariance_equivalent}) always holds. %
Then Algorithm~\ref{algo:filtering} would output some $q$  within $O(\epsilon n)$ iterations such that $F(q) \leq \xi$.
\end{lemma}

In the above lemma, we require that  (\ref{eqn.invariance_equivalent}) holds when $F(q^{(k)})\geq \xi$ and the invariance~(\ref{eqn.invariance_filter}) holds. We now show that if $g$ is a strict generalized quasi-gradient with appropriate parameters, the implication is satisfied. It suffices to check that under the invariance~(\ref{eqn.invariance_filter}),  $\sum_{i\in S} q_i^{(k)}g_i^{(k)} > \frac{1}{2}\sum_{i=1}^n q_i^{(k)}g_i^{(k)}$ implies $F(q)\leq \xi$. From the invariance~\eqref{eqn.invariance_filter} we know that $q_i\leq \frac{1-\epsilon}{1-2\epsilon} \cdot p_{S, i} $ for all $i\in S$, we have
\begin{align}
   \bE_{q^{(k)}}[g(X; q^{(k)})] =  \sum_{i=1}^n q_i^{(k)}g_i^{(k)}< 2 \sum_{i\in S} q_i^{(k)}g_i^{(k)} \leq \frac{2(1-\epsilon)}{1-2\epsilon} \cdot \bE_{p_S}[g(X; q^{(k)})].
\end{align}
If $g$ is a strict generalized quasi-gradient, from Definition~\ref{def.generalized_quasi_g} we know that the above formula implies $F(q)\leq C_1(1/(1-2\epsilon),0) \cdot F(p_S)+C_2(1/(1-2\epsilon), 0)$. Thus as long as $g$ is coordinate-wise non-negative and a strict generalized quasi-gradient with  $\xi\geq  C_1(1/(1-2\epsilon), 0)\cdot F(p_S) + C_2(1/(1-2\epsilon), 0)$, the filter algorithm would work. 

In the rest of the section, we will apply Lemma~\ref{lem.explicit_condition} and Lemma~\ref{lem.filter_condition} to all the examples in the paper.

\subsection{Application to mean estimation with bounded covariance}\label{subsec.low_regret_mean}

For mean estimation with bounded covariance, we want to minimize $F(q) = \|\Sigma_q\|$. The corresponding  generalized quasi-gradient is $g(X; \mu_q) = (v^\top(X-\mu_q))^2$, where $v\in\argmax_{\|v\|\leq 1} \bE_q[(v^\top(X-\mu_q))^2]$ is any of the supremum-achieving direction (Theorem~\ref{thm.approximate_mean_small}). We show that under this choice of $F, g$, both the explicit and filter algorithm output a near-optimal estimator for the mean.

\subsubsection{Explicit low-regret algorithm}\label{sec.algo_bdd_explicit}

Assume that $\|\Sigma_{p_S}\|\leq \sigma^2$.
To apply Lemma~\ref{lem.explicit_condition}, we need $|g(X_i; q)|\leq B$. 
Although $X_i$ may come from the adversary and thus $g(X_i; q) = (v^\top(X_i-\mu_q^{(k)}))^2$ can be unbounded, a standard ``naive pruning'' procedure (see e.g.~\citet[Lecture 5]{jerry2019lecnotes} and ~\citet{diakonikolas2019robust, dong2019quantum}) yields data such that all points $X_i$ satisfy  
$\|X_i - \mu\| \leq \sigma\sqrt{d/\alpha\epsilon}$   for some $\mu\in\bR^d$  while only deleting $O(\alpha\epsilon)$ fraction of points. Thus we know that $\|X_i - X_j\|\leq 2\sigma\sqrt{d/\alpha\epsilon}$ for any $i, j$. By re-centering the points we can get  $\|X_i\|\lesssim  \sigma\sqrt{d/\epsilon}$.  %
Thus in the later argument, we assume that $(v^\top(X_i-\mu_{q^{(k)}}))^2\leq \sigma^2 {d/\epsilon}$ for any $X_i, q, \|v\|=1$ after  naive pruning. %

After applying this pruning procedure, our low-regret algorithm yields the following guarantee 
for mean estimation: %

\begin{theorem}\label{thm.low_regret_mean_cov}
Assume that $\|\Sigma_{p_S}\|\leq \sigma^2$ and   $\|X_i\|\leq \sigma\sqrt{d/\epsilon}/2, \forall i \in [n]$. Take (describe $F$ and $g$).
For any $\gamma \in (0,1)$, instantiate Algorithm~\ref{algo:explicit_lowregret} with fixed 
step size $\gamma \cdot \frac{\epsilon}{2\sigma^2 d}$, and set
\begin{equation}
\xi = \Big(\frac{2\eta+7}{3(1 - (3+\eta)\epsilon)}\Big)^2\cdot \sigma^2.
\end{equation}
$\eta^{(k)} = \eta \cdot \epsilon/(2\sigma^2 d), \eta\in[0, 1]$, $F(q) = \|\Sigma_q\|$,  $g(X; \mu_q)$ as in~\eqref{eqn.choice_g_mean}\footnote{For the sake of computation efficiency, it suffices to find any $v$ such that $\bE_q (v^\top (X_i-\mu_{q^{(k)}}))^2\geq 0.9\|\Sigma_q\|_2$,  which can be achieved by power method within $O(\log(d))$ time. One can see from the later proof that this will only affect the final bound up to some constant factor. In the rest of the paper, all the $\argmax$ in the algorithm can be substituted with this feasibility problem on $v$ for computational efficiency. }.  
Then for $\epsilon \in (0, 1/(3+\eta))$,   Algorithm~\ref{algo:explicit_lowregret} will output some $q\in\Delta_{n, \epsilon}$ within $d/(\gamma\sigma^2)$ iterations such that 
\begin{align}
   \|\Sigma_q\|\leq \xi.
\end{align}
\end{theorem}
The conclusion also holds if we have arbitrary initialization $q^{(0)}$ instead of uniform. 
The result follows directly from Lemma~\ref{lem.explicit_condition}; see Appendix~\ref{sub.proof_mean_cov} for proof.
We also show there that the computational complexity within each iteration is $O(nd\log(d))$.  

Combining the result with Lemma~\ref{lem.mean_modulus}, we know the output $q$ satisfies
\begin{align}
    \|\mu_q-\mu_{p_S}\|\leq \sigma\cdot \sqrt{\frac{ \epsilon}{1-\epsilon}} + \frac{(2\eta+7)\sigma}{3(1 - (3+\eta)\epsilon)}\cdot\sqrt{\frac{ \epsilon}{1-\epsilon}}.
\end{align}

When $\eta \rightarrow 0$, the breakdown point approaches $1/3$, which is consistent with the result in Theorem~\ref{thm.low_regret_mean_relation}. 

\subsubsection{Filter algorithm}
Different from explicit low-regret algorithm, pre-pruning is not required for filter algorithm. 
We present the guarantee for filter algorithm in the following theorem. %
\begin{theorem}[Filter algorithm achieves optimal breakdown point and near-optimal error]\label{thm.filtering_mean_cov}
In Algorithm~\ref{algo:filtering}, take  $\xi =  \frac{2(1-\epsilon)}{(1-2\epsilon)^2} \cdot \sigma^2$, $\eta^{(k)}$ as in~\eqref{eqn.choice_eta_general}, $F(q) = \|\Sigma_q\|$,  $g(X; \mu_q) = (v^\top(X-\mu_q))^2$, where $v\in\argmax_{\|v\|\leq 1} \bE_q[(v^\top(X-\mu_q))^2]$ is any of the supremum-achieving direction.  Then for $\epsilon \in [0, 1/2)$,  Algorithm~\ref{algo:filtering} will output some $q$ within $O(\epsilon n)$ iterations such that 
\begin{align}
  \TV(q, p_S)\leq \frac{\epsilon}{1-\epsilon}, \|\Sigma_q\|\leq \xi.
\end{align}
\end{theorem}

This result follows directly from Lemma~\ref{lem.filter_condition} and we defer the proof to Appendix~\ref{proof.filtering_mean_cov}.  The conclusion only holds if we have uniform initialization.

Combining the result with Lemma~\ref{lem.mean_modulus}, we know the output $q$ satisfies
\begin{align}
    \|\mu_q-\mu_{p_S}\|\leq \sigma\cdot \sqrt{\frac{ \epsilon}{1-\epsilon}} + \sigma\cdot \frac{\sqrt{2\epsilon}}{1-2\epsilon}.
\end{align}

\paragraph{Breakdown points and optimality.}

The filter algorithm   achieves near-optimal error and {optimal} breakdown point $1/2$~\citep[Page 270]{rousseeuw1987robust} under good initialization, while the explicit low-regret algorithm has breakdown point $1/3$ with arbitrary initialization. It may sound confusing given the negative results that the stationary point may be a bad estimate when $\epsilon\geq 1/3$ in the Section~\ref{sec.landscape}. 
Indeed, if one does not initialize properly, both approaches would definitely fail for $\epsilon\geq 1/3$, and the filter algorithm may fail even for much smaller $\epsilon$. In some sense, the success of the filter algorithm is due to appropriate initialization at uniform distribution and the landscape.  We believe that one can show the breakdown point of the explicit low-regret algorithm initialized at uniform is also $1/2$ with a better analysis.
Theorem~\ref{thm.filtering_mean_cov} also improves over previous analyses~\cite{steinhardt2018robust,jerry2019lecnotes} of filtering algorithms in breakdown point and obtains the sharper rate. 

However, the exact rate obtained in the theorem is still not optimal. 
If we solve the minimum distance functional in~\citet{donoho1988automatic} by minimizing $\TV$ distance between $q$ and the set of distributions with covariance bounded by $\sigma^2$, the error would be $\sqrt{\frac{8\sigma^2{\epsilon}}{1-2\epsilon}}$~\citep[Example 3.1]{zhu2019generalized}, which is significantly smaller than what Theorem~\ref{thm.low_regret_mean_relation} and Theorem~\ref{thm.filtering_mean_cov} achieve  when $\epsilon$ is close to $1/2$. It remains an open problem whether there exists an efficient algorithm that achieves error $O(\sqrt{\frac{\sigma^2{\epsilon}}{1-2\epsilon}})$ for all $\epsilon\in (0,1/2)$. %

\subsection{Application to linear regression}\label{sec.algo_linreg}

Under the same setting as in Section~\ref{subsec.low_regret_linreg}, we would like to find $q\in\Delta_{n, \epsilon}$ such that $\tilde F_1(q)\leq \kappa'$ and $F_2(q)\leq  \sigma'^2$, where  
$\kappa'$ and $\sigma'$ are parameters to be specified later. 
Due to computational consideration, we relax the choice of maximizer in generalized quasi-gradient in Section~\ref{subsec.low_regret_linreg}, and take $g_1, g_2$ as %
\begin{align}
   g_1(X; q) = \pseudoE[(v^\top X)^4], & \text{ where } \pseudoE \in \cE_4 \text{ satisfies }  \pseudoE[\bE_q[(v^\top X)^4]]\geq \kappa'^2 \pseudoE[\bE_q[(v^\top X)^2]^2], \nonumber 
\\
    g_2(X; q) = (Y-X^\top \theta(q))^2(v^\top X)^2, & \text{ where }  {v\in\bR^d} \text{ satisfies }  {\bE_q[ (Y-X^\top \theta(q))^2(v^\top X)^2]} \geq \sigma'^2{\bE_q[(v^\top X)^2]}.\label{eqn.implicit_choice_g_linreg}
\end{align}

To guarantee that both $\tilde F_1$ and $F_2$ small simultaneously, we will check both and update $q$ sequentially within the filter algorithm. We summarize the algorithm in Algorithm~\ref{algo:linreg_implicit}, and the guarantee in Theorem~\ref{thm.linreg_implicit}\footnote{An initial version of the regression result appeared in unpublished lecture notes~\cite{steinhardt2019lecnotes}.}.

\begin{theorem}[Filter algorithm for linear regression]\label{thm.linreg_implicit}
Under the same setting as Theorem~\ref{thm.low_regret_linreg},  in Algorithm~\ref{algo:linreg_implicit}, take $\eta^{(k)} = 1/g_{\mathrm{max}}^{(k)}$, where $g_{\mathrm{max}}^{(k)} = \max_{i} g_i^{(k)}$. Take $\tilde F_1, F_2$ as in~\eqref{eqn.SOS_F1} and~\eqref{eqn.linreg_noise_F}, $g_1, g_2 $ as  in~\eqref{eqn.implicit_choice_g_linreg}, $\kappa'^2 = \frac{2\kappa^2}{1-2\kappa^2\epsilon}$, $\sigma'^2 = \frac{4\sigma^2(1 +2\kappa'\sqrt{\epsilon(1-\epsilon)})}{(1-2\epsilon)^3-20\kappa'^3\epsilon(1-\epsilon)}$. Then Algorithm~\ref{algo:linreg_implicit} will output $q$ within $O(\epsilon n)$ iterations such that
\begin{align}
     \tilde F_1(q) \leq  \kappa'^2, F_2(q)\leq \sigma'^2, \TV(q, p_S)\leq \frac{\epsilon}{1-\epsilon}.
\end{align}
\end{theorem}
We defer the proof to Appendix~\ref{proof.linreg_implicit}.  Combining the theorem with the regression error bound in Example~\ref{example.linreg}, we know that the output satisfies
\begin{align}
    \bE_{p_S}[(Y-X^\top \theta(q))^2] - \bE_{p_S}[(Y-X^\top \theta(p_S))^2] \leq O( (\kappa+\kappa')(\sigma+\sigma') \epsilon).  
\end{align}
Our result provides the first near-optimal and polynomial time algorithm for linear regression under the assumption of hypercontractivity and bounded noise, which improves the rate in~\citet{klivans2018efficient} from $O(\sqrt{\epsilon})$ to $O(\epsilon)$.

\begin{algorithm}[!htbp]
\centering
\caption{filter algorithm for linear regression ($ p_n,  \kappa', \sigma'$)}
        \begin{algorithmic}
        \small{
\State Input: corrupted empirical distribution $ p_n = \frac{1}{n}\sum_{i=1}^n\delta_{X_i}$, $\xi'$.
\State Initialize $q_{i}^{(0)} = 1/n$, $i \in [n]$.
\For{$k = 0, 1, \dots$}
\If {$F_1(q^{(k)})\geq   \kappa'^2$}
    \State Compute $g_i^{(k)} = g_1(X_i; q^{(k)}) $.
    \Else \If {$F_2(q^{(k)})\geq  \sigma'^2$}
      \State Compute $g_i^{(k)} = g_2(X_i; q^{(k)}) $.
    \Else
    \State \Return $q^{(k)}$.
    \EndIf
    \EndIf
    \State Compute  $\tilde q_{i}^{(k+1)} = q_{i}^{(k)} \cdot (1- \eta^{(k)}\cdot g_i^{(k)} )$. 
    \State Update $q^{(k+1)} = \mathsf{Proj}^{KL}_{\Delta_{n}}(\tilde q^{(k+1)}) = \tilde q^{(k+1)}/ \sum_{i\in [n]} \tilde q^{(k+1)}_i $. 
\EndFor
}
\end{algorithmic}
\label{algo:linreg_implicit}
\end{algorithm}

\subsection{Application to joint mean and covariance estimation}\label{sec.algo_joint}

Under the setting of joint mean and covariance estimation with sum-of-squares condition in Section~\ref{sec.joint},    
we take the generalized quasi-gradient as 
\begin{align}\label{eqn.choice_g_joint}
    g(X; q) = \pseudoE[(v^\top (X-\mu_q))^4], \text{ where } \pseudoE \in\cE_4 \text{ satisfies }   {\pseudoE[\bE_{q}[(v^\top (X-\mu_q))^4]]}\geq \kappa'^2{\pseudoE[\bE_{q}[(v^\top (X-\mu_q))^2]^2]}.
\end{align}
Then  the filter algorithm in Algorithm~\ref{algo:filtering} will guarantee a good solution $q$. 
\begin{theorem}[Filter algorithm for joint mean and covariance estimation]\label{thm.filtering_joint}
Under the same setting as Theorem~\ref{thm.low_regret_joint_relation}, in Algorithm~\ref{algo:filtering}, take $\eta^{(k)} = 1/g_{\mathrm{max}}^{(k)}$, where $g_{\mathrm{max}}^{(k)} = \max_{i} g_i^{(k)}$. Take $F, g$ as in~\eqref{eqn.SOS_joint} and~\eqref{eqn.choice_g_joint}, $\kappa'= 7\kappa$. Assume $\kappa^2\epsilon\leq 1/4$. 
Then  Algorithm~\ref{algo:filtering} 
will output some $q$ within $O(\epsilon n)$  iterations such that  
\begin{align}
   F(q)\leq  \kappa'^2, \TV(q, p_S)\leq \frac{\epsilon}{1-\epsilon}.
\end{align}
\end{theorem}
We defer the proof to Appendix~\ref{proof.filtering_joint}. Combining the theorem with the  error bound in Example~\ref{example.joint}, we know that the output satisfies
\begin{align}
    \|\Sigma_{p_S}^{-1/2}(\mu_q- \mu_{p_S})\| &\leq \Theta( \sqrt{(\kappa+\kappa')}\epsilon^{3/4}) \nonumber \\ 
   \|I_d - \Sigma_{p_S}^{-1/2}\Sigma_q\Sigma_{p_S}^{-1/2}\|_2&\leq \Theta((\kappa + \kappa') \sqrt{\epsilon}) \nonumber .
\end{align}
Our result provides a new  efficient algorithm for joint mean and covariance estimation under the same setting as~\citet{kothari2017outlier}.

\subsection{Application to mean estimation with near identity covariance}\label{subsec.algo_id}

Under the setting of mean estimation with near identity covariance (Example~\ref{example.mean_id_cov}), we assume  the following holds for any $r\in\Delta_{S, \epsilon}$:
\begin{align}\label{eqn.assumption_p_S}
 \|\mu_r - \mu_{p_S}\|\leq \rho,  \|\Sigma_{p_S} - I\|\leq \tau,
 \end{align} 
 and would like to find some $q $ such that $\TV(q, p_S)\leq \frac{\epsilon}{1-\epsilon}, F(q) = \|\Sigma_q\|\leq  1 + C \cdot \tau$.
It is shown in Section~\ref{sec.id_quasi_gradient} that we can take the quasi-gradient $g$ the same as the case of bounded covariance.

We present an explicit low-regret algorithm for the case of mean estimation with near identity covariance.  For better bound of iteration complexity, 
 we choose a slightly different   generalized quasi-gradient $g$ as
\begin{align}\label{eqn.choice_g_mean_id_algo}
    g(X; q) = (v^\top(X-\mu_q))^2 -1, \text{ where } v\in\bR^d \text{ satisfies } \bE_q[(v^\top(X-\mu_q))^2] \geq (1-\gamma)\|\Sigma_q\|,
\end{align}
where $\gamma \in(0, 1)$ is the desired precision, and $v$ can be found via power method within $O(\log(d)/\gamma)$ time.
Here we lose the property that $g(X_i; q)\geq 0$. Thus Lemma~\ref{lem.filter_condition}  for filter algorithm does not apply directly. However, we can still run explicit low-regret algorithm.

The explicit low-regret algorithm will still work even if we take $g(X; q) = (v^\top (X-\mu_q))^2$, since as $T\rightarrow \infty$, we can find some $q$ that satisfies the condition  $ \bE_q[(v^\top (X-\mu_q))^2] \leq \bE_{p_S}[(v^\top (X-\mu_q))^2] +\gamma$ for arbitrarily small $\gamma$.  The choice of adding $-1$ in the generalized quasi-gradient is only due to the consideration of iteration complexity, which will be elaborated in the proof.

Since we know that $p_S$ has bounded covariance, we assume that $\|X_i\|_2 \leq \sqrt{d/\epsilon}$ after  the same naive filtering method in Section~\ref{sec.algo_bdd_explicit}. Then we have the following result for mean estimation.

\begin{theorem}[Explicit low-regret algorithm for mean estimation with near identity  covariance]\label{thm.low_regret_mean_identity_cov}
Assume that $p_S $ satisfies~\eqref{eqn.assumption_p_S}, and $\forall i\in[n], \|X_i\|\leq \sqrt{d/\epsilon}$. 
In Algorithm~\ref{algo:explicit_lowregret}, take $\eta^{(k)} = \frac{\beta \tau}{1+\tau/2} \cdot \frac{\epsilon}{8d}, \beta\in(0,1)$, $\xi = 1+C_1 \cdot \frac{\tau+\epsilon \rho^2+\epsilon}{(1-3(1+\beta \tau/(1-\gamma\epsilon))\epsilon )^2} $ for some universal constant $C_1$, $F(q) = \|\Sigma_q\|$, $g$ as in~\eqref{eqn.choice_g_mean_id_algo}.
Then  Algorithm~\ref{algo:explicit_lowregret} will output some $q\in\Delta_{n, \epsilon}$ within $O(d/ \tau^2)$ iterations such that
\begin{align}
   \|\Sigma_{q}\|_2  \leq 1+ C \cdot \frac{(1+1/\beta)\tau+  \rho^2+\epsilon}{(1-3(1+\beta \tau/(1-\gamma\epsilon))\epsilon)^2}.  
\end{align}
\end{theorem}

We defer the proof to Appendix~\ref{proof.low_regret_mean_identity_cov}.
 Combining the result with Lemma~\ref{lem.empirical_id_abbreviated}, we know that the output satisfies
\begin{align}
    \|\mu_{p_S} -\mu_q\| = O(\rho+\frac{\sqrt{\epsilon((1+1/\beta)\tau+\rho^2+\epsilon)}}{(1-3(1+\beta \tau/(1-\gamma\epsilon))\epsilon)}).
\end{align}
As $\beta, \gamma \rightarrow 0$, we can see the breakdown point of the algorithm approaches $1/3$, which is tight and consistent with the result in Section~\ref{sec.id_quasi_gradient}. 

\paragraph{Application of the guarantee. }
The result applies to two different cases: true distribution as either sub-Gaussian distribution with identity covariance or bounded $k$-th moment with identity covariance.

\begin{enumerate}
    \item 
If $p_S$ is the empirical distribution from a sub-Gaussian distribution, and  the sample size satisfies $n\gtrsim d/(\epsilon\log(1/\epsilon))$, then $\rho = C_1\cdot \epsilon\sqrt{\log(1/\epsilon)}$, $\tau = C_2 \cdot \epsilon {\log(1/\epsilon)}$ for some univeresal constants $C_1$, $C_2$ (see e.g.~\citet[Lemma E.11]{zhu2019generalized}, \citet[Lemma 4.4]{diakonikolas2019robust}). Thus the algorithm will output $q$ such that $\|\mu_q - \mu_{p_S}\|_2 \leq C_3 \cdot \epsilon \sqrt{\log(1/\epsilon)}$ for some universal constant $C_3$. 

In this case, our result improves over the concurrent and independent work in~\citet{cheng2020highdimensional}  in both iteration complexity and breakdown point. The best iteration complexity in~\citet{cheng2020highdimensional} is $O(nd^3/\epsilon)$, while our algorithms achieves $O(d/\tau^2)$. In the case of sub-Gaussian distributions, $\tau = C\cdot \epsilon\log(1/\epsilon)$, and $\epsilon\geq 1/n$, thus we always have $d/\tau^2 \gtrsim nd^3/\epsilon$.
 
\item  
If $p_S$ is the empirical distribution from a distribution with bounded $k$-th moment, and   the sample size satisfies $n\gtrsim d\log(d)/\epsilon^{2-2/k}$, then $ \rho = C_1 \epsilon^{1-1/k}$, $\tau = C_2\epsilon^{1-2/k}$ for some constant $C_1, C_2$ (see~\citet[Theorem 5.6]{zhu2019generalized} ). Thus the algorithm will output $q$ such that  $\|\mu_q - \mu_{p_S}\| \leq C_3 \epsilon^{1-1/k}$ for some universal constants $C_3$. 

\end{enumerate}

\paragraph{Designing filter algorithm.}
For mean estimation with near identity covariance, it is well known that the filtering algorithm can work under a different choice of objective function $F(q)$, which only considers the second moment on the $2\epsilon$ tail of the points~(see e.g.~\cite{diakonikolas2017being, li2018principled, jerry2019lecnotes}).
However, running filter algorithm using the generalized quasi-gradient in Section~\ref{sec.id_quasi_gradient} would fail since  it is only able to guarantee that $\|\Sigma_q\| \leq C(\epsilon) \cdot \|\Sigma_{p_S}\|$ for some $C(\epsilon)$ that cannot approach $1$ when $\epsilon$ is small~(Theorem~\ref{thm.filtering_mean_cov}). Indeed, $C(\epsilon) \to 2$ as $\epsilon\to 0$. However, it is fine to have constant $C(\epsilon)$ in $\|\Sigma_q\|_2 - 1 \leq  C(\epsilon)\cdot  |\|\Sigma_{p_S}\| - 1|$ instead. The failure of naively running filtering algorithm is also discussed  in~\citep[Lecture 7]{jerry2019lecnotes}. From another point of view, Lemma~\ref{lem.filter_condition} shows that filtering in general cannot guarantee that the constant $C(\epsilon)$ approaches 1 as $\epsilon \to 0$, so the success of the upper tail filtering algorithms in~\cite{diakonikolas2017being, li2018principled, jerry2019lecnotes} may be explained by constructing a new $F(q)$ such that a constant approximation ratio gives good estimation error.

\section{Conclusion}

In this paper, we investigate why the feasibility problem (Problem~\eqref{prob.feasibility}) can be efficiently solved, which was the target of essentially all computationally efficient robust estimators in high dimensions. Based on the insights, we are able to develop new algorithms for different robust inference tasks.

We start from exploring the landscape  for mean estimation. We show that any approximate stationary point  is an approximate global minimum for the associated minimization   problem under either bounded covariance assumption or near identity covariance assumption with stronger tail bounds. 

We then generalize the insights from mean estimation to other tasks. We  identify generalized quasi-gradients for different tasks. Based on the generalized quasi-gradients,  we design algorithms to approach approximate  global minimum for a variety of tasks, which produces efficient  algorithms for mean estimation with bounded covariance which is near optimal in both rate and breakdown point, 
first polynomial time and near-optimal algorithm  for linear regression under hypercontractivity assumption, and new efficient algorithm for joint mean and covariance estimation. Our algorithm also improves both the breakdown point and computational complexity for the task of mean estimation with near identity covariance. 

Beyond the questions we investigated, 
the framework applies to a large family of robust inference questions, including sparse mean estimation and sparse linear regression. The following steps may be followed to deal with a new robust statistics problem: first identify the bound on the worst case error using modulus of continuity~\cite{donoho1988automatic,  zhu2019generalized}, then formulate an approximate MD problem in the form of Problem~\ref{prob.feasibility}, at last, identify the efficiently computable
generalized quasi-gradients for the approximate MD problem, and approach the approximate global minimum using either explicit low-regret or implicit low-regret algorithm.

\bibliographystyle{plainnat}
\bibliography{di}
\newpage

\appendix

\section{Omitted definitions and notations}
\label{appendix.sos}

\begin{definition}[{Clarke subdifferential~\citep[Chapter 2]{clarke1990optimization}}]\label{def.clarke_subdifferential}
We work in a Banach space $X$. Let $Y$ be a subset of $X$. For a given function $f: Y\mapsto \bR$ that is locally Lipschitz near a given point $x\in X$, let  $f^\circ(x; v)$ denote its generalized directional derivative at $x$ in the  direction $v\in X$:
\begin{align}
    f^\circ(x; v) = \limsup_{y\rightarrow x, t\rightarrow 0+} \frac{f(y+tv)-f(y)}{t}.
\end{align}
Consider the dual space $X^*$. The Clarke subdifferential of $f$ at $x$, denoted $\partial f(x)$, is the subset of $X^*$ given by
\begin{align}
    \{\xi \in X^* \mid f^\circ(x; v)\geq \langle \xi, v \rangle \text{ for all } v \in X \}.
\end{align}
\end{definition}

\begin{definition}[Sum-of-squares (SOS) proof]
For any two polynomial functions $p(v), q(v)$ with degree at most $d$, we say $p(v)\geq q(v)$ has a degree-$d$ sum-of-squares proof if there exists some degree-$d/2$ polynomials $r_i(v)$ such that
\begin{align}
    p(v)-q(v) = \sum_{i} r_i^2(v),
\end{align} we denote it as
\begin{align}
    p(v)\succeq_{sos} q(v).
\end{align}
\end{definition}
\begin{definition}[Certifiable $k$-hypercontractivity]
We say a $d$-dimensional random variable $X$ is \textit{certifiably} $k$-hypercontractive with parameter $\kappa$ if there exists a degree-$2k$ sum-of-squares proof for the $k$-hypercontractivity condition, i.e. 
\begin{equation}
\bE_{p}[(v^\top X)^{2k}] \psos  (\kappa\bE_{p}[(v^\top X)^2])^k,
\end{equation}
\end{definition}

We will also need to introduce one additional piece of sum-of-squares machinery, called \emph{pseudoexpectations on the sphere}:
\begin{definition}[pseudoexpectation on the sphere]\label{def.pseudoexpectation}
A degree-$2k$ pseudoexpectation  on the sphere is a linear map $E$ from the space of degree-$2k$ polynomials to $\bR$ satisfying 
the following three properties:
\begin{itemize}
\item $E[1] = 1$ (where $1$ on the LHS is the constant polynomial).
\item $E[p^2] \geq 0$ for all polynomials $p$ of degree at most $k$.
\item $E[(\|v\|^2-1)p] = 0$ for all polynomials $p$ of degree at most $k$.
\end{itemize}
We let  $\cE_{2k}$ denote the set of degree-$2k$ pseudoexpectations on the sphere. 
\end{definition}
The space $\cE_{2k}$ can be optimized over efficiently, because it has a separation oracle expressible as a sum-of-squares program. 
Indeed, checking that $E \in \cE_{2k}$ amounts to solving the problem $\min\{ E[p] \mid p \sos 0\}$, which is a sum-of-squares program 
because $E[p]$ is a linear function of $p$. Throughout the paper, we use $\pseudoE$ to denote the pseudoexpectation with respect to $v$.

\section{Connections with classical literature}\label{sec.classicalliter}

In this section, we discuss the progress of robust statistics and the connections to our paper. 
For problems such as mean and covariance estimation, where the loss function $L(p_S, \hat{\theta}(p_n))$ takes a special form $\|\theta(p_S)- \hat{\theta}(p_n)\|$ for some norm $\| \cdot \|$, the task of robust estimation is usually decomposed to two separate goals: bounding the \emph{maximum bias} and being \emph{Fisher-consistent}. The maximum bias measures $\| \hat{\theta}(p_n) -\hat{\theta}(p_S) \|$ over the worst case corruption $p_n$, while Fisher-consistency means that the estimator's output given the real distribution, $\hat{\theta}(p_S)$, is exactly the same as the real parameter one wants to compute given $p_S$: $\theta(p_S)$. 

Checking Fisher-consistency may be doable, but bounding the maximum bias proves to be challenging for various estimators.  \cite{huber1973robust,donoho1982breakdown,donoho1992breakdown,chen2002influence,chen2018robust,zhu2020does} analyzed the maximum bias for the Tukey median, while~\citet{davies1992asymptotics} analyzed that for the \emph{minimum volume ellipsoid (MVE)}, but the maximum bias for the \emph{minimum covariance determinant (MCD)} is still largely open~\cite{adrover2002projection,hubert2010minimum}.  
Given the difficulty of analyzing the maximum bias, statisticians turned to surrogates of these concepts. Two popular notions are the \emph{breakdown point} and \emph{equivariance} property. The breakdown point is defined as the smallest corruption level $\epsilon$ such that the maximum bias is infinity. In other words, it measures the smallest level of corruption $p_n$ does to $p_S$ to drive the estimate $\hat{\theta}(p_n)$ to infinity. Ideally we want a large breakdown point, but this single criterion is not enough. Indeed, the constant zero estimator has breakdown point one but is completely useless. The second criteria, \emph{equivariance} mandates that the estimator $\hat{\theta}$ has to follow similar group transformations if we transform the data. For example, in the case of mean estimation,  \emph{translation equivariance} means that $\hat{\theta}(\{X_1+b,X_2+b,\ldots,X_n +b \}) = \hat{\theta}(X_1,X_2,\ldots,X_n) + b$ for any vector $b\in \mathbf{R}^d$, and \emph{affine equivariance} means that $\hat{\theta}(\{A X_1+b,A X_2+b,\ldots,AX_n +b \}) = A\hat{\theta}(X_1,X_2,\ldots,X_n) + b$ for any vector $b\in \mathbf{R}^d$ and nonsingular matrix $A$. It was shown that the the maximal breakdown point for any translation equivariant mean estimator is at most $\lfloor (n+1)/2 \rfloor /n$~\citep[Page 270]{rousseeuw1987robust}, which as $n\to \infty$ approaches $1/2$. If we enforce affine equivariance, then the maximum breakdown point decreases to $\lfloor (n-d+1)/2\rfloor /n$ as shown in~\citep[Page 271]{rousseeuw1987robust}, which is way below $1/2$ when $n$ and $d$ are comparable. 

Translation equivariance for mean estimation looks natural since it is implied by Fisher-consistency for mean estimation, but why do we additionally consider affine equivariance? One observation might be that there exist translation equivariant estimators with $1/2$ breakdown point, but it fails to achieve good maximum bias. Indeed, if $p_S$ comes from $d$-dimensional isotropic Gaussian and the estimator is coordinatewise median, then its maximum bias is of order $\epsilon\sqrt{d}$ while the information theoretic optimal error $O(\epsilon)$. Moreover, it was shown in~\citep[Page 250]{rousseeuw1987robust} that it does not necessarily lie in the convex hull of the samples when $d\geq 3$. Requiring affine equivariance rules out the coordinatewise median estimator, and may be a desirable property since the Tukey median is affine equivariant. However, it is quite challenging to find estimators that are both affine equivariant and have the largest possible breakdown point. A few notable examples are the Stahel-Donoho estimator~\cite{stahel1981breakdown, donoho1982breakdown}, the minimum volume ellipsoid (MVE)~\cite{davies1992asymptotics}, and the minimum covariance determinant (MCD)~\cite{butler1993asymptotics, rousseeuw1999fast, hubert2010minimum}, which are all shown to be NP-hard to compute in the worst case in~\cite{bernholt2006robust}. It was also shown in~\cite{davies1992asymptotics} that even if we can compute MVE, its maximum bias is suboptimal. 

Till today, researchers have not found any efficiently computable estimator that is both affine equivariant and has the largest possible breakdown point among affine equivariant estimators. However, the computationally efficient schemes we are discussing in this paper are solving the \emph{original problem} of analyzing maximum bias and Fisher-consistency. It appears that once we remove the affine-equivariance requirement, the problem becomes computationally efficiently solvable. 

The most interesting connection between the classical literature and recent computationally efficient robust estimators is the MCD~\cite{butler1993asymptotics, rousseeuw1999fast, hubert2010minimum}, which is defined as 
\begin{align}
&\textrm{minimize}   \quad \det(\Sigma_q) \\
& \textrm{subject to} \quad q\in\Delta_{n, \epsilon}.
\end{align}
It looks strikely similar to our example of robust mean estimation under bounded covariance:
\begin{align}
&\textrm{find}   \qquad \qquad q \\
& \textrm{subject to} \quad q\in\Delta_{n, \epsilon}, \| \Sigma_q\| \leq \sigma'^2. 
\end{align}
We can see that the major difference is that our problem is a feasibility problem while MCD is a minimization problem, and MCD considers the determinant but we use the operator norm. Interestingly, it was shown that among all minimization problems using the covariance matrix $\Sigma_q$, the determinant is the \emph{only} function that guarantees affine equivariance~\cite{liu2020note}.

\section{Proof for Section~\ref{sec.landscape}}

\subsection{Proof of Auxillary Lemmas}

\begin{lemma}\label{lem.qS_and_pS}
Denote $q_S$ as the distribution of $q$ conditioned on the good set $S$, i.e. 
\begin{align}
    q_{S, i} = \begin{cases}\frac{q_i}{\sum_{i\in S} q_i},  & i\in S,\\
    0, & otherwise.
    \end{cases}
\end{align}
Assume that~\eqref{eqn.invariance_filter} holds and $\forall i\in[n], c_i \leq \frac{1}{n}$, then $q_S$ is an  $\epsilon/(1-\epsilon)$-deletion of $p_S$, and an $\epsilon$-deletion of $q$. We also have $\TV(q, p_S)\leq \frac{\epsilon}{1-\epsilon}$.
\end{lemma}
\begin{proof}
From the update rule we have $\forall i, c_i\leq \frac{1}{n}$. Furthermore, we know that
\begin{align}
    \sum_{i\in S} (\frac{1}{n}-c_i ) \leq \sum_{i\in [n]/ S} (\frac{1}{n}-c_i ) \leq \epsilon.
\end{align}
Thus we have $ \sum_{i \in S} c_i  \geq 1-2\epsilon$, and
\begin{align}
    \forall i\in S, q_{S, i} & = \frac{c_i}{\sum_{i \in S} c_i} \leq \frac{1}{(1-2\epsilon)n} = \frac{1}{1-\epsilon/(1-\epsilon)}\cdot\frac{1}{(1-\epsilon)n}.
\end{align}
Thus we can conclude that $q_S$ is an $\epsilon/(1-\epsilon)$ deletion of $p_S$. On the other hand,
\begin{align}
    \forall i\in S, q_{S, i} & = \frac{c_i}{\sum_{i \in S} c_i} \leq \frac{c_i}{\sum_{i=1}^n c_i} \cdot \frac{\sum_{i=1}^n c_i}{\sum_{i \in S} c_i} \leq q_i \cdot \frac{1}{1-\epsilon}.
\end{align}

Now we  show that $\TV(p_S, q) \leq \frac{\epsilon}{1-\epsilon}$. We use the following formula for $\TV$: $\TV(p, q) = \int \max(q(x) - p(x), 0) dx$. Let $\beta$ be such that 
$\sum_{i=1}^n c_i = (1-\beta)n$. Then we have
\begin{align}
\TV(p_S, q) 
 &= \sum_{i \in S} \max\big(\frac{c_i}{(1-\beta)n} - \frac{1}{(1-\epsilon)n}, 0\big) + \sum_{i \not\in S} \frac{c_i}{(1-\beta)n}.
\end{align}
If $\beta \leq \epsilon$, then the first sum is zero while 
the second sum is at most $\frac{\epsilon}{1-\beta} \leq \frac{\epsilon}{1-\epsilon}$. If on the other hand $\beta > \epsilon$, 
we will instead use the equality obtained by swapping $p$ and $q$, 
which yields
\begin{align}
\TV(p_S, q) 
 &= \sum_{i \in S} \max\big(\frac{1}{(1-\epsilon)n} - \frac{c_i}{(1-\beta)n}, 0\big) \\
 &= \frac{1}{(1-\epsilon)(1-\beta)n} \sum_{i \in S} \max((1-\beta)(1-c_i) + (\epsilon - \beta)c_i, 0).
\end{align}
Since $(\epsilon - \beta)c_i \leq 0$ and $\sum_{i \in S} (1-c_i) \leq \epsilon n$, this yields a bound of $\frac{(1-\beta)\epsilon}{(1-\epsilon)(1-\beta)} = \frac{\epsilon}{1-\epsilon}$. 
We thus obtain the desired bound no matter the value of $\beta$, 
so $\TV(p_S, q) \leq \frac{\epsilon}{1-\epsilon}$.
\end{proof}

\subsection{Proof of Lemma~\ref{lem.deletion_TV}}\label{sec.proof_lemma_sec2}
\begin{proof} 
Note that for any set $A$, we have
\begin{align}
    p(A) & \leq \frac{r(A)}{1-\epsilon_1} \\
    q(A) & \leq \frac{r(A)}{1-\epsilon_2}.
\end{align}
Apply it to the complement of $A$, we have
\begin{align}
    p(A) & \geq \frac{r(A)-\epsilon_1}{1-\epsilon_1} \\
    q(A) & \geq \frac{r(A)-\epsilon_2}{1-\epsilon_2}. 
\end{align}
It then implies that
\begin{align}
    p(A) & \leq \frac{\epsilon_2 + (1-\epsilon_2)q(A)}{1-\epsilon_1}  \\
    q(A) & \leq \frac{\epsilon_1 + (1-\epsilon_1)p(A)}{1-\epsilon_2} 
\end{align}
If $\epsilon_2 \geq \epsilon_1$, we have
\begin{align}
    \TV(p,q) & = \sup_A p(A) - q(A) \\
    & \leq \sup_A \frac{\epsilon_2 + (1-\epsilon_2)q(A)}{1-\epsilon_1} - q(A) \\
    & = \frac{\epsilon_2}{1-\epsilon_1} + \sup_A \frac{\epsilon_1 - \epsilon_2}{1-\epsilon_1} q(A) \\
    & \leq \frac{\epsilon_2}{1-\epsilon_1} \\
    & = \frac{\max\{\epsilon_2,\epsilon_1\}}{1-\min\{\epsilon_2,\epsilon_1\}}. 
\end{align}
The case of $\epsilon_1 > \epsilon_2$ follows similarly by writing $\TV(p,q) = \sup_A q(A) - p(A)$. 
\end{proof}

\subsection{Proof of Lemma~\ref{lem.mean_modulus}}\label{proof.mean_modulus}

We first prove the following lemma for the resilience property of mean estimation.
\begin{lemma}[Resilience for mean estimation]\label{lem.mean_resilience}
For any $q$ and event $E$ such that $q(E)\geq 1-\epsilon$, we have
\begin{align}
    \|\bE_q[X] - \bE_{q}[X | E]\| \leq \sqrt{\|\Sigma_q\|\frac{\epsilon}{1-\epsilon} }.
\end{align}

\end{lemma}

\begin{proof}
For any $a\in\bR^d$ and event $E$ such that $q(E)\geq 1-\epsilon$, we have
\begin{align}
\bE_q[X]-a & =   \bE_q[(X-a) \mathbb{1}(E)] + \bE_q[(X-a) \mathbb{1}(E^c)]
\end{align}
For any direction $v\in\bR^d, \|v\|_2\leq 1$, by applying H\"older's inequality to $\bE_q[v^\top(X-a) \mathbb{1}(E^c)]$ we obtain 
\begin{align}
  v^\top(\bE_q[X]-a) & =  v^\top(\bE_q[(X-a) \mathbb{1}(E)] + \bE_q[(X-a) \mathbb{1}(E^c)])  \\
& \geq  v^\top\bE_q[(X-a) \mathbb{1}(E)] -
\sqrt{q(E^c)} \sqrt{\bE_q[(v^\top(X-a))^2]} \nonumber \\ 
& = q(E) (\bE_q[v^\top X | E] - v^\top a) - 
\sqrt{q(E^c)} \sqrt{\bE_q[(v^\top(X-a))^2]}
\end{align}
By taking $a =\bE_q[X] + \sqrt{\|\Sigma_q\| \frac{1-q(E)}{q(E)}}$, we have
\begin{align}
    \bE_q[v^\top X | E] - \bE_q[v^\top X]\leq \sqrt{\|\Sigma_q\|\frac{1-q(E)}{q(E)}} \leq \sqrt{\|\Sigma_q\|\frac{\epsilon}{1-\epsilon}}
\end{align}
Thus we can conclude that $\|\bE_q[X|E] - \bE_q[X] \| \leq \sqrt{\|\Sigma_q\|\frac{\epsilon}{1-\epsilon}}$ by taking the supremum over $v: \|v\|_2 = 1$. 
\end{proof}

From $\TV(p, q)\leq \epsilon$, we know that there exists some $r$ such that $r\leq \frac{p}{1-\epsilon}$, $r\leq \frac{q}{1-\epsilon} $~\citep[Lemma C.1]{zhu2019generalized}. Thus we have
\begin{align}
    \|\bE_q[X] - \bE_p[X]\| & \leq     \|\bE_q[X] - \bE_r[X]\| +    \|\bE_r[X] - \bE_p[X]\| \nonumber \\ 
    & \leq \sqrt{\|\Sigma_q\|\frac{\epsilon}{1-\epsilon}} + \sqrt{\|\Sigma_p\|\frac{\epsilon}{1-\epsilon}}
\end{align}

\begin{remark}  
Lemma~\ref{lem.mean_resilience} is tight since it achieves equality for the distribution $q$ where $q(0)=1-\epsilon,q(a) = \epsilon$, and the set $E = \{0\}$. It improves over existing results in the literature such as~\citep[Example 3.1]{zhu2019generalized}, ~\citep[Lemma 5.3]{cheng2019high},~\citep[Lecture 5, Lemma 1.1]{jerry2019lecnotes}. %
\end{remark}

\subsection{Discussions related to the lower bound for breakdown point}\label{sec.lower_KKT}

Now we show that not all the stationary points are global minimum, and provide some sufficient conditions when the distribution $q$ is stationary point the via the following example.
\begin{example}
Let $a>0$, and the corruption level $\epsilon = 1/n$. Let one dimensional corrupted distribution  $p$ be 
\begin{align}
    \bP_p[X=x] = \begin{cases}
    \frac{1}{n}, & x = -1
    \\\frac{n-2}{n}, & x = 0
    \\\frac{1}{n},  & x=a.
    \end{cases}
\end{align}
Let distribution $q$ be 
\begin{align}
     \bP_q[X=x] = \begin{cases}
    \frac{n-2}{n-1}, & x = 0
    \\\frac{1}{n-1},  & x=a.
    \end{cases}
\end{align}
Then $q, \mu_q$ is a stationary point in optimization problem~\eqref{eqn.KKT_min} when $a\leq\frac{n-1}{n-3}$, and is not a stationary point otherwise. 
\end{example}

\begin{proof}

The Lagrangian for the optimization problem is
\begin{equation*}
    L(q, w, u, y, \lambda) = F(q, w) + \sum_{i=1}^{n}u_i \Big(- q_{i}\Big) + \sum_{i=1}^{n}y_i \Big(q_{i} - \frac{1}{(1-\epsilon)n}\Big) + \lambda\Big(\sum_{i}^{n} q_i - 1\Big).
\end{equation*}
From the KKT conditions for locally Lipischitz functions~\cite{clarke1976new}, we know that the stationary points must satisfy
\begin{align}\label{eqn.kkt-1}
    &(\text{stationarity})\quad  0 \in \partial_{q, w}\Big( F(q, w) + \sum_{i=1}^{n}u_i q_{i} + \sum_{i=1}^{n}y_i \Big(q_{i} - \frac{1}{(1-\epsilon)n}\Big) + \lambda\Big(\sum_{i}^{n} q_i - 1\Big) \Big),\nonumber \\
    &(\text{complementary slackness})\quad  u_i (-q_i) = 0,\,\, y_i \Big(q_{i} - \frac{1}{(1-\epsilon)n}\Big)=0,\,\,  i \in [n], \\
    &(\text{primal feasibility})\quad  - q_{i} \leq 0,\,\,     q_{i} - \frac{1}{(1-\epsilon)n}\leq 0,\,\,
     \sum_{i}^{n} q_i = 1,\nonumber \\
    &(\text{dual feasibility})\quad  u_i \geq 0, \,\, y_i \geq 0,\,\,  i \in [n]. \nonumber \\
\end{align} 
It suffices to check the KKT condition in~\eqref{eqn.kkt-1}. Denote $q_1, q_2, q_3$  as the probability mass on $-1, 0, a$. Then the KKT conditions are equivalent to
\begin{align}
    0 & = (-1-\mu_q)^2-u_1+\lambda, \nonumber \\
    0 & = \mu_q^2 +y_2 + \lambda,\nonumber  \\
    0 & = (a-\mu_q)^2 + y_3+\lambda, \nonumber \\
    u_1, & y_2, y_3\geq 0. \nonumber 
\end{align}
Since $a>0$, the necessary and sufficient condition for the KKT conditions to hold is 
\begin{align}
    (a-\mu_q)^2 \leq (-1-\mu_q)^2,
\end{align}
Solving this inequality, we get $a\leq \frac{n-1}{n-3}$. 
\end{proof}

By substituting $1/n$ with $\epsilon$ and scaling all the points, we derive the example in Figure~\ref{fig:KKT}, which also shows  the tightness of Theorem~\ref{thm.low_regret_mean_relation}.  
This example shows that when the adversary puts the corrupted point ($X=a$) far away from the other points, the distribution that puts mass on the corrupted point will not be a local minimum if $n>3$. On the other hand, when the corrupted point is near the other points, the distribution can be a local minimum, but not a global minimum in general. What happens when $n = 3$? The next result shows when $n = 3$, one may have $a$ arbitrarily big and \emph{break down} when $\epsilon = 1/3$. 
\begin{theorem}\label{eqn.lower_bound_13}
For  $\epsilon = 1/3$ and  any $a>0$, there exists some distribution $p_S$ such that $\|\Sigma_{p_S}\|\leq \sigma^2$, while the mean of some local minimum of~\eqref{eqn.KKT_min} $\mu_q$ satisfies
\begin{align}
    \|\mu_q - \mu_{p_S}\|\geq a.
\end{align}

\end{theorem}

\begin{proof}
Here we consider the simple case when $n=3, d=1, \epsilon = 1/3$, and the number of `good' points is 2, and the number of `bad' points is 1, i.e.,
\begin{equation*}
    x_1 = 0, x_2 = 1, x_3 = a,
\end{equation*}
where $x_1$ and $x_2$ are good points and $x_3$ is the outlier. If we set $q$ as follows,
\begin{equation*}
    q_1 = 0, q_2 = \frac{1}{2}, q_3 = \frac{1}{2},
\end{equation*}
then we set $w$ as
\begin{equation*}
    w = \sum_{i=1}^{3} q_i x_i =   \frac{1}{2} \cdot a = \frac{1+a}{2},
\end{equation*}
then we have
\begin{align*}
    (x_1 -w)^2 = (0- \frac{1+a}{2})^{2} & = (\frac{a+1}{2})^{2}\\
    (x_2 -w)^2 = (1 - \frac{1+a}{2})^{2}&= (\frac{a-1}{2})^{2}\\
    (x_3 -w)^2 = (a - \frac{1+a}{2})^{2} &= (\frac{a-1}{2})^{2},\\
\end{align*}
thus, if we set 
$$\lambda = -(\frac{a-1}{2})^{2}, u_1 = (\frac{a+1}{2})^{2} - (\frac{a-1}{2})^{2}, v_1 = 0, u_2 = u_3 = 0, v_2 = v_3 =0,$$
then the current $\{x_i\}_{i=1}^{3}$, $\{q_i\}_{i=1}^{3}$, $\{u_i\}_{i=1}^{3}$, $\{v_i\}_{i=1}^{3}$, $\lambda$, $w$, satisfy the KKT condition, but it is not a good solution.

Now we verify that it is also a local minimum. For any fixed $q$, the optimal $w$ is always $w = \sum_{i=1}^3 q_ix_i$. Thus it suffices to consider any perturbation on $q$. Denote $q_1' = s+t, q_2' = \frac{1}{2}- t, q_3' = \frac{1}{2}-s$ for small $s, t>0$. Then we have
\begin{align}
   & \sum_{i=1}^3 q_i(x_i -   \sum_{i=1}^3 q_ix_i)^2 -  \sum_{i=1}^3 q_i'(x_i -   \sum_{i=1}^3 q_i'x_i)^2 \nonumber \\ 
   = &  (\frac{a-1}{2})^2 - \frac{(s+t)(a+1)^2}{2} - (\frac{1}{2}-t)(\frac{a-1}{2}-t-sa)^2 - (\frac{1}{2}-s)(\frac{-a+1}{2}-t-sa)^2 & \nonumber \\
    = &(\frac{a-1}{2})^2 - \frac{(s+t)(a+1)^2}{2}- (1-s-t)((\frac{a-1}{2})^2+(t+sa)^2)\nonumber \\ 
    & \quad - (s-t)(1-a)(t+sa)  \nonumber \\ 
    =& - O(a(s+t)) <0.
\end{align} Thus we can see $q$ is a local minimum. %
\end{proof}

\subsection{Proof of Theorem~\ref{thm.approximate_mean_small}}\label{proof.approximate_mean_small}
For the supremum achieving $v$,
we have
\begin{align}
    \|\Sigma_q\| & =  \mathbb{E}_{q}[(v^\top (X-\mu_q)^2)] \\
    & \stackrel{(i)}{\leq} (1+\alpha)\mathbb{E}_{p_{*}}[(v^\top (X-\mu_q)^2)] + \beta \\
    & = (1+\alpha)\mathbb{E}_{p_{*}}[ (v^\top (X-\mu_{p_{*}})^2) + (v^\top (\mu_q - \mu_{p_{*}}))^2]+\beta \\
    & \leq (1+\alpha)(\|\Sigma_{p_{*}}\| + \|\mu_q - \mu_{p_{*}}\|^2)+\beta \\
    & \leq (1+\alpha)\left(\|\Sigma_{p_{*}}\| + \left(\sqrt{\frac{\|\Sigma_q\|\epsilon}{1-2\epsilon}} + \sqrt{\frac{\|\Sigma_{p_{*}}\|\epsilon}{1-2\epsilon}} \right)^2\right)+\beta.
\end{align}
Here (i) comes from the assumption in the theorem, (ii) comes from Lemma~\ref{lem.mean_modulus} and Lemma~\ref{lem.deletion_TV}.
 
Solving the above inequality on $\|\Sigma_q\|$, we know that when $\epsilon\in[0, 1/(3+\alpha))$,
\begin{align}
\|\Sigma_q\|\leq \left(1 +  \frac{C_1 (\alpha+\epsilon)}{(1-(3+\alpha)\epsilon)^2}\right) \|\Sigma_{p_{*}}\|+ \frac{C_2\beta}{(1-(3+\alpha)\epsilon)^2},
\end{align}
for some constant $C_1, C_2$.

\section{Proof for Section~\ref{sec.low_regret}}

\subsection{Stationary point for hypercontractivity is not an approximately good solution}\label{sec.KKT_counter}

Consider the task of finding some distribution $q$ that is hypercontractive given corrupted distribution from a hypercontractive distribution, which is a sub-question from linear regression. To be concrete, we assume that  the true distribution $p_S$ satisfies $F(p_S)\leq \kappa^2$, where 
\begin{align}
   F(p) =  \frac{\bE_p[(v^\top X)^4]}{\bE_p[(v^\top X)^2]^2}.
\end{align}
The feasibility problem in Problem~\ref{prob.feasibility} reduces to finding some distribution $q$  such that $\TV(q, p_S)\leq\epsilon/(1-\epsilon)$, $F(q)\leq \kappa'^2$. 
 
 As the case of mean estimation, a natural approach to solve the feasibility problem is transfer that to the optimization problem of $\min_{q\in\Delta_{n, \epsilon}} F(q)$. However, different from the case of mean estimation,  the stationary point for this function is not a good solution for the feasibility problem. We show it in the below theorem. 
\begin{theorem}\label{thm.hyperstationarybad}
Given any $\kappa'>\kappa>0$, there exists some $n, \epsilon$ such that one can design some one-dimensional distribution $p_n$ satisfying: (a) there exists some set $S\subset[n],|S|\geq (1-\epsilon)n$, $F(p_S)\leq \kappa^2 $; (b) there exists a distribution   $q$ with $F(q)\geq \kappa'^2$ , while $q$ is a stationary point for the optimization problem $\min_{q\in\Delta_{n, \epsilon}} F(q)$.
\end{theorem}
\begin{remark}
Since any stationary point guarantees $\bE_q[g]\leq \bE_{p_S}[g]$ with $g$ taken as the partial derivative of $F(q)$ with respect to $q$. The counter example also shows that we cannot take this $g$ as generalized quasi-gradient in general. 

\end{remark}
\begin{proof}
The Lagrangian for the optimization problem is
\begin{equation*}
    L(q, u, y, \lambda) = F(q) + \sum_{i=1}^{n}u_i \Big(- q_{i}\Big) + \sum_{i=1}^{n}y_i \Big(q_{i} - \frac{1}{(1-\epsilon)n}\Big) + \lambda\Big(\sum_{i}^{n} q_i - 1\Big).
\end{equation*}
From the KKT conditions, we know that the stationary points must satisfy
\begin{align*} 
    &(\text{stationarity})\quad  0 \in \partial_{q}\Big( F(q) + \sum_{i=1}^{n}u_i q_{i} + \sum_{i=1}^{n}y_i \Big(q_{i} - \frac{1}{(1-\epsilon)n}\Big) + \lambda\Big(\sum_{i}^{n} q_i - 1\Big) \Big),\\
    &(\text{complementary slackness})\quad  u_i (-q_i) = 0,\,\, y_i \Big(q_{i} - \frac{1}{(1-\epsilon)n}\Big)=0,\,\,  i \in [n], \\
    &(\text{primal feasibility})\quad  - q_{i} \leq 0,\,\,     q_{i} - \frac{1}{(1-\epsilon)n}\leq 0,\,\,
     \sum_{i}^{n} q_i = 1,\\
    &(\text{dual feasibility})\quad  u_i \geq 0, \,\, y_i \geq 0,\,\,  i \in [n].
\end{align*} 
Denote $\tau_i$ as
\begin{align}\label{eqn.compute_tau_counter}
\tau_i = \partial_{q_i} F(q) = \frac{X_i^4}{(\sum_{i\in[n]} q_i X_i^2)^2} - \frac{2X_i^2(\sum_{i\in[n]} q_i X_i^4)}{(\sum_{i\in[n]} q_i X_i^2)^3}.
\end{align}
Then we will have 
\begin{align}
    0 & = \tau_i  - u_i + y_i + \lambda,\,\,  i \in [n], 
\end{align}   
Next, we define two sets, 
\begin{align*}
    S_{+} = \Big\{i \big| q_i > 0\Big\},\quad  S_{-} = \Big\{i \big| q_i = 0\Big\}
\end{align*}
and $(1-\epsilon)n \leq |S_{+}| \leq n $,  $|S_{-}| \leq \epsilon n$.
For $i\in[n]$ such that $q_i = \frac{1}{(1-\epsilon)n}$, which implies that $u_i = 0$, we have 
\begin{equation*}
\tau_i  = \underbrace{- y_i}_{\leq 0} - \lambda  \leq - \lambda.
\end{equation*}
For $i \in S_{-}$, $q_i = 0$, which implies that $y_i = 0$, we know that 
\begin{equation*}
\tau_i= \underbrace{ u_i}_{\geq 0} - \lambda  \geq - \lambda.
\end{equation*}
For $i$ such that $0<q_i< 1/(1-\epsilon)n$, which implies that $u_i = y_i = 0$, we  know that
\begin{equation*}
\tau_i=  - \lambda.
\end{equation*}
Now we are ready to construct $p_n, p_S, q$ such that $q$ is a stationary point of the optimization problem. 

Consider distribution $p_n$ with $1-\delta-\gamma$ fraction of points to be $0$, $\delta$ fraction of points at $a>0$, and $\gamma$ fraction of points at $b$. We assume that $\delta> \epsilon>\gamma$ and $a<b$. Assume that $n$ is picked such that all the fractions listed below multiplying $n$ will be integer.    Let $p_S$ be the distribution of completely deleting point $b$ from $p_n$, i.e. $p_S$ has $\frac{1-\delta-\gamma}{1-\gamma}$ fraction of points on $0$, and $\frac{\delta}{1-\gamma}$ fraction of points on $a$. We take $\delta,\gamma $ such that it  satisfies 
\begin{align}
    \frac{1-\gamma}{\delta} \leq \kappa^2,
\end{align}
which implies that $F(p_S)\leq \kappa^2$. 

Let $q$ be the distribution of deleting $\epsilon$ mass from point $a$, i.e. $q$ has $\frac{1-\delta-\gamma}{1-\epsilon}$ mass on $0$, $\frac{\delta-\epsilon}{1-\epsilon}$ mass on $a$, $\frac{\gamma}{1-\epsilon}$ mass on $b$. We first verify that $q$ is a stationary point for this problem.
Denote $A$, $B$, $C$ as the set of indexes of points that are supported on $0, a$ and $b$, separately. 
For any point $i\in A$, from~\eqref{eqn.compute_tau_counter}, we have $\tau_i = 0$. Since all the points have $q_i = \frac{1}{(1-\epsilon)n}$, we know that 
\begin{align}
\forall i\in A, \tau_i = 0 = -y_i - \lambda, y_i \geq 0.\label{eqn.counter_sol1}
\end{align}
For any point $i\in C$, from~\eqref{eqn.compute_tau_counter}, we have $\tau_i = 0$. Since all the points have $q_i = \frac{1}{(1-\epsilon)n}$, we know that 
\begin{align}
\forall i\in C, \tau_i & = \frac{X_i^4}{(\sum_{i\in[n]} q_i X_i^2)^2} - \frac{2X_i^2(\sum_{i\in[n]} q_i X_i^4)}{(\sum_{i\in[n]} q_i X_i^2)^3} \nonumber  \\
& = \frac{b^4}{((\delta-\epsilon)a^2/(1-\epsilon) + \gamma b^2/(1-\epsilon))^2} - \frac{2b^2((\delta-\epsilon)a^4/(1-\epsilon) + \gamma b^4/(1-\epsilon)) }{((\delta-\epsilon)a^2/(1-\epsilon) + \gamma b^2/(1-\epsilon))^3}.
\end{align}
For some fixed $\delta, a$,  we set $ \gamma$ such that $(\delta-\epsilon)a^2/(1-\epsilon) = \gamma b^2/(1-\epsilon)$.  Thus we have
\begin{align}
    \forall i \in C, \tau_i & = -y_i -\lambda = \frac{(1-\epsilon)^2}{4\gamma^2} - \frac{2b^2\cdot((\delta-\epsilon)a^4/(1-\epsilon) + \gamma b^4/(1-\epsilon)) }{(2\gamma b^2/(1-\epsilon))^3} \nonumber \\ 
    & < \frac{(1-\epsilon)^2}{4\gamma^2} - \frac{2b^2\cdot ( \gamma b^4/(1-\epsilon)) }{(2\gamma b^2/(1-\epsilon))^3} \nonumber \\ 
    & =  0, y_i \geq 0.\label{eqn.counter_sol2}
\end{align}
Similarly, for any point $i\in B$, there are some points $i$ with $q_i = \frac{1}{(1-\epsilon)n}$. Denote the set as $D$, and the rest as $B/D$. Then  
we can similarly compute that 
\begin{align}
    \forall i \in D, \tau_i = c= -y_i -\lambda >0, y_i \leq 0 \nonumber \\
    \forall i \in B/D, \tau_i  = c = u_i -\lambda >0, u_i\geq 0\label{eqn.counter_sol3}
\end{align}
where $c$ is some positive value. 
We can see that there exists $y_i\leq 0, u_i\geq 0, \lambda$ such that Equation~\eqref{eqn.counter_sol1}, \eqref{eqn.counter_sol2} and \eqref{eqn.counter_sol3} 
hold simultaneously by taking $\lambda = -c$, $\forall i\in B, u_i = y_i = 0, $, $\forall i \in A, y_i = c>0$, $\forall i \in B, y_i = -\tau_i + c>0 $. Thus $q$ is a stationary point. Now we let $b\rightarrow \infty$, since we have set $ \gamma$ such that $(\delta-\epsilon)a^2/(1-\epsilon) = \gamma b^2/(1-\epsilon)$, we have $\gamma\rightarrow 0$, and
\begin{align}
    \frac{\bE_q[X^4]}{\bE_q[X^2]^2} = \frac{\gamma b^4(1-\epsilon)}{\gamma^2 b^4} = \frac{1-\epsilon}{\gamma} \rightarrow \infty. 
\end{align}

\end{proof}
\subsection{Proof of Auxillary Lemmas}

The following Lemma gives a tighter bound on the modulus with respect to the coefficient in front of $\tau$ than~\citep[Lemma E.3]{zhu2019generalized}. 

\begin{lemma}[Modulus of continuity for mean estimation with near identity covariance]\label{lem.empirical_identity_cov_modulus}
For some fixed $\epsilon\in[0,1)$ and non-negative constants $\rho$,  $\tau$ and    $\tau'$, define 
\begin{align}
    \mathcal{G}_1 &= \{p \mid   \forall r \leq \frac{p}{1-\epsilon},  \| \mu_r - \mu_p\|_2 \leq  \rho ,   \lambda_{\mathsf{min}}( \mathbb{E}_r[(X - \mu_p)(X - \mu_p)^{\top}]) \geq  1- \tau \}\\ 
    \mathcal{G}_2 &= \{p \mid  \| \mathbb{E}_p[(X - \mu_p)(X - \mu_p)^{\top}]\|_2 \leq 1+ \tau' \}.
\end{align}
Here $\lambda_{\mathsf{min}}(A)$ is the smallest eigenvalue of symmetric matrix $A$. Assume $p_S\in  \GG_1$, $q\in \GG_2$, $\TV(q, p_S)\leq \epsilon$.
Then   we have
\begin{align}
    \sup_{p\in\GG_1, q\in\GG_2, \TV(p, q)\leq \epsilon} \|\mu_p  - \mu_q\|_2   \leq      \frac{ \rho }{1- \epsilon} + \sqrt{\frac{\epsilon(\tau+\tau' + \epsilon)}{1- \epsilon} + \frac{\epsilon \rho^2}{(1- \epsilon)^2}}. 
\end{align}
Here $C$ is some universal constant.
\end{lemma}
\begin{proof}
Assume $p\in \GG_1, q\in\GG_2$, $p\neq q$. 
Without loss of generality, we assume $\mu_p=0$. From $\TV(p, q)= \epsilon_0 \leq \epsilon$, we construct distribution $r = \frac{\min(p, q)}{1-\epsilon_0}$. Then we know that 
$r\leq \frac{p}{1-\epsilon_0}$, $r\leq \frac{q}{1-\epsilon_0}$.
Denote $\tilde r = (1-\epsilon_0)r$. Consider measure $p-\tilde r, q-\tilde  r$. We have $\mu_q = \mu_p - \mu_{p-\tilde r} + \mu_{q-\tilde r} =  -\mu_{p-\tilde r} + \mu_{q-\tilde r}$. Note that $\| \mu_{p-\tilde r}\|_2 = \| \mu_p - \mu_{\tilde r}\|_2  = \| \mu_{\tilde r} \|_2 \leq (1-\epsilon_0)\rho \leq \rho$.
For any $v\in\bR^d, \|v\|_2=1$, we have
\begin{align}
    v^{\top}\Sigma_qv^{\top} & = v^{\top}(\bE_q[XX^{\top}] -\mu_q\mu_q^{\top})v \nonumber \\
    & = v^{\top}(\bE_{\tilde r}[XX^{\top}] + \bE_{q-\tilde r}[XX^{\top}] - (\mu_{q-\tilde r}-\mu_{p-\tilde r})(\mu_{q-\tilde r}-\mu_{p-\tilde r})^{\top})v \nonumber \\
    & \geq (1-\tau)(1-\epsilon_0) + \bE_{q-\tilde r}[(v^{\top}X)^2] - (v^{\top}\mu_{q-\tilde r})^2 + 2 v^{\top}\mu_{q-\tilde r}v^{\top}\mu_{p-\tilde r} - (v^{\top}\mu_{p-\tilde r})^2 \nonumber \\
    & \geq  1- \tau  - \epsilon_0   + \bE_{q-\tilde r}[(v^{\top}X)^2] - (v^{\top}\mu_{q-\tilde r})^2 - 2 \|\mu_{q-\tilde r}\|_2 \|\mu_{p-\tilde r}\|_2 - \|\mu_{p-\tilde r}\|^2 \nonumber \\
    & \geq  1- \tau - \epsilon_0   + \bE_{q-\tilde r}[(v^{\top}X)^2] - (v^{\top}\mu_{q-\tilde r})^2 - 2\rho \|\mu_{q-\tilde r}\|_2 - \rho^2.
\end{align}
Denote $b_q = \frac{q-\tilde r}{\epsilon_0}$. Then $b_q$ is a distribution. If $\mu_{b_q} = 0$, then we already know that $\|\mu_q - \mu_r\|\leq \epsilon_0 \|\mu_{b_q}\|_2 =0$. Otherwise
we take $v = \frac{\mu_{b_q}}{\|\mu_{b_q}\|_2}$.
Then we can see $ \bE_{q-\tilde r}[(v^{\top}X)^2]= \epsilon_0\bE_{b_q}[(v^{\top}X)^2]\geq \epsilon_0 \|\mu_{b_q}\|_2^2$. 
From $q\in \GG_2$, we know  that $v^{\top}\Sigma_qv\leq 1 + \tau'$. Thus
\begin{align}
    (\epsilon_0 - \epsilon^2_0)\|\mu_{b_q}\|^2_2 -  2 \epsilon_0 \rho \|\mu_{b_q}\|_2 \leq \rho^2 + \tau  +  \tau' + \epsilon_0.
\end{align}
Solving the inequality, we derive that
\begin{align}
  \|\mu_q - \mu_r\|_2 \leq \epsilon_0 \|\mu_{b_q}\|_2 \leq \frac{\epsilon \rho}{1- \epsilon} + \sqrt{\frac{\epsilon( \tau+\tau' + \epsilon)}{1- \epsilon} + \frac{\epsilon \rho^2}{(1- \epsilon)^2}}.
\end{align}
where $C$ is some universal constant.
Thus we can conclude
\begin{align}
    \| \mu_p - \mu_q\|_2\leq \| \mu_p - \mu_r\|_2 + \| \mu_r - \mu_q\|_2 \leq  \frac{\rho}{1- \epsilon} + \sqrt{\frac{\epsilon( \tau+\tau' + \epsilon)}{1- \epsilon} + \frac{\epsilon \rho^2}{(1- \epsilon)^2}}.
\end{align}
\end{proof}

\begin{lemma}[Sum-of-squares modulus of continuity for bounded covariance distributions]\label{lem.modu_sec}
  Assume that $\TV(p, q)\leq {\epsilon}$ and both $p, q$ has finite second moment. Then we have
\begin{align*}
  (\bE_q[v^\top X] - \bE_p[v^\top X] )^2 &  \psos \frac{2\epsilon}{(1-\epsilon)^2} \cdot ( \bE_p[(v^\top (X-\mu_p))^2] +  \bE_q[(v^\top (X-\mu_q))^2]) \nonumber \\ 
  & \psos \frac{2\epsilon}{(1-\epsilon)^2} \cdot  ( \bE_p[(v^\top X)^2] +  \bE_q[(v^\top X)^2]).
\end{align*}
\end{lemma}

\begin{proof}
For any distribution $p, q$ with $\TV(p, q)\leq \epsilon$, there exists some distribution $r$ such that $r$ is an $\epsilon$-deletion of both distributions, i.e. $r\leq \frac{p}{1-\epsilon}$, $r\leq \frac{q}{1-\epsilon}$. 

The proof uses the property that for any $r\leq \frac{p}{1-\eta}$, there exists some event $E$ such that $\bP_p(E)\geq 1-\epsilon$ and $\bE_r[f(X)] = \bE_p[f(X)|E]$ for any measurable $f$~(see e.g. \citep[Lemma C.1]{zhu2019generalized}).
    For any event $E$ with $\bP_p(E)\geq 1-\epsilon$, denote its compliment as $E^c$.
    We have
    \begin{align}
    \bE_q[(v^\top X)^2] & \sos  \bE_q[(v^\top (X-\mu_q))^2] \nonumber \\ 
        & \sos \bE_q[(v^\top (X-\mu_q))^2 \mathbb{1}(E^c)] \nonumber \\ 
        & \stackrel{(i)}{\sos} \bE_q[(v^\top (X-\mu_q)) \mathbb{1}(E^c)]^2 / \epsilon \nonumber \\ 
        & \stackrel{(ii)}{=} \bE_q[(v^\top (X-\mu_q)) \mathbb{1}(E)]^2 / \epsilon \nonumber \\ 
        & \stackrel{(iii)}{\sos} (\bE_r[ v^\top X] - \bE_q[v^\top X])^2 (1-\epsilon)^2 / \epsilon. \nonumber 
    \end{align}
Here (i) comes from SOS-H\"older's inequality, (ii) comes from the fact that $\bE_q[((v^\top X)^2 - \bE_q[(v^\top X)^2])^2 \mathbb{1}(E^c)] + \bE_q[((v^\top X)^2 - \bE_q[(v^\top X)^2])^2 \mathbb{1}(E)]  = 0$, (iii) comes from that $r = p | E$. Thus we have
\begin{align}
       (\bE_r[v^\top X] - \bE_q[v^\top X] )^2 & \psos \frac{\epsilon}{(1-\epsilon)^2} \bE_q[(v^\top (X-\mu_q))^2].\nonumber
\end{align}
Using the same argument for $p$, we have
\begin{align}
    (\bE_r[v^\top X] - \bE_p[v^\top X] )^2 & \psos \frac{\epsilon}{(1-\epsilon)^2} \bE_p[(v^\top (X-\mu_p))^2].\nonumber
\end{align}
By summing the two SOS-inequalities, we have
\begin{align*}
    (\bE_q[v^\top X] - \bE_p[v^\top X] )^2 & \psos 2  ((\bE_r[v^\top X] - \bE_q[v^\top X] )^2 + (\bE_r[v^\top X] - \bE_p[v^\top X] )^2) \nonumber \\ 
    & \psos \frac{2\epsilon}{(1-\epsilon)^2} ( \bE_p[(v^\top (X-\mu_p))^2] +  \bE_q[(v^\top (X-\mu_q))^2]) \nonumber \\ 
    & \psos \frac{2\epsilon}{(1-\epsilon)^2} ( \bE_p[(v^\top X)^2] +  \bE_q[(v^\top X)^2]).
\end{align*}
\end{proof}

\begin{lemma}[Modulus of continuity for certifiable hypercontractivity]\label{lem.modu_hyper}
Denote $\GG(\kappa^2) = \{ p \mid  \bE_p[(v^\top X)^4] \psos \kappa^2\bE_p[(v^\top X)^2]^2  \}$. Assume that $p\in\GG(\kappa^2), q \in \GG(\kappa'^2)$, $\TV(p, q)\leq {\epsilon}$. Then we have
\begin{align*}
    \frac{(1-\epsilon-\kappa\sqrt{\epsilon})^2}{(1-\epsilon+\kappa\sqrt{\epsilon})^2}\bE_q[(v^\top X)^2]^2 \psos \bE_p[(v^\top X)^2]^2 \psos \frac{(1-\epsilon+\kappa\sqrt{\epsilon})^2}{(1-\epsilon-\kappa\sqrt{\epsilon})^2}\bE_q[(v^\top X)^2]^2.
\end{align*}
    
\end{lemma}

\begin{proof}
For any distribution $p, q$ with $\TV(p, q)\leq \epsilon$, there exists some distribution $r$ such that $r$ is an $\epsilon$-deletion of both distributions, i.e. $r\leq \frac{p}{1-\epsilon}$, $r\leq \frac{q}{1-\epsilon}$ from~\citep[Lemma C.1]{zhu2019generalized}.  The proof uses the property that for any $r\leq \frac{p}{1-\eta}$, there exists some event $E$ such that $\bP_p(E)\geq 1-\epsilon$ and $\bE_r[f(X)] = \bE_p[f(X)|E]$ for any measurable $f$~(see e.g. \citep[Lemma C.1]{zhu2019generalized}).
    For any event $E$ with $\bP_p(E)\geq 1-\epsilon$, denote its compliment as $E^c$.
    We have
    \begin{align}
        \kappa^2 \bE_q[(v^\top X)^2]^2 & \sos \bE_q[(v^\top X)^4] \nonumber \\ 
        & \sos \bE_q[((v^\top X)^2 - \bE_q[(v^\top X)^2])^2] \nonumber \\ 
        & \sos \bE_q[((v^\top X)^2 - \bE_q[(v^\top X)^2])^2 \mathbb{1}(E^c)] \nonumber \\ 
        & \stackrel{(i)}{\sos} \bE_q[((v^\top X)^2 - \bE_q[(v^\top X)^2]) \mathbb{1}(E^c)]^2 / \epsilon \nonumber \\ 
        & \stackrel{(ii)}{=} \bE_q[((v^\top X)^2 - \bE_q[(v^\top X)^2]) \mathbb{1}(E)]^2 / \epsilon \nonumber \\ 
        & \stackrel{(iii)}{\sos} (\bE_r[(v^\top X)^2] - \bE_q[(v^\top X)^2] )^2 (1-\epsilon)^2 / \epsilon. \nonumber 
    \end{align}
Here (i) comes from SOS-H\"older's inequality, (ii) comes from the fact that $\bE_q[((v^\top X)^2 - \bE_q[(v^\top X)^2])^2 \mathbb{1}(E^c)] + \bE_q[((v^\top X)^2 - \bE_q[(v^\top X)^2])^2 \mathbb{1}(E)]  = 0$, (iii) comes from that $r = p | E$. Thus we have
\begin{align}
       (\bE_r[(v^\top X)^2] - \bE_q[(v^\top X)^2] )^2 & \psos \frac{\epsilon \kappa^2}{(1-\epsilon)^2} \bE_q[(v^\top X)^2]^2\nonumber  \\
 \Rightarrow    \bE_r[(v^\top X)^2]^2 + \bE_q[(v^\top X)^2] &\psos \frac{\epsilon \kappa^2}{(1-\epsilon)^2} \bE_q[(v^\top X)^2]^2 + 2\bE_r[(v^\top X)^2]^2 \bE_q[(v^\top X)^2]^2
 \nonumber \\ 
   & \psos  \frac{\epsilon \kappa^2}{(1-\epsilon)^2} \bE_q[(v^\top X)^2]^2 + \frac{1}{\alpha} \bE_r[(v^\top X)^2]^4 + \alpha \bE_q[(v^\top X)^2]^2. \nonumber 
\end{align}
By optimizing over $\alpha $ in the regime $\alpha>1$ and $\alpha\in(0, 1)$ separately, we can get 
\begin{align}
    \frac{1}{(1+\sqrt{\epsilon\kappa^2/(1-\epsilon)^2})}\bE_q[(v^\top X)^2]^2 \psos \bE_r[(v^\top X)^2]^2 \psos \frac{1}{(1-\sqrt{\epsilon\kappa^2/(1-\epsilon)^2})}\bE_q[(v^\top X)^2]^2.
\end{align}
Using the same argument we have
\begin{align}
    \frac{1}{(1+\sqrt{\epsilon\kappa^2/(1-\epsilon)^2})^2}\bE_p[(v^\top X)^2]^2 \psos \bE_r[(v^\top X)^2]^2 \psos \frac{1}{(1-\sqrt{\epsilon\kappa^2/(1-\epsilon)^2})^2}\bE_p[(v^\top X)^2]^2.
\end{align}
Thus we have
\begin{align}
    \frac{(1-\sqrt{\epsilon\kappa^2/(1-\epsilon)^2})^2}{(1+\sqrt{\epsilon\kappa^2/(1-\epsilon)^2})^2}\bE_q[(v^\top X)^2]^2 \psos \bE_p[(v^\top X)^2]^2 \psos \frac{(1+\sqrt{\epsilon\kappa^2/(1-\epsilon)^2})^2}{(1-\sqrt{\epsilon\kappa^2/(1-\epsilon)^2})^2}\bE_q[(v^\top X)^2]^2.
\end{align}
    
\end{proof}

\subsection{Proof of Theorem~\ref{thm.low_regret_linreg} }\label{proof.low_regret_linreg}

We prove the theorem via two separate arguments. First, we show that  $\tilde g_1$ is a valid generalized quasi-gradient for  certifiable hypercontractivity. Then, given the knowledge that $q$ is certifiably hypercontractive, we  show that  $g_2$ is a valid generalized quasi-gradient for    bounded noise condition.

\begin{lemma}[Generalized quasi-gradient for hypercontractivity]\label{lem.low_regret_hypercontractivity_relation}
Under the same assumption as Theorem~\ref{thm.low_regret_linreg}, for any $q\in\Delta_{n, \epsilon}$ that satisfies $\bE_q[g_1(X; q)] \leq \bE_{p_S}[g_1(X; q)]$,  when $9\epsilon  \kappa^2/(1-2\epsilon)^2 \leq 1$ we have
\begin{align}
 \bE_q[(v^\top X)^4] \psos 4\kappa^2  \bE_q[(v^\top X)^2]^2.
\end{align}
\end{lemma}
\begin{proof}
Denote $\tilde \kappa^2 = \sup_{\pseudoE\in\cE_4}\frac{  \pseudoE[\bE_q[(v^\top X)^4]]}{  \pseudoE[\bE_q[(v^\top X)^2]^2]}$. Assume that $\tilde \kappa \geq \kappa$, since otherwise we already have $F(q)\leq \kappa^2$.  From the modulus of  continuity for second moment in Lemma~\ref{lem.modu_sec},   since $\TV(p_S, q) \leq \frac{\epsilon}{1-\epsilon}$,  we have
\begin{align}
    (\bE_q[(v^\top X)^2] - \bE_{p_S}[(v^\top X)^2] )^2   & \psos \frac{2\epsilon(1-\epsilon)}{(1-2\epsilon)^2} ( \bE_{p_S}[(v^\top X)^4] +  \bE_q[(v^\top X)^4]).
\end{align}
Thus we know that for the specific choice of pseudoexpectation $\pseudoE$, we have
\begin{align}
    \pseudoE [(\bE_q[(v^\top X)^2] - \bE_{p_S}[(v^\top X)^2] )^2]   & \leq \frac{2\epsilon(1-\epsilon)}{(1-2\epsilon)^2} \pseudoE[ \bE_{p_S}[(v^\top X)^4] +  \bE_q[(v^\top X)^4]] \nonumber \\ 
    & \stackrel{(i)}{\leq}  \frac{4\epsilon }{(1-2\epsilon)^2} \cdot \pseudoE[\bE_{p_S}[(v^\top X)^4]] \nonumber \\ 
    & \stackrel{(ii)}{\leq}  \frac{4\kappa^2 \epsilon }{(1-2\epsilon)^2} \cdot \pseudoE[\bE_{p_S}[(v^\top X)^2]^2]. 
\end{align}
Here (i) comes from the assumption, and (ii) is from the assumption that $\tilde F(p_S)\leq \kappa^2$.
Rearranging the inequality gives us
\begin{align}
    \pseudoE [(\bE_q[(v^\top X)^2]^2] +  \pseudoE [\bE_{p_S}[(v^\top X)^2]^2] & \leq \frac{4\kappa^2 \epsilon }{(1-2\epsilon)^2} \cdot \pseudoE[\bE_{p_S}[(v^\top X)^2]^2] + 2 \pseudoE[\bE_q[(v^\top X)^2]\cdot \bE_{p_S}[(v^\top X)^2]] \nonumber \\ 
    & \leq \frac{4\kappa^2 \epsilon }{(1-2\epsilon)^2} \cdot \pseudoE[ \bE_{p_S}[(v^\top X)^2]^2] + \frac{1}{\alpha} \pseudoE[\bE_q[(v^\top X)^2]^2] + \alpha  \pseudoE[\bE_{p_S}[(v^\top X)^2]^2],
\end{align}
for any $\alpha>0$. By optimizing over $\alpha$, we have
\begin{align}
       \pseudoE[\bE_q[(v^\top X)^2]^2] \leq \gamma^2   \pseudoE[\bE_{p_S}[(v^\top X)^2]^2],
\end{align}
where $\gamma^2 = \frac{(1+\sqrt{\epsilon  \kappa^2/(1-2\epsilon)^2})^2}{(1-\sqrt{\epsilon \kappa^2/(1-2\epsilon)^2})^2} $.
Thus we know that
\begin{align}
    \pseudoE[\bE_q[ (v^\top X)^4)]] & \leq \pseudoE[\bE_{p_S}[(v^\top X)^4)]] \nonumber \\ 
    & \leq \kappa^2\pseudoE[\bE_{p_S}[(v^\top X)^2)]^2]\nonumber \\   & \leq \gamma^2\kappa^2 \pseudoE[\bE_q[ (v^\top X)^2)]^2] \\ 
    & = \frac{\gamma^2\kappa^2}{\tilde \kappa^2} \pseudoE[\bE_q[ (v^\top X)^4)]].
\end{align}
By solving the above inequality, we have when $9\epsilon  \kappa^2/(1-2\epsilon)^2 \leq 1$,
\begin{align}
\tilde \kappa \leq 2\kappa.
\end{align}
\end{proof}

Now assume that we already know that  $q$ is hypercontractive with parameter $2\kappa$.   We show that $q$ also satisfies bounded noise condition via the following lemma. 
\begin{lemma}\label{lem.low_regret_bdd_noise}
Under the same assumption as Theorem~\ref{thm.low_regret_linreg}, assume  $q\in\Delta_{n, \epsilon}$ satisfies 
\begin{align}
    \bE_q[g_2(X; q)] \leq  \bE_{p_S}[g_2(X; q)], \forall v\in\bR^d, 
 \bE_q[(v^\top X)^4] \leq 4\kappa^2  \bE_q[(v^\top X)^2]^2.
\end{align}
Then when $\kappa^3\epsilon<1/64$, %
we have
\begin{align}
    \forall v\in\bR^d,  {\bE_q[(Y-X^\top \theta^*(q))^2(v^\top X)^2]} \leq 3\sigma^2{\bE_q[(v^\top X)^2]}.
\end{align}
\end{lemma}

\begin{proof}
Denote $\tilde \sigma^2 = \sup_{v\in\bR^d}\bE_q[(Y-X^\top \theta^*(q)^2(v^\top X)^2]/\bE_q[(v^\top X)^2]$.   Then we have
\begin{align}
      \bE_q[(Y-X^\top \theta^*(q))^2(v^\top X)^2] & \leq  \bE_{p_S}[(Y-X^\top \theta^*(q))^2(v^\top X)^2] \nonumber \\ 
      & \leq \bE_{p_S}[(Y-X^\top \theta^*(p_S))^2(v^\top X)^2] + \bE_{p_S}[  X^\top (\theta^*(q) - \theta^*(p_S))^2(v^\top X)^2] \nonumber \\ 
      & \stackrel{(i)}{\leq}  \bE_{p_S}[(Y-X^\top \theta^*(p_S))^2(v^\top X)^2] +\bE_{p_S}[(X^\top(\theta^*(p_S)-\theta^*(q)))^4]^{1/2}\cdot \bE_{p_S}[(v^\top X)^4]^{1/2} \nonumber \\ 
      & \stackrel{(ii)}{\leq} \sigma^2 \bE_{p_S}[(v^\top X)^2] + \kappa^2\bE_{p_S}[(X^\top(\theta^*(p_S)-\theta^*(q)))^2]\cdot \bE_{p_S}[(v^\top X)^2].
\end{align}
Here (i) is a result of Cauchy-Schwarz inequality, (ii) comes from the hypercontractivity of $p_S$.
From Lemma~\ref{lem.modu_hyper}, we know that 
\begin{align}
    \bE_{p_S}[(v^\top X)^2] \leq \frac{1-\epsilon+2\kappa\sqrt{\epsilon}}{1-\epsilon-2\kappa\sqrt{\epsilon}} \bE_q[(v^\top X)^2].
\end{align}
From~\citep[Theorem 3.4]{zhu2019generalized}, we know that when $\kappa^2\epsilon<   1/32 $
\begin{align}
    \bE_{q}[((\theta^*(p_S)-\theta^*(q))^\top X)^2]\leq \frac{2\kappa\tilde\sigma^2\epsilon(1-\epsilon)}{(1-2\epsilon)^2}.
\end{align}
Thus
\begin{align}
      \bE_q[(Y-X^\top \theta^*(q))^2(v^\top X)^2] & \leq \frac{1-\epsilon+2\kappa\sqrt{\epsilon}}{1-\epsilon-2\kappa\sqrt{\epsilon}}\cdot (\sigma^2 + \frac{2\kappa^3\tilde\sigma^2\epsilon(1-\epsilon)}{(1-2\epsilon)^2})\bE_q[(v^\top X)^2] \nonumber \\ 
      & \leq \frac{1-\epsilon+2\kappa\sqrt{\epsilon}}{1-\epsilon-2\kappa\sqrt{\epsilon}}\cdot (\frac{\sigma^2}{\tilde \sigma^2} + \frac{2\kappa^3\epsilon(1-\epsilon)}{(1-2\epsilon)^2})\bE_q[(Y-X^\top \theta^*(q))^2(v^\top X)^2].
\end{align}
By solving the inequality, we know that when $\kappa^3\epsilon < 1/64$, $\tilde \sigma^2 \leq 3\sigma^2$ (here we use the fact that $\kappa \geq 1$ always holds).
\end{proof}

\subsection{Proof of Theorem~\ref{thm.low_regret_joint_relation}}\label{proof.low_regret_joint}

\begin{proof}
Denote $\tilde \kappa^2 = \sup_{\pseudoE\in\cE_4}\frac{  \pseudoE[\bE_q[(v^\top (X-\mu_q))^4]]}{  \pseudoE[\bE_q[(v^\top (X-\mu_q))^2]^2]}$. Assume that $\tilde \kappa \geq \kappa$, since otherwise we already have $F(q)\leq \kappa^2$.  From the SOS modulus of continuity for second moment in Lemma~\ref{lem.modu_sec},  since $\TV(p_S, q) \leq \frac{\epsilon}{1-\epsilon}$ from Lemma~\ref{lem.deletion_TV},  we have,
\begin{align}
    (\bE_q[(v^\top (X-\mu_q))^2] - \bE_{p_S}[(v^\top (X-\mu_q))^2] )^2   & \psos \frac{2\epsilon(1-\epsilon)}{(1-2\epsilon)^2} ( \bE_{p_S}[(v^\top (X-\mu_q))^4] +  \bE_q[(v^\top (X-\mu_q))^4]).
\end{align}
Thus we know that for $\pseudoE$, we have
\begin{align}
    & \pseudoE [(\bE_q[(v^\top (X-\mu_q))^2] - \bE_{p_S}[(v^\top (X-\mu_q))^2] )^2] \nonumber \\ 
    \leq & \frac{2\epsilon(1-\epsilon)}{(1-2\epsilon)^2} \cdot \pseudoE[ \bE_{p_S}[(v^\top (X-\mu_q))^4] +  \bE_q[(v^\top (X-\mu_q))^4]] \nonumber \\ 
    \stackrel{(i)}{\leq} &  \frac{4\epsilon }{(1-2\epsilon)^2} \cdot \pseudoE[\bE_{p_S}[(v^\top (X-\mu_q))^4]] \nonumber \\ 
 \leq &  \frac{32 \epsilon }{(1-2\epsilon)^2} \cdot \pseudoE[\bE_{p_S}[(v^\top (X-\mu_{p_S}))^4] + (v^\top(\mu_q - \mu_{p_S}))^4] \nonumber \\ 
     \stackrel{(ii)}{\leq}  &  \frac{32 \epsilon }{(1-2\epsilon)^2} \cdot \pseudoE[\kappa^2\bE_{p_S}[(v^\top (X-\mu_{p_S}))^2]^2 + (v^\top(\mu_q - \mu_{p_S}))^4] \nonumber \\ 
     \stackrel{(iii)}{\leq}  & \frac{32 \epsilon }{(1-2\epsilon)^2} \cdot \pseudoE\left[\kappa^2\bE_{p_S}[(v^\top (X-\mu_{p_S}))^2]^2 + \left(\frac{2\epsilon(1-\epsilon)}{(1-2\epsilon)^2}\left(\bE_q[(v^\top(X-\mu_q))^2] + \bE_{p_S}[(v^\top(X-\mu_{p_S}))^2]\right)\right)^2\right] \nonumber \\ 
     \leq & \frac{32 \epsilon }{(1-2\epsilon)^2} \cdot \pseudoE\left[\kappa^2\bE_{p_S}[(v^\top (X-\mu_{q}))^2]^2 + \left(\frac{2\epsilon(1-\epsilon)}{(1-2\epsilon)^2}\left(\bE_q[(v^\top(X-\mu_q))^2] + \bE_{p_S}[(v^\top(X-\mu_{q}))^2]\right)\right)^2\right]. \nonumber 
\end{align}
Here (i) comes from the assumption that $\bE_{q}[g] \leq \bE_{p_S}[g]$, (ii) is by the certifiable hypercontractivity of $p_S$, (iii) is from Lemma~\ref{lem.modu_sec}. 
By solving the above inequality, we can derive that when $\epsilon<1/(200\kappa^2)$,
\begin{align}
    \pseudoE [\bE_{p_S}[(v^\top (X-\mu_q))^2]] \leq \frac{3}{2} \pseudoE [\bE_q[(v^\top (X-\mu_q))^2]]. 
\end{align} 

Thus following a similar line of argument as above, we know that
\begin{align}
    & \pseudoE[\bE_q[ (v^\top (X-\mu_q))^4]] \nonumber \\ 
    & \leq \pseudoE[\bE_{p_S}[(v^\top (X-\mu_q))^4]] \nonumber \\ 
    & \leq 8  \pseudoE[\bE_{p_S}[(v^\top (X-\mu_{p_S}))^4] + (v^\top(\mu_q - \mu_{p_S}))^4 ] \nonumber \\ 
    & \leq 8  \pseudoE\left[\kappa^2\bE_{p_S}[(v^\top (X-\mu_{p_S}))^2]^2 + \left(\frac{2\epsilon(1-\epsilon)}{(1-2\epsilon)^2}\left(\bE_q[(v^\top(X-\mu_q))^2] + \bE_{p_S}[(v^\top(X-\mu_{q}))^2]\right)\right)^2 \right]\nonumber \\ 
    & \leq 8  \pseudoE\left[\kappa^2\bE_{p_S}[(v^\top (X-\mu_{q}))^2]^2 + \left(\frac{2\epsilon(1-\epsilon)}{(1-2\epsilon)^2}\left(\bE_q[(v^\top(X-\mu_q))^2] + \bE_{p_S}[(v^\top(X-\mu_{q}))^2]\right)\right)^2 \right] \nonumber \\ 
    & \leq (6\kappa^2+0.03)\pseudoE[\bE_{q}[(v^\top (X-\mu_q))^2)]^2] \nonumber \\  
    & = \frac{7\kappa^2}{\tilde \kappa^2} \cdot  \pseudoE[\bE_q[ (v^\top (X-\mu_q))^4)]].
\end{align}
By solving the above inequality, we have  
\begin{align}
\tilde \kappa \leq  \sqrt{7}\kappa.
\end{align}

\end{proof}

\subsection{Generalized quasi-gradient for  sparse mean estimation }

We discuss the generalized quasi-gradient for robust sparse mean estimation here. Let $\mathcal{A}_k$ denote the set
\begin{align}
 \mathcal{A}_k =  \{A\in\bR^{d\times d}:\mathsf{Tr}(A)=1, \|A\|_1 \leq k, A \succeq 0\}.   
\end{align}
The dual norm induced by $\mathcal{A}_k$, is defined by $\|B\|_{\mathcal{A}_k}^* = \sup_{A\in\mathcal{A}_k} \mathsf{Tr}(AB)$. In the task of robust sparse mean estimation, we  set $F(q) = \|\Sigma_q - I\|_{\mathcal{X}_k}^*$ in Problem~\ref{prob.feasibility}~\citep[Chapter 3]{li2018principled},~\cite{diakonikolas2019outlier, li2017robust}. Let $g(X; q)$ be
\begin{align}
    g(X; q) = \mathsf{Tr}(A((X-\mu_q)(X-\mu_q)^\top - I)), \text{ where } A \in \argmax_{A\in\mathcal{A}_k}  \mathsf{Tr}(A(\Sigma_q - I)).
\end{align}

We show in the following theorem that $g(X; q)$ is a valid generalized quasi-gradient for  $F$.

\begin{theorem}[Generalized quasi-gradients for sparse mean estimation]\label{thm.generalized_sparse}
Let $\Delta_{S, \epsilon}= \{r\mid \forall i \in [n], r_i \leq \frac{p_{S,i}}{1-\epsilon}\}$ denote the set of $\epsilon$-deletions on $p_S$. 
We assume that the true distribution $p_S$  has near identity covariance, and its mean is stable under deletions under sparse norm, i.e. the following holds for any $r\in\Delta_{S, \epsilon}$:
\begin{align}
 \sup_{\|v\|_2\leq 1, \|v\|_0\leq k} v^\top(\mu_r - \mu_{p_S})\leq \rho.\nonumber 
    \end{align} 
Assume that $F(p_S)\geq \rho\geq \epsilon$. The following implication holds for $q\in\Delta_{n, \epsilon}$: 
\begin{align} 
   \bE_{q}[g(X; q)] \leq   \bE_{p_S}[g(X; q)]  \Rightarrow F(q)\leq C_2(\epsilon)\cdot F(p_S).
\end{align} 
Here $C_2(\epsilon)$ is some constant that depends on $\epsilon$.
 Thus $g$ is a generalized quasi-gradient  for $F$ with parameter $C_2(\epsilon)$. %
\end{theorem}
\begin{proof}
We have
\begin{align}
     F(q) = \bE_{q}[g(X; q)] & \leq   \bE_{p_S}[g(X; q)] \nonumber \\ 
     & =  \bE_{p_S}[\mathsf{Tr}(A((X-\mu_q)(X-\mu_q)^\top - I))] \nonumber \\ 
     & = \bE_{p_S}[\mathsf{Tr}(A((X-\mu_{p_S})(X-\mu_{p_S})^\top - I))]  +  \mathsf{Tr}(A(\mu_q - \mu_{p_S})(\mu_q - \mu_{p_S})^\top)  \nonumber \\ 
     & \leq F(p_S) +  \mathsf{Tr}(A(\mu_q - \mu_{p_S})(\mu_q - \mu_{p_S})^\top) \nonumber \\ 
     & \stackrel{(i)}{\leq} F(p_S)    +  4\sup_{\|v\|_2\leq 1, \|v\|_0\leq k} (v^\top (\mu_q - \mu_{p_S}))^2 \nonumber \\ 
     &  \stackrel{(ii)}{\leq} F(p_S) + C_1\cdot \sqrt{\epsilon(F(p_S)+ F(q))}. 
\end{align}
Here (i) comes from~\citet[Lemma 5.5]{li2017robust} , (ii) comes from~\citet[Proposition 5.6]{li2017robust}, and $C_1$ is some universal constant. By solving the above self-normalzing inequality for $F(q)$, we have
\begin{align}
    F(q) \leq C_2(\epsilon)\cdot F(p_S)
\end{align}
for some $C_2$ that is a function of $\epsilon$.
\end{proof}
\section{Proof for Section~\ref{sec.filter}}

\subsection{Proof of auxillary lemmas}

We also introduce the following lemma, which shows that hypercontractivity is approximately closed under deletion. 
\begin{lemma}\label{lem:deletion-hypercontractive}
Suppose that the set $\gs$ of good points is hypercontractive in the sense that 
$\frac{1}{|S|} \sum_{i \in S} (v^\top X_i)^{4} \psos (\frac{\kappa}{|S|} \sum_{i \in S} (v^\top X_i)^2)^2$. 
Then, for any $c_i\leq \frac{1}{|S|}$ such that $ 1 - \sum_{i \in S} c_i \leq \epsilon $, we have
\begin{equation}
\frac{1}{n} \sum_{i \in S} c_i (v^\top X_i)^{4} \psos \frac{\kappa^2}{1-\kappa^2 \epsilon} (\frac{1}{|S|} \sum_{i \in S} c_i (v^\top X_i)^2)^2.
\end{equation}
\end{lemma}
\begin{proof}
We expand directly; let 
\begin{align}
A = \frac{1}{|S|} \sum_{i \in S} (v^\top X_i)^{4}, \quad & B = \frac{1}{|S|} \sum_{i \in S} (v^\top X_i)^2, \\
C = \sum_{i \in S} (\frac{1}{|S|} -c_i)(v^\top X_i)^{4}, \quad & D = \sum_{i \in S} (\frac{1}{|S|} -c_i) (v^\top X_i)^2.
\end{align}
Then our goal is to show that $\frac{\kappa^2}{1-\kappa^2 \epsilon}(B-D)^2 - (A-C) \sos 0$. We are also given that 
(i) $\kappa^2 B^2 \sos A$ and we observe that (ii) $C \sos D^2 / ( 1-\sum_{i=1}^n c_i)\sos D^2/\epsilon$ by sum-of-squares H\"older's inequality. We thus have
\begin{align}
\frac{\kappa^2}{1-\kappa^2 \epsilon}(B-D)^2 - (A-C)
 &= \frac{\kappa^2}{1-\kappa^2 \epsilon}B^2 - \frac{2\kappa^2}{1-\kappa^2 \epsilon}BD + \frac{\kappa^2}{1-\kappa^2 \epsilon}D^2 - A + C \\
 &\stackrel{(i)}{\sos} (\frac{\kappa^2}{1-\kappa^2 \epsilon}-\kappa^2)B^2 - \frac{2\kappa^2}{1-\kappa^2\epsilon}BD + (\frac{\kappa^2}{1-\kappa^2\epsilon}D^2+C) \\
 &\stackrel{(ii)}{\sos} (\frac{\kappa^2}{1-\kappa^2 \epsilon}-\kappa^2)B^2 - \frac{2\kappa^2}{1-\kappa^2\epsilon}BD + (\frac{\kappa^2}{1-\kappa^2\epsilon} + \frac{1}{\epsilon})D^2 \\
 &= \frac{\kappa^4\epsilon}{1-\kappa^2 \epsilon}B^2 - \frac{2\kappa^2}{1-\kappa^2\epsilon}BD + \frac{1/\epsilon}{1-\kappa^2\epsilon}D^2 \\
 &= \frac{\epsilon}{1-\kappa^2 \epsilon} (\kappa^2 B - D/\epsilon)^2 \sos 0, 
\end{align}
as was to be shown.
\end{proof}

\subsection{Proof of Theorem~\ref{thm.low_regret_mean_cov}}\label{sub.proof_mean_cov}

With the help of Lemma~\ref{lem.explicit_condition}, we are ready to prove Theorem~\ref{thm.low_regret_mean_cov}.

Take $\beta = \sigma^2$ in Lemma~\ref{lem.explicit_condition}. From Lemma~\ref{lem.explicit_condition}, it suffices to verify Equation~\ref{eqn.approximate_quasi_stationary_statement}.  Assume for contradiction that $F(q) = \|\Sigma_q\|\geq \sigma'^2 $. 
Denote $\tilde \sigma^2 = \bE_{q}[(v^{\top}(X_i - \mu_{q}))^2]$,  where $v$ is chosen such that $\tilde \sigma^2 = \|\Sigma_q\|_2$.
Then we have %
\begin{align}
   \tilde \sigma^2  & \leq (1+\eta)\bE_{p_S} [(v^{ \top}(X - \mu_{q}))^2] + \sigma^2 \nonumber \\
    & \leq (1+\eta)(\bE_{p_S} [(v^{ \top}(X - \mu_{p_S}))^2] + (\mu_{p_S} - \mu_{q})^2) +\sigma^2 \nonumber \\ 
    & \leq (2+\eta)\sigma^2 + (1+\eta)\left(\sqrt{\frac{\sigma^2\epsilon}{1-2\epsilon}} + \sqrt{\frac{\tilde \sigma^2\epsilon}{1-2\epsilon}} \right)^2.\label{eqn.proof_mean_cov_eq1}
\end{align}
By solving the self-normalizing inequality, we have for $\epsilon<1/(3+\eta)$,
\begin{align}
    \tilde \sigma^2 \leq \frac{((1+\eta)\epsilon\sqrt{\sigma^2} + \sqrt{(1+\eta)^2\epsilon^2\sigma^2 + (2-3\epsilon-3\epsilon\eta+\eta)(1-(3+\eta)\epsilon)\sigma^2})^2 }{(1-(3+\eta)\epsilon)^2} < \sigma'^2.
\end{align}
which contradicts the assumptions. Thus the algorithm must terminate. The conclusion can be seen from Lemma~\ref{lem.mean_modulus}.

If we take $v$ such that $\tilde \sigma \geq 0.9 \|\Sigma_q\|_2$, then it suffices to replace $\tilde \sigma^2$ in~\eqref{eqn.proof_mean_cov_eq1} with $1.2\tilde \sigma^2$, which still gives a near-optimal bound up to constant.

In Theorem~\ref{thm.low_regret_mean_cov}, we have shown that it suffices to run Algorithm~\ref{algo:explicit_lowregret} with $d/(\gamma\sigma^2)$ iterations to guarantee small operator norm of covariance of $q$. Here we bound the computational complexity within each iteration. 

\begin{itemize}
    \item \textbf{Finding $v$ such that $\bE_q[(v^\top(X-\mu_q))^2]\geq 0.9\|\Sigma_q\|$. } This can be done within $O(nd\log(d))$ time from power method~\cite{hardt2014noisy}. %
    \item \textbf{Solving the projection step.} The projection step is ${q}_{i}^{(t+1)} = \mathsf{Proj}^{KL}_{\Delta_{n, \epsilon}} (c_i^{(t)} / (\sum_{i=1}^n c_i^{(t)})) = \argmin_{q\in\Delta_{n, \epsilon}} \sum_{i=1}^n q_i \log((q_i\sum_{j=1}^n c_j^{(t)})/c_i^{(t)} )$. 
    This can be done within $O(n)$ time. We discuss it in detail as below.%
\end{itemize}

Denote $p_i = c_i^{(t)} / \sum_{i=1}^n c_i^{(t)}$, and  
\begin{align}
    F(q) =  \sum_{i=1}^n q_i \log(q_i /p_i  ).
\end{align}
The Lagrangian for the optimization problem is
\begin{equation*}
    L(q,  u, y, \lambda) = F(q) + \sum_{i=1}^{n}u_i \Big(- q_{i}\Big) + \sum_{i=1}^{n}y_i \Big(q_{i} - \frac{1}{(1-\epsilon)n}\Big) + \lambda\Big(\sum_{i}^{n} q_i - 1\Big).
\end{equation*}
From the KKT conditions, we have
\begin{align*} 
    &(\text{stationarity})\quad  0 = \partial_{q}\Big( F(q) + \sum_{i=1}^{n}u_i q_{i} + \sum_{i=1}^{n}y_i \Big(q_{i} - \frac{1}{(1-\epsilon)n}\Big) + \lambda\Big(\sum_{i}^{n} q_i - 1\Big) \Big),\\
    &(\text{complementary slackness})\quad  u_i (-q_i) = 0,\,\, y_i \Big(q_{i} - \frac{1}{(1-\epsilon)n}\Big)=0,\,\,  i \in [n], \\
    &(\text{primal feasibility})\quad  - q_{i} \leq 0,\,\,     q_{i} - \frac{1}{(1-\epsilon)n}\leq 0,\,\,
     \sum_{i}^{n} q_i = 1,\\
    &(\text{dual feasibility})\quad  u_i \geq 0, \,\, y_i \geq 0,\,\,  i \in [n].
\end{align*} 
Let $g_i = \partial_{q_i}  F(q) = 1+ \log(q_i/p_i)$, from the stationary conditions we have 
\begin{align}
    0 & = g_i  - u_i + y_i + \lambda,\,\,  i \in [n],\\
     \mu_q & =  w. 
\end{align}   
For any $i\in[n]$, if $q_i = \frac{1}{(1-\epsilon)n}$, we have
\begin{align}
    0 = g_i +y_i + \lambda, y_i \geq 0.
\end{align}
If $q_i \in(0, \frac{1}{(1-\epsilon)n})$, we have
\begin{align}
     0 = g_i + \lambda. 
\end{align}
If $q_i = 0$, we have
\begin{align}
    0 = g_i - u_i + \lambda, u_i \geq 0.
\end{align}
Combine the above three equalities, we know that  
\begin{equation} 
  \frac{q_i}{p_i} \leq \frac{q_j}{p_j} \leq \frac{q_k}{p_k}, \quad \forall i, j, k \text{ with } q_i = \frac{1}{(1-\epsilon)n}, q_j \in(0,  \frac{1}{(1-\epsilon)n}), q_k = 0. 
\end{equation}
From this we know that there does not exist $k$ such that $q_k =0$, otherwise all $q_i, q_j$ shall be $0$. We also know that there must be $p_i\geq p_j$. And for all $j$, we have $q_j/p_j = -\lambda$ for some constant $\lambda$.

Thus the algorithm to compute the projection is straightforward: order $p_i$ in a descent way and compare the corresponding KL divergence. In the $m$-th itereation, we let $q_i = 1/((1-\epsilon)n)$  for all $i$ such that $p_i$ is among the $m$-th largest masses. For the left $q_i$, we just renormalize $p_i$ such that $\sum_{i=1}^n q_i = 1$. We compare the $n$ cases and  pick one that has the smallest KL divergence. This can be done within $O(n)$ time since it suffices to compare the difference.

Overall, within each iteration, the computational complexity is $O(nd\log(d))$. Overall the computational complexity is $O(nd^2\log(d)/\eta\sigma^2)$. The same complexity within single iteration also applies to filter algorithm for bounded covariance case.

\subsection{Proof of Theorem~\ref{thm.filtering_mean_cov}}\label{proof.filtering_mean_cov}

\begin{proof}

In iteration $k$, if the algorithm terminates because of  $\|\Sigma_{q^{(k)}}\|\leq  \sigma'^2$, then we know from $\TV(q^{(k)}, p_S)\leq \frac{\epsilon}{1-\epsilon}$ and  Lemma~\ref{lem.mean_modulus} that
\begin{align}\label{eqn.filtererrormean}
     \|\mu_q - \mu_{p_S}\| \leq (\sigma+\sqrt{\xi'})\sqrt{\frac{\epsilon}{1-2\epsilon}}<\frac{4\sigma\sqrt{\epsilon}}{(1-2\epsilon)^{3/2}}
\end{align}

If $F(q^{(k)},\xi')>0$, we use induction to show that   the invariance~(\ref{eqn.invariance_filter}) holds at $c^{(k)}$. Obviously, it holds at $k=0$. 
Assume it holds in step $k$, we will show that   the invariance~(\ref{eqn.invariance_filter}) still holds at step $k+1$. Since the deleted probability mass is proportional to $q_i^{(k)}\tau_i^{(k)}$ for each point $X_i$, it suffices to check
\begin{align}\label{eqn.invariance_equivalent_mean_bdd}
\sum_{i\in S} q_i^{(k)}\tau_i^{(k)} \leq \frac{1}{2}\sum_{i=1}^n q_i^{(k)}\tau_i^{(k)}.
\end{align} 

From Lemma~\ref{lem.qS_and_pS},  we know that under the invariance~(\ref{eqn.invariance_filter}) the probability of set $S$ under $q$, $q(S)$, satisfies $q(S)\geq 1-\epsilon$  and $q | S$ is a $\frac{\epsilon}{1-\epsilon}$-deletion of $p_S$. 
We have
\begin{align}
\sum_{i\in S} q_i \tau_i & = 
\sum_{i\in S} q_i  (v^{\top}(X_i - \mu_{q}))^2 \\
& =  q(S) \bE_q[  (v^{\top}(X - \mu_{q}))^2 | S] \nonumber \\ 
& = q(S) (\bE_q[  (v^{\top}(X - \mu_{q|S}))^2 | S] + (v^\top (\mu_q - \mu_{q|S}))^2) \\
& = q(S) \bE_q[  (v^{\top}(X - \mu_{q|S}))^2 | S] + q(S)(v^\top (\mu_q - \mu_{q|S}))^2). 
\end{align}
The term $q(S)(v^\top (\mu_q - \mu_{q|S}))^2$ could be upper bounded by
\begin{align}
    q(S)(v^\top (\mu_q - \mu_{q|S}))^2 & \leq  q(S)\frac{1-q(S)}{q(S)} \sum_{i=1}^n q_i (v^\top(X_i-\mu_q))^2) \\
    & \leq \epsilon \sum_{i=1}^n q_i (v^\top(X_i-\mu_q))^2)
\end{align}
following Lemma \ref{lem.mean_resilience}, $q(S)\geq 1-\epsilon$, and the fact that $q|S$ is a $1-q(S)$ deletion of $q$. For the first term, we have
\begin{align}
    q(S) \bE_q[  (v^{\top}(X - \mu_{q|S}))^2 | S] & \leq q(S) \bE_q[  (v^{\top}(X - \mu_{p_S}))^2 | S] \\
    & \leq q(S) \frac{1}{1-(\epsilon/(1-\epsilon))} \bE_{p_S}[(v^{\top}(X - \mu_{p_S}))^2] \label{eqn.qasdeletionofps} \\
    & \leq \frac{1-\epsilon}{1-2\epsilon} \bE_{p_S}[(v^{\top}(X - \mu_{p_S}))^2],
\end{align}
where in~(\ref{eqn.qasdeletionofps}) we used the inequality $q|S\leq \frac{1}{1-(\epsilon/(1-\epsilon))} p_S$ and $q(S)\leq 1$. Combining these two together, we have
\begin{align}
    \sum_{i\in S} q_i  (v^{\top}(X_i - \mu_{q}))^2 & \leq \frac{1-\epsilon}{1-2\epsilon} \bE_{p_S}[(v^{\top}(X - \mu_{p_S}))^2] + \epsilon \sum_{i=1}^n q_i (v^\top(X_i-\mu_q))^2) \\
    & \leq \frac{1-\epsilon}{1-2\epsilon}\sigma^2 + \epsilon \sum_{i \in [n]} q_i \tau_i \\
    & \leq \frac{1}{2} \sum_{i\in [n]} q_i \tau_i,
\end{align}
since we have assumed $\|\Sigma_q\|\geq \sigma'^2$ which implies $\sum_{i\in [n]} q_i \tau_i \geq \frac{2(1-\epsilon)}{(1-2\epsilon)^2} \cdot \sigma^2$. Thus~\eqref{eqn.invariance_equivalent} holds. From Lemma~\ref{lem.filter_condition}  the error is bounded by~(\ref{eqn.filtererrormean}).

\end{proof}

\subsection{Proof of Theorem~\ref{thm.linreg_implicit}}\label{proof.linreg_implicit}

When both $F_1, F_2$ are negative, the algorithm will just output $q^{(k)}$ with the desired results.  Thus
it suffices to show that no matter which cases we are in (either $F_1\geq 0$ or $F_2\geq 0$),  we will always have the invariance 
\begin{align} 
 \sum_{i\in S} (\frac{1}{n}-c_i^{(k)} ) \leq \sum_{i\in [n]/ S} (\frac{1}{n}-c_i^{(k)} ).
\end{align} 

To show the invariance, it suffices to show that under either $F_1\geq 0$ or $F_1\leq 0, F_2\geq 0$,  we have
\begin{align}
\sum_{i\in S} q_i^{(k)}\tau_i^{(k)} \leq \frac{1}{2}\sum_{i=1}^n q_i^{(k)}\tau_i^{(k)}.
\end{align} 
Now we show the first case: when $F_1\geq 0$, the invariance holds.

\paragraph{Certifiable hypercontractivity.}

It follows from the general analysis of filter algorithms~(Lemma~\ref{lem.filter_condition}) that it suffices to show the implication  
\begin{align}
  \sum_{i\in S} \frac{1}{n} - c_i \leq \sum_{i\not\in S} \frac{1}{n} - c_i, & \quad  \kappa'^2\bE_q[(v^\top X)^2]^2 \psos  \bE_q[(v^\top X)^4] \nonumber \\ 
     \Rightarrow &  \sum_{i\in S} q_i \pseudoE(v^\top X_i)^4 \leq \frac{1}{2} \sum_{i=1}^n q_i \pseudoE(v^\top X_i)^4.
\end{align}

To see the implication holds, 
observe that
\begin{align}
\sum_{i \in S} c_i \pseudoE[(v^\top X_i)^4] 
 & \stackrel{(i)}{\leq} \frac{\kappa^2}{1-2\kappa^2 \epsilon} \pseudoE[(\sum_{i \in S} c_i (v^\top X_i)^2)^2] \\
 & \stackrel{(ii)}{\leq} \frac{\kappa^2}{1-2\kappa^2 \epsilon} \pseudoE[(\sum_{i=1}^n c_i (v^\top X_i)^2)^2] \\
 & \stackrel{(iii)}{\leq} \frac{1}{2} \sum_{i=1}^n c_i \pseudoE[(v^\top X_i)^4].
\end{align}
Here (i) is by Lemma~\ref{lem:deletion-hypercontractive} (and the fact that $\pseudoE[p] \leq \pseudoE[q]$ if $p \psos q$), 
(ii) is by the fact that adding the $c_i (v^\top X_i)^2$ terms for $i \not\in S$ is adding a sum of squares, 
and (iii) is by the assumption that $E$ refutes hypercontractivity. Thus as long as $\kappa^2 \epsilon \leq 1$ 
we have the desired property.

\paragraph{Filtering for bounded noise}

Next, we show that when $F_1\leq 0, F_2\geq 0$, the invariance still holds. 

From Lemma~\ref{lem.filter_condition},  
it suffices to show the implication  
\begin{align}
  \sum_{i\in S} \frac{1}{n} - c_i \leq \sum_{i\not\in S} \frac{1}{n} - c_i,   & \forall v \in\bR^d,  {\bE_q[(Y - X^\top \theta(q))^2(v^\top X)^2]} \geq { \sigma'^2\bE_q[(v^\top X)^2]}  \nonumber \\ 
     \Rightarrow &  \sum_{i\in S} q_i (Y_i - X_i^\top \theta(q))^2(v^\top X_i)^2 \leq \frac{1}{2} \sum_{i=1}^n q_i (Y_i - X_i^\top \theta(q))^2(v^\top X_i)^2.
\end{align}
Denote $\tilde \sigma^2 = F(q)$. 
 Then 
the LHS satisfies
\begin{align}
    \sum_{i\in S} q_i (Y_i - X_i^\top \theta(q))^2(v^\top X_i)^2 
    & \leq 2 \sum_{i\in S} q_i( (Y_i - X_i^\top \theta(p_S))^2(v^\top X_i)^2 + ((\theta(p_S)-\theta(q))^\top X_i)^2(v^\top X_i)^2) \nonumber \\ 
    & \stackrel{(i)}{\leq} 2  (\frac{1}{1-2\epsilon} \bE_{p_S}[(Y - X^\top \theta(p_S))^2(v^\top X)^2] \nonumber \\ 
    & \quad + \bE_{q}[((\theta(p_S)-\theta(q))^\top X)^4]^{1/2}\cdot \bE_{q}[(v^\top X)^4]^{1/2}) \nonumber \\
    & \stackrel{(ii)}{\leq}  \frac{2}{1-2\epsilon}  (\sigma^2 \bE_{p_S}[(v^\top X)^2] + {5\kappa'^2} \bE_{q}[((\theta(p_S)-\theta(q))^\top X)^2]\cdot \bE_{q}[(v^\top X)^2])\nonumber \\ 
    & \stackrel{(iii)}{\leq} \frac{2}{1-2\epsilon}( \frac{(1-2\epsilon+2\kappa \sqrt{\epsilon(1-\epsilon)})^2\sigma^2}{(1-2\epsilon-2\kappa \sqrt{\epsilon(1-\epsilon)})^2} \nonumber \\ 
    & \quad + 5\kappa'^2\bE_{q}[((\theta(p_S)-\theta(q))^\top X)^2])\bE_{q}[(v^\top X)^2] \nonumber \\ 
    & \stackrel{(iv)}{\leq}  \frac{2}{1-2\epsilon}( \frac{(1-2\epsilon+2 \kappa \sqrt{\epsilon(1-\epsilon)})^2\sigma^2}{(1-2\epsilon-2\kappa \sqrt{\epsilon(1-\epsilon)})}\nonumber \\ 
    & \quad  + 5\kappa'^2\bE_{q}[((\theta(p_S)-\theta(q))^\top X)^2])\bE_{q}[(Y-X^\top\theta(q))^2(v^\top X)^2]/\tilde\sigma^2
\end{align}
Here (i) comes from that $q_i\leq \frac{p_{S,i}}{1-2\epsilon}$ for all $i\in[n]$, (ii) comes from the assumption on $p_S$ and the hypercontractivity of $q$, (iii) is by Lemma~\ref{lem.modu_hyper}, (iv) comes from the definition of $\tilde \sigma^2$.
From~\citep[Theorem 3.4]{zhu2019generalized}, we know that when $\epsilon<  1/(1+4\kappa'^2) $
\begin{align}
    \bE_{q}[((\theta(p_S)-\theta(q))^\top X)^2]\leq \frac{2\kappa'\tilde\sigma^2\epsilon(1-\epsilon)}{(1-2\epsilon)^2}.
\end{align}

Denote $\tilde \sigma^2 = \sup_{v\in\bR^d}\bE_q[(Y-X^\top \theta(q))^2(v^\top X)^2]/\bE_q[(v^\top X)^2]$.
Then overall, we have
\begin{align}
    \sum_{i\in S} q_i (Y_i - X_i^\top \theta(q))^2(v^\top X_i)^2 
    & \leq \frac{2}{(1-2\epsilon)\tilde \sigma^2}( \frac{(1-2\epsilon+2 \kappa \sqrt{\epsilon(1-\epsilon)})^2\sigma^2}{(1-2\epsilon-2 \kappa \sqrt{\epsilon(1-\epsilon)})^2}\nonumber \\ 
    & \quad  + \frac{10\kappa'^3\tilde\sigma^2\epsilon(1-\epsilon)}{(1-2\epsilon)^2})\bE_{q}[(Y-X^\top\theta(q))^2(v^\top X)^2]) \nonumber \\ 
    & \leq  \frac{2}{(1-2\epsilon) \sigma'^2}( \frac{(1-2\epsilon+2 \kappa \sqrt{\epsilon(1-\epsilon)})^2\sigma^2}{(1-2\epsilon-2 \kappa \sqrt{\epsilon(1-\epsilon)})^2}\nonumber \\ 
    & \quad  + \frac{5\kappa'^3\sigma'^2\epsilon(1-\epsilon)}{(1-2\epsilon)^2})\bE_{q}[(Y-X^\top\theta(q))^2(v^\top X)^2]). 
\end{align}

By taking $\sigma'^2 =  \frac{4\sigma^2(1-2\epsilon+2\kappa'\sqrt{\epsilon(1-\epsilon)})}{(1-2\epsilon)^3-20\kappa'^3\epsilon(1-\epsilon)}$, we know that the implication holds.  
\subsection{Proof of Theorem~\ref{thm.filtering_joint}}\label{proof.filtering_joint}

\begin{proof}[Proof of Theorem~\ref{thm.filtering_joint}]
It follows from the general analysis of filter  algorithm~(Lemma~\ref{lem.filter_condition}) that it suffices to show the implication 
\begin{align}
  \sum_{i\in S} \frac{1}{n} - c_i \leq \sum_{i\not\in S} \frac{1}{n} - c_i,   &  \kappa'^2\bE_q[(v^\top (X-\mu_q))^2]^2 \psos  \bE_q[(v^\top (X-\mu_q))^4] \nonumber \\ 
     \Rightarrow &  \sum_{i\in S} q_i \pseudoE(v^\top (X_i-\mu_q))^4 \leq \frac{1}{2} \sum_{i=1}^n q_i \pseudoE(v^\top (X_i-\mu_q))^4.
\end{align}

Observe that when $\kappa^2\epsilon< 1/4 $, %
\begin{align}
\sum_{i \in S} c_i \pseudoE[(v^\top (X_i-\mu_{q}))^4] & \leq \sum_{i \in S} 8c_i \pseudoE[(v^\top (X_i-\mu_{q|S}))^4  + (v^\top(\mu_q - \mu_{q|S}) )^4]  \nonumber \\ 
 & \stackrel{(i)}{\leq} \frac{8\kappa^2}{1-2\kappa^2 \epsilon} \pseudoE[(\sum_{i \in S} c_i (v^\top (X_i-\mu_{q|S}))^2)^2] + \frac{32\epsilon(1-\epsilon)^2}{(1-2\epsilon)^2} \pseudoE[\bE_q[(v^\top(X-\mu_q))^2]^2 ] \nonumber \\
 & \stackrel{(ii)}{\leq} \frac{8\kappa^2}{1-2\kappa^2 \epsilon} \pseudoE[(\sum_{i=1}^n c_i (v^\top (X_i-\mu_{q}))^2)^2]+ \frac{32\epsilon(1-\epsilon)^2}{(1-2\epsilon)^2} \pseudoE[\bE_q[(v^\top(X-\mu_q))^2]^2 ]  \nonumber\\ 
 & {\leq} (\frac{8\kappa^2}{1-2\kappa^2 \epsilon}+ \frac{32\epsilon(1-\epsilon)}{(1-2\epsilon)^2})\cdot  \pseudoE[(\sum_{i=1}^n c_i (v^\top (X_i-\mu_{q}))^2)^2]  \nonumber \\
 & \stackrel{(iii)}{\leq} \frac{1}{2} \sum_{i=1}^n c_i E_v[(v^\top X_i)^4].
\end{align}
Here (i) is by Lemma~\ref{lem:deletion-hypercontractive} and Lemma~\ref{lem.modu_sec} (and the fact that $\pseudoE[p] \leq \pseudoE[q]$ if $p \psos q$), 
(ii) is by the fact that adding the $c_i (v^\top X_i)^2$ terms for $i \not\in S$ is adding a sum of squares, 
and (iii) is by the assumption that $E$ refutes hypercontractivity. Thus as long as $\epsilon< 1/4\kappa^2$
we have the desired property.

\end{proof}

\subsection{Proof of Theorem~\ref{thm.low_regret_mean_identity_cov}}\label{proof.low_regret_mean_identity_cov}
From boundedness of $X_i$, we know that $g(X_i; q^{(k)})\leq 4d/\epsilon$.  Denote $\eta^{(k)} = \delta \cdot \frac{\epsilon}{8d}$. The algorithm has the following regret bound from~\citep[Theorem 2.4]{arora2012multiplicative}:
\begin{align}
    & \frac{1}{T}\sum_{t=1}^T \left(\bE_{ q^{(t)}}[(v^{(t)\top}(X_i-\mu_q))^2 - 1] -\bE_{p_S} [(v^{(t)\top}(X_i-\mu_q))^2 - 1] \right) \nonumber \\ 
    \leq & \frac{\delta}{ 2T} \sum_{t=1}^T \bE_{p_S}[|(v^{(t)\top}(X_i-\mu_q))^2 - 1|]  + \frac{16d}{T\delta} \nonumber \\ 
     \leq & \frac{\delta}{  2T} \sum_{t=1}^T (\bE_{p_S}[(v^{(t)\top}(X_i-\mu_{p_S}))^2 ] +1+ \|\mu_q-\mu_{p_S}\|_2^2) + \frac{16d}{T\delta}\nonumber \\ 
     \leq & \frac{\delta}{  2T} \sum_{t=1}^T (2+\xi+ \|\mu_q-\mu_{p_S}\|_2^2) + \frac{16d}{T\delta}  \nonumber \\ 
     \leq & \frac{\delta}{  2T} \sum_{t=1}^T \|\mu_q-\mu_{p_S}\|_2^2 + (1+\xi/2)\delta +  \frac{16d}{T\delta}
\end{align}
By taking $T = T_0 = \frac{8(2+\xi)d}{\xi^2}$ and taking $\delta =  \frac{4\beta \sqrt{d}}{\sqrt{(1+\xi/2)T}} = 2\beta\xi/(2+\xi), \beta \in(0, 1)$, we have
\begin{align}
     \frac{1}{T_0}\sum_{t=1}^{T_0} \left(\bE_{ q^{(t)}}[(v^\top(X_i-\mu_q))^2 - 1] - \bE_{p_S} [ (v^\top(X_i-\mu_q))^2 - 1 ] - \frac{ \beta\xi}{2+\xi } \|\mu_q-\mu_{p_S}\|_2^2\right)\leq 2\xi/\beta.
\end{align}
Thus there must exists some $t_0\in [T_0]$ such that
\begin{align}
   \bE_q [(v^\top(X_i-\mu_q))^2 - 1] & \leq    \bE_{p_S} [  (v^\top(X_i-\mu_q))^2 - 1  ] + \frac{ \beta\xi}{2+\xi } \|\mu_q-\mu_{p_S}\|_2^2 + 2\xi/\beta \nonumber \\ 
   & = \bE_{p_S} [  (v^\top(X_i-\mu_{p_S}))^2 - 1  ]   + (1+\frac{ \beta\xi}{2+\xi })\cdot  \|\mu_q-\mu_{p_S}\|_2^2 + 2\xi/\beta \nonumber \\ 
   & \leq   (1+\frac{ \beta\xi}{2+\xi })\cdot \|\mu_{p_S}-\mu_q\|^2   + (1+2/\beta)\xi \nonumber \\ & \leq (1+\frac{ \beta\xi}{2+\xi })\cdot \left( \frac{  \rho}{1-2\epsilon} + \sqrt{\frac{\epsilon( \xi+\xi'+\epsilon/(1-\epsilon))}{1-2\epsilon} + \frac{\epsilon(1-\epsilon)\rho^2}{(1-2\epsilon)^2}}\right)^2   + (1+2/\beta)\xi,
\end{align}
where we denote $\xi'  = \max(\|\Sigma_q\|-1, 0)$. The last inequality comes from Lemma~\ref{lem.empirical_identity_cov_modulus}. 
On the other hand, from the choice of $v$, we have
\begin{align}
     \bE_q [(v^\top(X_i-\mu_q))^2 - 1]\geq \xi'(1 - \gamma)-\gamma. 
\end{align}

By combining the above two inequalies and solve them for $\xi'$, we can see that there exists some constant $C$, 
\begin{align}
    \xi' \leq C \cdot \frac{(1+1/\beta)\xi+  \rho^2+\epsilon}{(1-3(1+\beta \xi/(1-\gamma\epsilon))\epsilon)^2}.
\end{align}

\end{document}